\documentclass{article} % For LaTeX2e
\usepackage{iclr2026_conference,times}

\usepackage{amsmath,amsfonts,bm}

\def\eqref#1{equation~\ref{#1}}
\def\1{\bm{1}}

\DeclareMathAlphabet{\mathsfit}{\encodingdefault}{\sfdefault}{m}{sl}
\SetMathAlphabet{\mathsfit}{bold}{\encodingdefault}{\sfdefault}{bx}{n}

\newcommand{\Var}{\mathrm{Var}}

\usepackage{hyperref}
\usepackage{url}

\usepackage[utf8]{inputenc} % allow utf-8 input
\usepackage[T1]{fontenc}    % use 8-bit T1 fonts
\usepackage{hyperref}       % hyperlinks
\usepackage{url}            % simple URL typesetting
\usepackage{booktabs}       % professional-quality tables
\usepackage{amsfonts}       % blackboard math symbols
\usepackage{nicefrac}       % compact symbols for 1/2, etc.
\usepackage{microtype}      % microtypography
\usepackage{xcolor}         % colors

\usepackage{amssymb}
\usepackage{pifont}
\newcommand{\cmark}{\ding{51}} % 对号
\newcommand{\xmark}{\ding{55}} % 错号

\usepackage{graphicx}
\usepackage{paralist}
\usepackage{multirow}
\usepackage{makecell}
\usepackage{subfigure}
\usepackage{wrapfig}
\usepackage[normalem]{ulem}
\useunder{\uline}{\ul}{}

\usepackage{colortbl}

\usepackage{amsmath}
\usepackage{amsthm}
\newtheorem{theorem}{Theorem}

\title{\method: Learning Hierarchical Structures for High-Dimensional Time Series Forecasting}

\author{%
Juntong Ni\thanks{Equal contribution}, 
\ Shiyu Wang\footnotemark[1]\hspace{0.4em}\thanks{Project lead}
,
Zewen Liu\textsuperscript{},
Xiaoming Shi\textsuperscript{},
Xinyue Zhong\textsuperscript{},
Zhou Ye\textsuperscript{},
Wei Jin\textsuperscript{}\\
\textsuperscript{}{Emory University}\\
\texttt{\{juntong.ni, zewen.liu, wei.jin\}@emory.edu, kwuking@gmail.com}}

\iclrfinalcopy % Uncomment for camera-ready version, but NOT for submission.
\begin{document}
\newcommand{\method}{\textsc{U-Cast}}
\newcommand{\dataset}{\textsc{Time-HD}}

\maketitle

\begin{abstract}
Time series forecasting (TSF) is a central problem in time series analysis. However, as the number of channels in time series datasets scales to the thousands or more, a scenario we define as \textbf{High-Dimensional Time Series Forecasting (HDTSF)}, it introduces significant new modeling challenges that are often not the primary focus of traditional TSF research. HDTSF is challenging because the channel correlation often forms complex and hierarchical patterns. Existing TSF models either ignore these interactions or fail to scale as dimensionality grows. To address this issue, we propose \method{}, a channel-dependent forecasting architecture that learns latent hierarchical channel structures with an innovative query-based attention. To disentangle highly correlated channel representation, \method{} adds a full-rank regularization during training. We also release \dataset{}, the first benchmark of large, diverse, high-dimensional datasets. Our theory shows that exploiting cross-channel information lowers forecasting risk, and experiments on \dataset{} demonstrate that \method{} surpasses strong baselines in both accuracy and efficiency. Together, \method{} and \dataset{} provide a solid basis for future HDTSF research. Our \href{https://github.com/UnifiedTSAI/Time-HD-Lib}{code} and \href{https://huggingface.co/datasets/Time-HD-Anonymous/High_Dimensional_Time_Series}{benchmark} are available to ensure reproducibility.  

\end{abstract}

\section{Introduction}
\label{sec:Introduction}
% \wei{one major comment for the writing of our intro is: we should spend more efforts to explain why }

% 1. 时间序列很重要，在各个现实应用中无处不在 
% 2. 时间序列中的维度很至关重要，维度信息有丰富的现实含义，目前的研究对于维度的关注不够 
% 3.进一步我们提出高维时序，在图像或者其他领域中都是高维数据，但是时间序列中的维度都偏低，比如一元时序只有一个维度，多元时序的维度也只有数百维，现实中时序的维度是非常多的高维数据

% \begin{wrapfigure}[]{R}{0.25\textwidth}
% \vskip-1.2em
%     \centering
%     \includegraphics[width=1\linewidth]{figures/Time-HD.pdf}
%     \vskip-0.3em
%     \caption{\dataset{} provides diverse high-dimensional datasets.}
%     % \caption{\dataset{} provides high-dimensional datasets that are diverse in scale, frequency, domains.}
%     \label{fig:time_hd}
% \vskip-1.2em
% \end{wrapfigure}

\begin{wraptable}{r}{0.3\textwidth}
\centering
\vskip-2.2em
\caption{Comparison of existing datasets by channel scale.}
\resizebox{0.28\textwidth}{!}{%
\begin{tabular}{lc}
\toprule
\textbf{Datasets} & \textbf{\#Channels} \\
\midrule
ETT        & 7 \\
Weather    & 21 \\
Solar      & 137 \\
ECL      & 321 \\
Traffic      & 862 \\
\midrule
\textbf{\dataset{}} & \textbf{1k--20k} \\
\bottomrule
\end{tabular}
}
\label{tab:comparison_channel_scale}
\vskip-1.5em
\end{wraptable}

% \begin{wraptable}{r}{0.3\textwidth}
% \centering
% \vskip-2em
% \caption{Comparison of benchmark datasets by channel scale.}
% \resizebox{0.3\textwidth}{!}{%
% \begin{tabular}{lc}
% \toprule
% \textbf{Benchmark} & \textbf{Channel Range} \\
% \midrule
% TSlib        & 7--862 \\
% BasicTS      & 7--883 \\
% TFB          & 7--2,000 \\
% \midrule
% \textbf{\dataset{}} & \textbf{2,000--20,000} \\
% \bottomrule
% \end{tabular}
% }
% \label{tab:comparison}
% \vskip-1.2em
% \end{wraptable}

Time series data involving multiple variables (a.k.a. channels) that change over time are fundamental to numerous real-world applications. 
% such as stock market prices~\citep{granger2014forecasting}, highway traffic flows~\citep{yin2021deep}, solar power generation outputs~\citep{wan2015photovoltaic}, and temperature measurements~\citep{wu2023interpretable}. 
Consequently, time series forecasting (TSF) has become a central challenge in the time series analysis community, aiming to predict future values based on historical observations~\citep{chatfield2000time, de200625}.  Recent advances in machine learning have catalyzed significant methodological progress in this field~\citep{wang2024deep, lim2021time}, leading to a variety of novel architectures~\citep{autoformer, fedformer, informer, itransformer, nbeats, patchtst, timemixer, timesnet, timedistill}.  
% One important consideration in time series forecasting is the modeling of inter-channel dependencies~\citep{han2024capacity,itransformer,zhao2024rethinking}, as such correlations can often provide valuable cues for improving predictive accuracy~\citep{qiu2024tfb}. 

% These models can be broadly categorized by how they handle the inter-channel dependencies: channel-dependent (CD) methods~\citep{itransformer,crossformer,autoformer,informer,ccm,fedformer} explicitly capture cross-variable interactions, while channel-independent (CI) methods~\citep{patchtst,dlinear,pdf,timemixer,fits,sparsetsf} treat each variable separately. 

% However, as the number of channels in time series datasets scales into the thousands or more, a scenario we define as \textbf{High-Dimensional Time Series Forecasting (HDTSF)} has posed significant new modeling challenges that are often not the primary focus of traditional TSF research.

Despite this progress, most existing TSF datasets, as shown in Table~\ref{tab:comparison_channel_scale} and Table~\ref{tab:comparison}, remain relatively low-dimensional, with only a handful of channels (e.g., ETT, Weather) to at most a few hundred (e.g., Traffic, ECL). 
However, in many real-world applications, \textbf{the number of channels can easily scale into the thousands or more}. 
For instance, financial markets involve thousands of stocks~\citep{granger2014forecasting}, 
urban traffic systems rely on tens of thousands of sensors~\citep{yin2021deep}, 
smart grids generate massive streams from large-scale meters~\citep{wan2015photovoltaic}, 
and climate reanalysis contains hundreds of thousands of variables defined on global spatial grids~\citep{wu2023interpretable}.
We define this scenario as \textbf{High-Dimensional Time Series Forecasting (HDTSF)}, where the scale of channels introduces significant new modeling challenges that are often not the primary focus of traditional TSF research.

Two difficulties are especially pressing. First, scaling to thousands of channels with complicated dependencies makes many channel-dependent models, such as iTransformer~\citep{itransformer} and TSMixer~\citep{tsmixer}, computationally inefficient or unreliable. Second, high dimensionality often gives rise to \textbf{latent hierarchical structures} among the channels. These structures are common in large real-world systems {(Appendix~\ref{app:Empirical Evidence of Hierarchical Structures}), where channels exhibit implicit groupings and multi-level correlations based on factors such as spatial proximity (e.g., nested geographical regions in climate data) or semantic relationships (e.g., stocks within related economic sub-sectors).  While such multi-scale correlations are pervasive, most existing models are not designed to explicitly discover or leverage them, which limits generalizability in high-dimensional contexts.
This motivates two central research questions: (1) How can we build \textit{forecasting architectures that explicitly capture latent hierarchical structures} to scale across thousands of channels? (2) How can we systematically \textit{evaluate such} models when existing benchmarks rarely exceed a few hundred channels?

To address the first question, we propose \textbf{\method{}}, a scalable forecasting architecture that introduces a Hierarchical Latent Query Network and a Hierarchical Upsampling Network to efficiently uncover latent multi-scale channel organizations. To disentangle highly correlated inputs, \method{} also employs a full-rank regularization objective that encourages diverse, non-redundant channel representations.  Answering the second question requires a \textit{new empirical foundation}. Existing TSF benchmarks contain only limited channels and may not possess the scale or structural richness representative of high-dimensional applications, making them inadequate for evaluating scalability and advanced dependency modeling. To this end, we curate \textbf{\dataset{}}, the first comprehensive benchmark suite for HDTSF. {\dataset{}} spans 16 datasets with 1k–20k channels across diverse domains, providing a much-needed testbed for evaluating scalability and dependency modeling. While this paper does not claim to exhaustively analyze all benchmarking results, we release \dataset{} as a community resource to facilitate systematic study of HDTSF in future work. Our experiments on \dataset{} validate \method{}'s efficacy in learning hierarchical structures and achieving superior forecasting accuracy in these demanding high-dimensional settings. Our main contributions are summarized as:

\textbf{C1: Theoretical and empirical analysis of high-dimensional channel dependency.} We prove that channel-dependent (CD) models have lower Bayes risk than channel-independent (CI) models whenever non-redundant channels exist, with the advantage increasing in higher dimensions (Section~\ref{sec:Preliminary Study}). Our controlled experiments on synthetic and real-world datasets further corroborate these findings.

% We extend prior work on channel interactions by formally proving that channel-dependent (CD) models achieve lower Bayes risk than channel-independent (CI) models whenever non-redundant channels exist, and that the benefit grows with dimensionality. Controlled experiments confirm this, highlighting the need for benchmarks beyond traditional low-dimensional datasets.

\textbf{C2: \method{}: A scalable architecture for HDTSF.} We propose \textsc{\method{}} (Section~\ref{sec:UCast Framework}), a new CD forecasting model that introduces an innovative query-based attention mechanism to efficiently learn latent hierarchical channel structures and enables scalable modeling of inter-channel correlations. \method{} combines strong performance across \dataset{} with the best efficiency among baselines, making it a strong reference design for future HDTSF models.

\textbf{C3: \dataset{}: the first benchmark suite for HDTSF.} We release \textsc{\dataset{}} (Section~\ref{sec:High-Dimensional Time Series Forecasting Benchmark}), the first comprehensive benchmark suite curated for HDTSF. \dataset{} spans 16 datasets (1k–20k channels) across diverse domains, providing the necessary foundation for evaluating scalability and dependency modeling in realistic high-dimensional contexts. While not intended as an exhaustive benchmarking study, \dataset{} establishes a standardized and challenging testbed that enables rigorous evaluation and comparison of future models. 
% \wei{not sure if we want to keep the following sentence. if space limited, we may remove/shorten it..} Building on this benchmark, we observe that CD methods do not consistently outperform CI methods, likely because current CD baselines either fail to capture complex channel dependencies or do not scale to high-dimensional time series.

% These datasets, characterized by substantial scale and broad domain representation, are systematically collected, curated, and preprocessed to offer a realistic and challenging testbed for evaluating model scalability and generalizability.

\textbf{C4: Open-source library for reproducibility.} We provide \textsc{\dataset{}-lib}, an open-source library supporting the benchmark. It offers standardized preprocessing, unified evaluation protocols, and automated hyperparameter optimization. This toolkit ensures reproducibility of our results and lowers the barrier for future work on HDTSF.

\section{Related Work}
\textbf{High-Dimensional Time Series Datasets.}
Some high-dimensional datasets have already been used in recent work on building time series foundation models~\citep{moirai, chronos, time-moe, moment}. These studies primarily focus on leveraging high-dimensional datasets for pretraining. However, several limitations exist in how they handle such datasets. \textbf{First}, during pretraining, channels are treated independently, capturing only temporal dependencies without modeling inter-channel correlations. \textbf{Second}, despite pretraining on high-dimensional datasets, fine-tuning and evaluation are still performed on low-dimensional ones. \textbf{Third}, some pretraining datasets have unaligned timestamps, complicating their use in evaluation. These limitations highlight the importance of developing high-dimensional time series datasets specifically tailored for benchmarking and evaluation in time series forecasting research. Table~\ref{tab:comparison} in Appendix~\ref{app:Comparison with Existing Works} provides a more detailed summarization of the differences between \dataset{} and existing works.

\textbf{Multivariate Time Series Forecasting.}
Previous methods for modeling multivariate time series can be broadly categorized into  \textbf{channel-independent (CI)} and \textbf{channel-dependent (CD)} strategies.
The CI modeling strategy employs a shared model backbone across all channels and processes each channel independently during the forward pass~\citep{patchtst,dlinear,fits,timemixer,pattn,pdf,sparsetsf,li2023extreme}. This design typically results in lower model capacity but offers greater robustness.
In contrast, the CD modeling strategy introduces dedicated modules to capture inter-channel dependencies. Based on the granularity at which these correlations are modeled, CD can be further divided into position-wise\citep{tsmixer,timesnet,autoformer,informer,fedformer,crossformer}, token-wise\citep{itransformer,timexer}, and cluster-wise~\citep{ccm,duet,li2025mc} approaches. 
These methods are expected to provide higher capacity and make better use of cross-channel information. 
However, due to the constraints of existing low-dimensional time series forecasting benchmarks, most methods have not been validated at scale.

\section{Preliminary Study}
\label{sec:Preliminary Study}
Prior studies on existing time series benchmarks, which are often of low dimensionality, do not consistently show a clear advantage of CD over CI models~\citep{zhao2024rethinking,patchtst,dlinear,pdf}. This observation might lead to skepticism regarding the practical benefits of explicitly modeling inter-channel correlations. However, we hypothesize that the limited dimensionality of these benchmarks inherently restricts the potential gains from CD approaches. This preliminary study aims to theoretically and empirically investigate the impact of data dimensionality on the forecasting performance achievable by leveraging channel correlations. We begin with a theoretical analysis of risk reduction.

% However, we argue that the low-dimensional nature of existing time series forecasting benchmarks hinders a fair comparison. To address this, we conduct a preliminary study to investigate the impact of data dimensionality on forecasting performance.
% Unexpectedly, prior results on low-dimensional benchmarks show no clear advantage of CD over CI~\citep{zhao2024rethinking,patchtst,dlinear}.
% Below, we first give a theoretical analysis to show the risk reduction achieved by using a CD approach compared to CI under squared-error loss. Then, we give an empirical analysis to support our findings.

\textbf{Task Formulation.}
For time series forecasting, given an input time series \( \mathbf{X} \in \mathbb{R}^{C \times T} \), where \( T \) represents the length of the look-back window and \( C \) represents the number of variables, the goal is to predict the future \( S \) time steps \( \mathbf{Y} \in \mathbb{R}^{C \times S} \). 

\subsection{Theoretical Analysis of Risk Reduction from CD Models}

To formally analyze the benefits of CD, we consider a simplified setting. Assume a \textbf{bivariate} time series ($C=2$) generated by a Vector Autoregression {of order 1} (VAR(1))~\citep{stock2001vector}:
\begin{equation}\label{eq:var_system}
  z_{t+1}=\mathbf{A} z_t+\varepsilon_{t+1},\quad
  \mathbf{A}=\begin{pmatrix}a_{11}&a_{12}\\ a_{21}&a_{22}\end{pmatrix},\quad
  \varepsilon_t\stackrel{i.i.d.}{\sim}\mathcal N(\mathbf0,\Sigma),\quad
  \;\mathbf{\Sigma}=\begin{pmatrix}\sigma_{11}&0\\0&\sigma_{22}\end{pmatrix}.
\end{equation}
% Here, \( z_t = \begin{pmatrix} z_t^{(1)} \\ z_t^{(2)} \end{pmatrix} \) 
Here, \( z_t = ( z_t^{(1)}, z_t^{(2)})^\top  \) is the bivariate time series value at time \( t \), with $z_0 \sim \mathcal{N}(\mathbf{0}, \mathbf{I})$, where \( \mathbf{I} \) is the identity matrix. The matrix \( \mathbf{A} \) captures the linear dependencies, with coefficients $a_{ij}$ describing self-dependencies ($i=j$) and cross-dependencies ($i\neq j$).  The noise \( \varepsilon_t \) is zero-mean Gaussian with independent components. Our goal is to predict \( z_{t+1} \) given \( z_{t} \). All time series are assumed to be normalized to have zero mean.  We compare two modeling strategies under squared-error loss: \textit{(1) CI Modeling:} Each channel $z_{t+1}^{(i)}$ is predicted using only its own past, $z_t^{(i)}$. The Bayes optimal forecast for $z_{t+1}^{(i)}$ is $\mathbb{E}[z_{t+1}^{(i)}|z_t^{(i)}]$ and the total risk is $R_{\text{CI}} = \sum_{i=1}^2 \mathbb{E}[(z_{t+1}^{(i)} - \mathbb{E}(z_{t+1}^{(i)}|z_{t}^{(i)})^2]$. \textit{(2) CD Modeling:} Each channel $z_{t+1}^{(i)}$ is predicted using the past of all channels. The Bayes optimal forecast for $z_{t+1}^{(i)}$ is $\mathbb{E}[z_{t+1}^{(i)}|z_t]$ and the total risk is $R_{\text{CD}} = \sum_{i=1}^2 \mathbb{E}[(z_{t+1}^{(i)} - E(z_{t+1}^{(i)}|z_{t})^2]$.

\begin{theorem}[Risk Reduction from CD]
\label{thm:risk_reduction} 
For the time series described in Eq.~\eqref{eq:var_system}, the Bayes risks of CI and CD models under squared-error loss satisfy: $R_{\text{CI}} - R_{\text{CD}} \ge 0$.
This inequality is strict if and only if $a_{12} \ne 0$ and $a_{21} \ne 0$, and the conditional variances $\Var(z_t^{(1)} \mid z_t^{(2)})$ (representing the variance of $z_t^{(1)}$ that remains after $z_t^{(2)}$ is known) and $\Var(z_t^{(2)} \mid z_t^{(1)})$ are positive.
% The inequality is strict if and only if $a_{12}, a_{21} \ne 0$ and one channel carries some information not already present in another channel.
\end{theorem}
We provide the proof of this theorem in Appendix~\ref{app:Theorem Proof1}. 
Theorem~\ref{thm:risk_reduction} establishes that CD modeling is at least as good as CI modeling. A strict advantage for CD arises if: (1) there is genuine predictive information flowing between channels (non-zero $a_{12}$ or $a_{21}$ means one channel directly influences the future of the other), and (2) each channel contains some unique information not present in the other (positive conditional variances). If channels are truly independent in their evolution ($a_{12}=a_{21}=0$) or one is a deterministic function of the other, then $R_\text{CI} =R_\text{CD}$.

% This theorem holds when the one channel has a non-zero direct influence on the another channel in the next period $a_{12}, a_{21} \ne 0$ and when one channel carries some information not already present in the first channel (i.e., one channel; is not perfectly predictable from another channel, ensuring $\Var(z_t^{(1)} \mid z_t^{(2)}),\Var(z_t^{(2)} \mid z_t^{(1)}) > 0$). If $a_{12}, a_{21} = 0$, or if one channel is perfectly predictable from another channel, the risks are equal. Thus, using the multivariate information (CD) set leads to a Bayes risk that is always less than or equal to the Bayes risk obtained using only the univariate information set (CI).

To investigate the impact of increasing dimensionality, we now consider a general $P$-channel VAR(1) process:   $z_{t+1}=\mathbf{A} z_t+\varepsilon_{t+1}$, where $z_t\in\mathbb{R}^P$. 
We focus on forecasting a single target channel, say $Y = z_{t+1}^{(1)}$, which is given by $z_{t+1}^{(1)} = \sum_{j=1}^P a_{1j} z_t^{(j)} + \varepsilon_{t+1}^{(1)}$, with $\Var(\varepsilon_{t+1}^{(1)}) = \sigma_{11}$.
Let $R_p$ denote the Bayes risk (minimum MSE) for forecasting $Y = z_{t+1}^{(1)}$ using the information from the first $p$ channels at time $t$, i.e., $\{z_t^{(1)}, \dots, z_t^{(p)}\}$.

% Further, we can extend this comparison to a general case with $P$ channels, $z_t = (z_t^{(1)}, \dots, z_t^{(P)})^T$, following the Vector Autoregression process: $z_{t+1} = \mathbf{A} z_t + \varepsilon_{t+1}$.
% We aim to forecast $Y = z_{t+1}^{(1)} = \sum_{j=1}^P a_{1j} z_t^{(j)} + \varepsilon_{t+1}^{(1)}$, where $\Var(\varepsilon_{t+1}^{(1)}) = \sigma_{11}$.

\begin{theorem}[Bayes Risk Monotonicity with Increasing Channel Information]
\label{thm:monotonicity}
For forecasting the target channel $z_{t+1}^{(1)}$, the Bayes risks obtained by conditioning on an increasing number of predictor channels $\{z_t^{(1)}, \dots, z_t^{(p)}\}$ exhibit monotonicity: $R_1 \ge R_2 \ge \dots \ge R_P = \sigma_{11}$. Note that  $R_1$ is the risk of CI Modeling (using only $z_t^{(1)}$ to predict $z_{t+1}^{(1)}$), and $R_p (p>1)$ is the risk of CD Modeling.
Furthermore, the risk difference between CI and CD, $\Delta_p = R_1 - R_p$, is non-decreasing in $p$:
\[
  0 = \Delta_1 \le \Delta_2 \le \dots \le \Delta_P = R_1 - \sigma_{11},
\]
This gap strictly increases (i.e., $\Delta_p > \Delta_{p-1}$) if and only if the $p$-th channel, $z_t^{(p)}$, provides new information for predicting $Y$ that is not already contained in channels $z_t^{(1)}, \dots, z_t^{(p-1)}$.
% and strictly increases ($\Delta_p > \Delta_{p-1}$) if and only if the $p$-th channel $z_t^{(p)}$ contributes to predicting $Y$ beyond what is already provided by $z_t^{(1)}, \dots, z_t^{(p-1)}$.
\end{theorem}
The proof is provided in Appendix~\ref{app:Theorem Proof2}.  Theorem~\ref{thm:monotonicity} formally shows that incorporating additional informative and non-redundant channels progressively reduces the forecasting risk for a target variable. This underscores the theoretical benefit of using more relevant dimensions (channels), supporting our hypothesis that higher-dimensional data provides greater opportunity for CD models to outperform CI approaches.
% This result directly motivates the subsequent empirical investigation on benchmarks of varying dimensionality.

% The proof is provided in Appendix~\ref{app:Theorem Proof2}. This theorem establishes that including an additional informative and non-redundant channel strictly reduces the Bayes risk and increases the risk reduction gap. Consequently, incorporating more relevant dimensions can enhance forecasting performance and emphasizes the benefit of high-dimensional time series data. 

\subsection{Empirical Analysis}
\label{sec:empirical_analysis}

\begin{table}[ht]
\centering
\vspace{-1em}
\caption{MSE for CI and CD models under different settings.}
% \vspace{-0.5em}
\label{tab:mse_transposed}
\begin{tabular}{lccccc}
\toprule
\textbf{Model} & \textbf{Independent 100} & \textbf{Anti-Self 100} & \textbf{Anti-Self 250} & \textbf{Anti-Self 2000} \\
\midrule
CI & 0.0043 & 0.0052 & 0.0054 & 0.0054  \\
CD & 0.0066 & 0.0014 & 0.0012 & 0.0011  \\
\bottomrule
\end{tabular}
\vspace{-1em}
\end{table}

To empirically validate our theoretical insights above, we conduct controlled experiments on synthetic data for evaluating the performance difference between CI and CD modeling under varying channel-dependency structures and dimensionalities.

% we design toy experiments to evaluate the practical difference between CI and CD under different channel-dependency structures.

\textbf{Data Generation.} We generate synthetic multivariate time series using a $C$-dimensional Vector Autoregression process in Equation~\ref{eq:var_system}, where the coefficient matrix \( {\bf A} \in \mathbb{R}^{C \times C} \)  encodes the dependency structure of channels. 
We focus on two distinct structures for ${\bf A}$:     
(1) \textbf{Independent}: Channels evolve independently of each other  i.e., \( {\bf A} \) is a diagonal matrix. For this setting, we use $C=100$ channels.
(2) \textbf{Anti-Self}: 
Channels depend significantly on other channels but not on their own immediate past, i.e., \( {\bf A} \) has zero diagonal entries and non-zero off-diagonal entries.  To examine the impact of dimensionality, we generate Anti-Self datasets with $C\in\{100,250,2000\}$.
More experimental details are provided in Appendix~\ref{app:experiment_details_empirical_analysis}.

\textbf{Forecasting Models.} To isolate the core impact of utilizing versus disregarding channel dependencies, we employ two simple linear forecasting architectures: (1) \textbf{CI}: A single univariate linear model is fitted for each channel using only its own past values.  (2) \textbf{CD}: A linear model is applied along the channel dimension to incorporate inter-channel dependencies.

\textbf{Results and Discussion.} Table~\ref{tab:mse_transposed} shows that CI models perform better in the Independent setting without inter-channel dependencies, while CD models achieve markedly lower MSE in Anti-Self settings, with larger gains as the number of interacting channels grows. These results align with our theory that \textit{CD modeling becomes increasingly effective when significant inter-channel dependencies exist, especially in higher-dimensional systems}. We further confirm this on real-world data (Appendix~\ref{app:Empirical Validation for Channel Dependency Findings on Real-World Data}).
These findings underscore two critical needs: \textbf{(1)} \textbf{novel CD models} that efficiently and effectively navigate the intricacies of high-dimensional time series, and \textbf{(2)} \textbf{dedicated high-dimensional forecasting benchmarks }with complex channel dependencies for rigorous evaluation. Accordingly, in the subsequent sections, we address these needs by first introducing our \method{} framework specifically tailored for HDTSF and then presenting \dataset{}, a comprehensive benchmark designed to support and galvanize further research in this evolving domain.

% These observations highlight two critical needs: (1) we need an efficient CD model that can effecively handle the high-dimeinsaional time series data; and (2) we need the development of high-dimensional time series forecasting datasets with complex channel dependencies, which can enhance the evaluation of CD methods. In the subsequent sections, we first introduce our \method{} framework and then introduce our proposed high-dimensional benchmark designed with these considerations \dataset{}.

% These results support our theoretical claim: CD becomes increasingly effective when meaningful cross-channel structures are present. This observation further motivates the development of high-dimensional time series forecasting datasets with complex channel dependencies, which can enhance the evaluation of CD methods. In the following section, we introduce a high-dimensional time series forecasting benchmark, \dataset{}.

\begin{figure*}[t!]
    \centering
    \vspace{-1em}
    \includegraphics[width=0.99\textwidth]{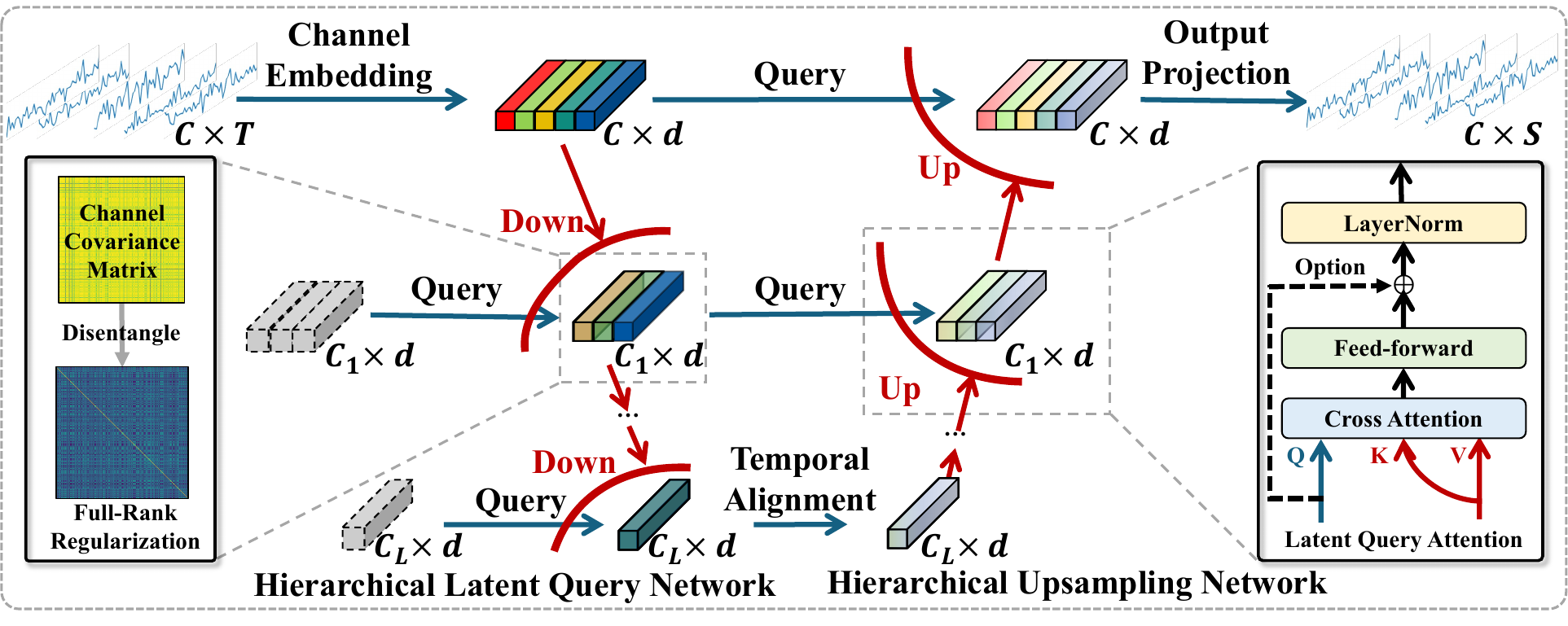}
    \vspace{-1em}
    \caption{Overall framework of \method{}, which consists of five main modules: channel embedding, hierarchical latent query, temporal alignment, hierarchical upsampling, and output projection.}
    \label{fig:method}
    \vspace{-1em}
\end{figure*}

\section{\method{} Framework}
\label{sec:UCast Framework}
Our theoretical and empirical analyses in Section~\ref{sec:Preliminary Study} highlight the need for novel CD models that can efficiently handle high-dimensional time series. Traditional models often struggle with scalability and overlook the structured nature of such data, especially the \textit{hierarchical} correlations among variable groups that frequently occur in real-world datasets (see Appendix~\ref{app:Empirical Evidence of Hierarchical Structures}). To address this, we propose \method{} (see Figure~\ref{fig:method}), an efficient model that captures channel correlations via learning latent hierarchical structures. We also introduce a full-rank regularization term to encourage disentanglement and improve the learning of structured representations.

\subsection{Model Architecture}
\textbf{Channel Embedding.}  
The input sequence is first normalized. A linear projection \(\mathbf{W}_{\text{in}}\in\mathbb{R}^{T\times d}\) converts the temporal dimension to a hidden dimension \(d\):
\begin{equation}
\mathbf{H}^{(0)}=\mathbf{X} \mathbf{W}_{\text{in}}
\in\mathbb{R}^{C\times d}.
\end{equation}

\textbf{Hierarchical Latent Query Network.}  
Self‑attention is well suited to modelling channel‑wise dependencies because every variable can attend to every other variable without assuming locality.  
However, applying full attention to \(C\) channels incurs quadratic cost, which is prohibitive when \(C\) is large.  
Hierarchical latent query network addresses this by introducing a set of \emph{latent queries} that serve as an information bottleneck.  
At layer \(\ell\in\{1,\dots,L\}\) the number of latent queries is
\(C_\ell=\lfloor C/r^{\,\ell}\rfloor\),
where \(r>1\) is a pre‑defined \emph{reduction ratio}.  
% A larger \(r\) yields stronger compression and lower cost; a smaller \(r\) preserves more detail.  
The queries, which are learnable, are shared across samples:
\(\mathbf{Q}_\ell\in\mathbb{R}^{C_\ell\times d}\).  
The Latent Query Attention at layer \(\ell\) is
\begin{equation}
\begin{array}{c}
\mathbf{Q}=W_q\mathbf{Q}_\ell,\quad
\mathbf{K}=W_k\mathbf{H}^{(\ell-1)},\quad
\mathbf{V}=W_v\mathbf{H}^{(\ell-1)},\\[6pt]
\mathbf{H}^{(\ell)}=
W_o\!\Bigl(
\operatorname{softmax}\!\bigl(\tfrac{\mathbf{Q}\mathbf{K}^{\!\top}}{\sqrt{d_h}}\bigr)\mathbf{V}
\Bigr)\in\mathbb{R}^{C_\ell\times d},
\end{array}
\end{equation}
followed by layer normalisation.  
Successive layers hierarchical latent query network builds a hierarchy in which higher‑level queries summarise wider channel groups.

\textbf{Temporal Alignment.}  
The deepest latent representation \(\mathbf{H}^{(L)}\in\mathbb{R}^{C_L\times d}\) stores \(C_L\) latent channel tokens whose \(d\)-dimensional feature vectors encode the \emph{temporal} dynamics extracted from the look‑back window.  
To keep these temporal features coherent when the model switches from down‑sampling to up‑sampling, we apply a shared linear predictor \(f_{\text{pred}}\) along the temporal dimension:
\begin{equation}
\mathbf{U}^{(L)} = f_{\text{pred}}\!\bigl(\mathbf{H}^{(L)}\bigr)\in\mathbb{R}^{C_L\times d}.
\end{equation}

\textbf{Hierarchical Upsampling Network.}  
Forecasting demands channel‑wise outputs of size \(C\), so the hierarchy is traversed in reverse to restore resolution.  
At layer \(\ell=L,\dots,1\) the Up‑Latent Query Attention uses the representation from the encoder as queries and the current decoder representation as keys and values:
\begin{equation}
\begin{array}{c}
\mathbf{Q}=\mathbf{W}_q\mathbf{H}^{(\ell-1)},\quad
\mathbf{K}=\mathbf{W}_k\mathbf{U}^{(\ell)},\quad
\mathbf{V}=\mathbf{W}_v\mathbf{U}^{(\ell)},\\[6pt]
\mathbf{U}^{(\ell-1)}=
\mathbf{W}_o\!\Bigl(
\operatorname{softmax}\!\bigl(\tfrac{\mathbf{Q}\mathbf{K}^{\!\top}}{\sqrt{d_h}}\bigr)\mathbf{V}
\Bigr)+\mathbf{H}^{(\ell-1)}\in\mathbb{R}^{C_{\ell-1}\times d}.
\end{array}
\end{equation}
The skip connection \(\mathbf{H}^{(\ell-1)}\) guides the reconstruction, ensuring that the original channel information is recovered with minimal distortion.

\textbf{Output Projection.}  
After upsampling \(\mathbf{U}^{(0)} \in\mathbb{R}^{C\times d}\) has the same channel dimension as the encoder output \(\mathbf{H}^{(0)}\).  
A residual link followed by a projection yields the horizon‑length prediction:
\begin{equation}
\hat{\mathbf{Y}}=
(\mathbf{U}^{(0)}+\mathbf{H}^{(0)})\mathbf{W}_{\text{out}},
\qquad
\mathbf{W}_{\text{out}}\in\mathbb{R}^{d\times S}.
\label{eq:output_projection}
\end{equation}
Finally, we apply the inverse of the initial normalisation to obtain the forecast in the original scale.

\subsection{Optimization Objective}
\label{sec:Optimization Objective}
% 1. 表征学习的核心在于disentangle解纠缠，高维数据的特性是不同channel之间相互关联纠缠在一起，我们构建一个良好的高维模型并且高效是需要能够实现disentanglement解纠缠
% 2. 从信息论的视角上解纠缠就是最小化互信息，满秩约束等价于最大化联合特征分布的熵，即每个channel都携带独特信息，整体信息量最大。这与信息瓶颈理论（Information Bottleneck）中的“最小冗余、最大相关”思想一致。

\textbf{Full-Rank Regularization.}
Representation learning is fundamentally about \emph{disentanglement}: channels are highly correlated in the high-dimensional time series data (see Table~\ref{tab:dataset}). This entanglement means that the latent matrix
\( \mathbf{H}^{(\ell-1)}\in\mathbb{R}^{C\times d} \)
often has rank
\( r=\operatorname{rank}(\mathbf{H}^{(\ell-1)})\ll C \).
The resulting rank deficiency signals redundancy among channels and obscures the latent hierarchical channel structure. Achieving disentanglement is therefore required to learn this structure.
To resolve this issue, we introduce a full-rank regularization, whose effect is formalised in Theorem~\ref{thm:fullrank} below. 

\begin{theorem}[Full-Rank Regularisation]\label{thm:fullrank}
Let \( \mathbf{H}^{(\ell-1)}\in\mathbb{R}^{C\times d} \) be of rank
\( r < \min(C,d) \).
Let \( \mathbf{Q}\in\mathbb{R}^{C'\times C} \) be learnable, full row rank, and
\( C'\le r \).
Define
\( \mathbf{H}^{(\ell)} = \mathbf{Q}\mathbf{H}^{(\ell-1)}
   \in\mathbb{R}^{C'\times d} \).
There exists a choice of \( \mathbf{Q} \) such that
\( \operatorname{rank}(\mathbf{H}^{(\ell)}) = C' \).
Further, if \( r \ge d \) and \( C' \ge d \),
adding a full-rank regulariser (e.g.\ the log-determinant of
\( \mathbf{H}^{(\ell)}\mathbf{H}^{(\ell)\!\top} \))
drives \( \mathbf{H}^{(\ell)} \) toward
row rank \( \min(C',d) \).
\end{theorem}

The proof is given in Appendix~\ref{app:Theorem Proof3}.
Theorem~\ref{thm:fullrank} states that enforcing full rank on
\( \mathbf{H}^{(\ell)} \)
is sufficient to remove linear redundancy among channels and reveal a clear  hierarchical latent channel structure. Define the row-covariance matrix
\(
\mathbf{\Sigma}^{(\ell)} =
   \tfrac{1}{d}\,\mathbf{H}^{(\ell)}\mathbf{H}^{(\ell)\!\top}
   \in\mathbb{R}^{C'\times C'}.
\)
The log-determinant
\(
\log\det\bigl(\mathbf{\Sigma}^{(\ell)}+\varepsilon I_{C'}\bigr)
\)
is proportional to the generalised variance\footnote{Product of the
eigenvalues of \(\mathbf{\Sigma}^{(\ell)}\).}.
Maximising this value keeps every eigenvalue bounded away from zero,
so the channel vectors occupy a larger subspace and share less
redundant information.
We thus define the full-rank regularization loss as  
\[
\mathcal{L}_{\mathrm{cov}}^{(\ell)} = -\frac{1}{C'} \log\det\bigl(\mathbf{\Sigma}^{(\ell)} + \varepsilon I_{C'}\bigr),
\]  
where \(\varepsilon I_{C'}\) ensures positive definiteness and stabilizes early training. The \(1/C'\) factor removes scale dependence, simplifying the weighting of this term. The overall loss is the average over all \(L\) layers. Minimizing \(\mathcal{L}_{\mathrm{cov}}\) increases the Shannon differential entropy (Theorem~\ref{thm:entropy}), reducing redundancy and promoting disentanglement by encouraging each channel to carry distinct information.

\textbf{Overall Objective.}  
The model parameters are learned by minimising the combined loss
\begin{equation}
\mathcal{L} = \mathcal{L}_{\text{mse}} + \alpha\,\mathcal{L}_{\text{cov}},
\end{equation}
where $\mathcal{L}_{\text{mse}}$ is the supervised mean‑squared forecasting error, $\mathcal{L}_{\text{cov}}$ is the layer‑averaged covariance penalty, and $\alpha$ controls the strength of this regulariser.

% \subsection{Computational Complexity}

% With the time axis pre‑compressed to width $d$, a flat channel Transformer (iTransformer) requires
% $
% \mathcal{T}_{\text{iTrans}}=\mathcal{O}\!\bigl(C^{2}d/h\bigr),
% \mathcal{M}_{\text{iTrans}}=\mathcal{O}(C^{2})
% $
% whereas \method{}, thanks to its latent‑query hierarchy with reduction ratio $r$, cuts both costs to
% $
% \mathcal{T}_{\text{\method{}}}=\mathcal{O}\!\bigl(C^{2}d/(h\,r)\bigr),
% \mathcal{M}_{\text{\method{}}}=\mathcal{O}\!\bigl(C^{2}/r\bigr),
% $
% achieving a uniform $1/r$ improvement (e.g., $16\times$ when $r=16$) in runtime and memory while retaining attention’s ability to model channel correlations.  
% We provide the detailed computational complexity analysis in Appendix~\ref{app:complexity}.

% \vspace{-1em}
\section{\dataset{}: High-Dimensional Time Series Forecasting Benchmark}
\label{sec:High-Dimensional Time Series Forecasting Benchmark}
% \vspace{-1em}

\begin{table*}[t!]
\centering
\vspace{-1em}
\caption{\dataset{} descriptions. Details on dataset generation, collection, cleaning, and prediction length determination are provided in Appendix~\ref{app:Dataset Description}.}
\setlength{\tabcolsep}{3pt}
\small
\scalebox{0.85}{%
\begin{tabular}{c|c|c|c|c|c|c|c}
\toprule
 \textbf{Dataset} & \textbf{Dimensions} & \textbf{Dataset Size} & \textbf{Frequency} & \textbf{Pred Length} & \textbf{Storage} & \textbf{Domains} & \textbf{Correlation} \\ 
 \midrule
  Neurolib           & 2,000          & 60,000      & 1 ms             & 336 & 2.44 GB & Neural Science & 0.926±0.007 \\ 
  \midrule
  Solar               & 5,162          & 105,120     & 5 Mins           & 336 & 2.33 GB & Energy         & 0.998±0.000 \\ 
  \midrule
  Atec                           & 1,569          & 8,928       & 10 Mins          & 336 & 158.8 MB & Cloud          & 0.851±0.036 \\ 
  \midrule
  Meter         & 2,898          & 28,512      & 30 Mins          & 336 & 651 MB & Energy         & 0.864±0.023 \\ 
  \midrule
  Temp                    & 3,850          & 17,544      & 1 Hour           & 168 & 383.6 MB & Weather        & 0.833±0.039 \\ 
  \midrule
  Wind                    & 3,850          & 17,544      & 1 Hour           & 168 & 331.4 MB & Weather        & 0.937±0.011 \\ 
  \midrule
  Traffic-CA                     & 7,491          & 43,824      & 1 Hour           & 168 & 2.48 GB & Traffic        & 0.962±0.009 \\ 
  \midrule
  Traffic-GLA                    & 3,376          & 43,824      & 1 Hour           & 168 & 1.12 GB & Traffic        & 0.947±0.012 \\ 
  \midrule
  Traffic-GBA                    & 2,229          & 43,824      & 1 Hour           & 168 & 689.1 MB & Traffic        & 0.978±0.003 \\ 
  \midrule 
  Air Quality            & 1,105          & 15,461      & 6 Hours          & 28  & 94.6 MB & Environment    & 0.854±0.030 \\ 
  \midrule
  SIRS              & 2,994          & 9,000       & 1 Day            & 7   & 675 MB & Epidemiology   & 0.991±0.002 \\ 
  \midrule
  SP500                & 1,475          & 7,553       & 1 Day            & 7   & 184.1 MB & Finance        & 0.738±0.046 \\ 
  \midrule
  M5                             & 3,049          & 1,941       & 1 Day            & 7   & 13.7 MB & Sale           & 0.726±0.073 \\ 
  \midrule
  Measles                & 1,161          & 1,330       & 1 Day            & 7   & 12.1 MB & Epidemiology   & 0.724±0.048 \\ 
  \midrule
  Wiki-20k   & 20,000         & 2,557       & 1 Day            & 7   & 289 MB & Web            & 0.923±0.017 \\ 
  \midrule
  Mobility      & 5,826          & 974         & 1 Day            & 7   & 29.7 MB & Social         & 0.847±0.027 \\ 
  \bottomrule
\end{tabular}
}
\label{tab:dataset}
\vspace{-1em}
\end{table*}

The goal of \dataset{} is to support and advance research in TSF, which is a rapidly growing and increasingly important area. 
As shown in Table~\ref{tab:dataset}, \dataset{} has several key characteristics:
\begin{compactenum}[\textbullet]
\item \textbf{High Dimensionality.} \dataset{} includes 16 high-dimensional time series forecasting datasets. The number of variables (dimensions) in these datasets ranges from 1,161 to 20,000, which is significantly larger than in commonly used benchmarks such as ETT, Weather, ECL, Solar, and Traffic (which typically have 7–862 channels). These high-dimensional datasets provide opportunities for developing scalable methods, enable exploration of channel-dependent models, and allow for more robust performance evaluations.

\item  \textbf{Diverse Sources.} \dataset{} comprises both simulated and real-world datasets. \textit{Neurolib} and \textit{SIRS} are simulated using domain-informed differential equations and are valuable for scientific modeling and hypothesis testing. The remaining datasets are based on real-world observations, making them suitable for evaluating the generalization ability of forecasting models in practical settings. The inclusion of both types enables comprehensive performance assessments.

\item \textbf{Varied Scales.} \dataset{} offers datasets of different sizes, defined by the number of variables and temporal length. Disk memory usage reflects this variability. As indicated in Table~\ref{tab:dataset}, there are 4 large-scale (gigabyte-level), 8 medium-scale (hundreds of megabytes), and 4 small-scale (tens of megabytes) datasets. The small and medium-scale datasets fit in the memory of a single GPU, making them suitable for evaluating resource-intensive models. The large-scale datasets support research on scalable methods using mini-batching and distributed training.

\item  \textbf{Different Sampling Frequencies.} \dataset{} covers a range of sampling frequencies, including milliseconds, minutes, hours, and days (see Table~\ref{tab:dataset}). This variety reflects real-world conditions across application domains and enables evaluation across different temporal resolutions. Furthermore, \dataset{} uses frequency-specific prediction lengths that are more realistic than fixed-length horizons commonly used in previous benchmarks.

\item \textbf{Broad Domain Coverage.} \dataset{} includes datasets from 10 diverse domains, such as neural science, energy, cloud computing, weather, traffic, epidemiology, finance, and social behavior. This diversity supports the development of general-purpose forecasting models and enables comparison against domain-specific approaches.

\end{compactenum}

\textit{Channel Correlation Statistic.} We compute channel correlation following TFB~\citep{qiu2024tfb}, with results reported in the \textit{correlation} column of Table~\ref{tab:dataset}. 
All datasets exhibit strong channel dependencies, which are not intentionally curated but naturally arise in large-scale real-world systems. The high correlation reflects complex inter-channel dependencies and a shared global trend across channels. 
Further analysis on channel correlation of \dataset{} is provided in Appendix~\ref{sec:Channel Correlation Statistic}.

\begin{table*}[t!]
\centering
\vspace{-1em}
\caption{Forecasting results on various datasets. Input length $T$ is tuned per model for best performance (see Appendix~\ref{sec:Implementation Details}), and prediction length $S$ is given in Table~\ref{tab:dataset}. Lower MSE/MAE indicates better performance; best results are in \textcolor{red}{\textbf{red}} and second-best are \textcolor{blue}{\underline{underlined}}.}
\resizebox{\textwidth}{!}{%
\begin{tabular}{crr |rr |rr |rr |rr |rr |rr |rr |>{\columncolor{gray!20}}r>{\columncolor{gray!20}}r}
\toprule
\multirow{2}{*}{\textbf{Models}} & \multicolumn{6}{c}{\textbf{Channel-Independent}} & \multicolumn{12}{c}{\textbf{Channel-Dependent}}\\
\cmidrule(lr){2-7} \cmidrule(lr){8-19}
& \multicolumn{2}{c}{\textbf{DLinear}} 
& \multicolumn{2}{c}{\textbf{PAttn}} 
& \multicolumn{2}{c|}{\textbf{PatchTST}} 
& \multicolumn{2}{c}{\textbf{iTransformer}} 
& \multicolumn{2}{c}{\textbf{TSMixer}} 
& \multicolumn{2}{c}{\textbf{TimesNet}} 
& \multicolumn{2}{c}{\textbf{CCM}} 
& \multicolumn{2}{c}{\textbf{DUET}} 
& \multicolumn{2}{>{\columncolor{gray!20}}c}{\textbf{\method{}}} \\
\cmidrule(lr){2-3} \cmidrule(lr){4-5} \cmidrule(lr){6-7} \cmidrule(lr){8-9} \cmidrule(lr){10-11} \cmidrule(lr){12-13} \cmidrule(lr){14-15} \cmidrule(lr){16-17} \cmidrule(lr){18-19}
Metrics & MSE & MAE & MSE & MAE & MSE & MAE & MSE & MAE & MSE & MAE & MSE & MAE & MSE & MAE & MSE & MAE & MSE & MAE \\
\midrule
Atec & 0.318 & 0.314 & 0.299 & \textcolor{red}{\textbf{0.275}} & \textcolor{blue}{\underline{0.298}} & 0.298 & 0.345 & 0.319 & 0.398 & 0.387 & 0.493 & 0.429 & 0.362 & 0.346 & 0.330 & 0.339 & \textcolor{red}{\textbf{0.287}} & \textcolor{blue}{\underline{0.280}} \\
Air Quality & 0.449 & 0.446 & 0.449 & 0.432 & 0.448 & 0.432 & \textcolor{blue}{\underline{0.447}} & \textcolor{blue}{\underline{0.431}} & 0.447 & 0.438 & 0.457 & 0.438 & 0.458 & 0.451 & 0.452 & 0.444 & \textcolor{red}{\textbf{0.446}} & \textcolor{red}{\textbf{0.430}} \\
Temp & 0.272 & 0.391 & 0.278 & 0.395 & 0.279 & 0.396 & \textcolor{blue}{\underline{0.265}} & 0.386 & 0.266 & \textcolor{blue}{\underline{0.389}} & 0.287 & 0.408 & 0.279 & 0.398 & 0.435 & 0.511 & \textcolor{red}{\textbf{0.262}} & \textcolor{red}{\textbf{0.383}} \\
Wind & 1.128 & \textcolor{blue}{\underline{0.697}} & 1.256 & 0.758 & 1.254 & 0.757 & \textcolor{blue}{\underline{1.116}} & 0.699 & 1.346 & 0.742 & 1.161 & 0.708 & 1.165 & 0.698 & 1.227 & 0.746 & \textcolor{red}{\textbf{1.104}} & \textcolor{red}{\textbf{0.692}} \\
Mobility & 0.344 & 0.359 & 0.337 & 0.336 & 0.344 & 0.341 & \textcolor{red}{\textbf{0.312}} & \textcolor{red}{\textbf{0.314}} & 1.165 & 0.787 & 0.410 & 0.388 & 0.523 & 0.468 & 0.439 & 0.410 & \textcolor{blue}{\underline{0.315}} & \textcolor{blue}{\underline{0.317}} \\
Traffic-CA & \textcolor{blue}{\underline{0.063}} & \textcolor{blue}{\underline{0.141}} & 0.491 & 0.554 & 0.295 & 0.417 & 0.271 & 0.391 & 0.082 & 0.186 & 0.101 & 0.205 & 0.067 & 0.153 & --- & --- & \textcolor{red}{\textbf{0.061}} & \textcolor{red}{\textbf{0.131}} \\
Traffic-GBA & 0.062 & 0.137 & \textcolor{blue}{\underline{0.059}} & \textcolor{blue}{\underline{0.132}} & 0.060 & 0.139 & 0.063 & 0.141 & 0.074 & 0.168 & 0.098 & 0.204 & 0.067 & 0.151 & 0.650 & 0.638 & \textcolor{red}{\textbf{0.059}} & \textcolor{red}{\textbf{0.126}} \\
Traffic-GLA & 0.062 & 0.142 & \textcolor{blue}{\underline{0.060}} & 0.136 & 0.060 & \textcolor{blue}{\underline{0.136}} & 0.065 & 0.145 & 0.071 & 0.168 & 0.094 & 0.199 & 0.066 & 0.152 & 0.810 & 0.738 & \textcolor{red}{\textbf{0.060}} & \textcolor{red}{\textbf{0.132}} \\
M5 & 3.688 & 0.870 & 3.650 & 0.867 & 3.655 & 0.872 & \textcolor{blue}{\underline{3.549}} & \textcolor{blue}{\underline{0.853}} & 6.863 & 1.623 & 4.490 & 0.919 & 4.916 & 0.941 & 3.768 & 0.880 & \textcolor{red}{\textbf{3.501}} & \textcolor{red}{\textbf{0.849}} \\
Measles & 0.128 & 0.252 & 0.011 & 0.048 & 0.013 & 0.058 & \textcolor{blue}{\underline{0.010}} & \textcolor{blue}{\underline{0.048}} & 0.569 & 0.547 & 0.018 & 0.060 & 0.244 & 0.323 & 0.015 & 0.064 & \textcolor{red}{\textbf{0.010}} & \textcolor{red}{\textbf{0.042}} \\
Neurolib & 1.793 & 0.381 & 2.458 & 0.445 & 2.395 & 0.438 & \textcolor{red}{\textbf{1.718}} & \textcolor{red}{\textbf{0.347}} & 2.240 & 0.532 & 2.475 & 0.458 & 1.774 & 0.403 & 2.519 & 0.451 & \textcolor{blue}{\underline{1.750}} & \textcolor{blue}{\underline{0.350}} \\
Solar & 0.174 & 0.255 & 0.604 & 0.582 & 0.416 & 0.469 & 0.343 & 0.427 & \textcolor{red}{\textbf{0.155}} & \textcolor{red}{\textbf{0.216}} & \textcolor{blue}{\underline{0.157}} & \textcolor{blue}{\underline{0.224}} & 0.177 & 0.258 & --- & --- & 0.172 & 0.246 \\
SIRS & 0.058 & 0.168 & \textcolor{blue}{\underline{0.025}} & \textcolor{blue}{\underline{0.109}} & 0.033 & 0.129 & 0.028 & 0.113 & 0.016 & 0.078 & 0.162 & 0.327 & 0.048 & 0.156 & 0.095 & 0.236 & \textcolor{red}{\textbf{0.007}} & \textcolor{red}{\textbf{0.052}} \\
Meters & 0.944 & \textcolor{red}{\textbf{0.549}} & \textcolor{red}{\textbf{0.941}} & 0.552 & 1.254 & 0.706 & 0.949 & 0.556 & 0.987 & 0.564 & 1.034 & 0.586 & 0.946 & 0.551 & 1.308 & 0.731 & \textcolor{blue}{\underline{0.943}} & \textcolor{blue}{\underline{0.551}} \\
SP500 & 0.630 & 0.367 & 0.516 & 0.309 & 0.523 & 0.313 & \textcolor{blue}{\underline{0.511}} & \textcolor{blue}{\underline{0.306}} & 2.674 & 1.120 & 0.611 & 0.343 & 0.727 & 0.414 & 0.568 & 0.335 & \textcolor{red}{\textbf{0.502}} & \textcolor{red}{\textbf{0.301}} \\
Wiki-20k & 10.740 & 0.394 & 10.290 & 0.306 & \textcolor{blue}{\underline{10.291}} & \textcolor{blue}{\underline{0.305}} & 10.933 & 0.405 & 10.446 & 0.332 & 10.586 & 0.325 & 11.413 & 0.373 & 10.278 & 0.304 & \textcolor{red}{\textbf{10.273}} & \textcolor{red}{\textbf{0.302}} \\
\midrule
\textbf{1st Count} & 0 & 1 & 1 & 1 & 0 & 0 & \textcolor{blue}{\underline{2}} & \textcolor{blue}{\underline{2}} & 1 & 1 & 0 & 0 & 0 & 0 & 0 & 0 & \textcolor{red}{\textbf{12}} & \textcolor{red}{\textbf{11}} \\
% 2nd Count & 1 & 2 & 4 & 3 & 2 & 2 & \textbf{5} & \textbf{3} & 0 & 1 & 1 & 1 & 0 & 0 & 0 & 0 & \textbf{3} & \textbf{4} \\
\bottomrule
\end{tabular}%
}
\vspace{-1.2em}
\label{tab:dataset_results}
\end{table*}

\vspace{-0.8em}
\section{Experiment Results and Framework Analysis}
\vspace{-0.7em}

\textbf{Baselines.} We consider the following representative models as our baselines for HDTSF:
\begin{compactenum}[(a)]
% \item \textbf{No channel correlation modeling.}  
% Each channel is assigned an independent DLinear~\citep{dlinear} model, referred to as \textbf{DLinear-I}.

\item \textbf{Channel-Independent.}  
We include three methods that employ a shared model backbone across all channels: \textbf{DLinear}~\citep{dlinear}, \textbf{PAttn}~\citep{pattn}, and \textbf{PatchTST}~\citep{patchtst}. These methods use linear projection, attention mechanisms, and the Transformer architecture, respectively, to model temporal dependencies.

\item \textbf{Channel-Dependent.}  
We adopt five methods that incorporate dedicated modules to capture inter-channel dependencies. \textbf{TimesNet}~\citep{timesnet} and \textbf{TSMixer}~\citep{tsmixer} are position-wise methods that capture channel correlations at each time step by embedding. \textbf{iTransformer}~\citep{itransformer} is a token-wise method that treats the entire time series of each channel as a token and models inter-channel dependencies using a Transformer. Additionally, we consider two clustering-based methods: \textbf{CCM}~\citep{ccm} and \textbf{DUET}~\citep{duet}, which explicitly group highly correlated channels into clusters and model dependencies within each cluster.

\end{compactenum}

\begin{wrapfigure}{r}{0.4\textwidth}
    \centering
    \vspace{-2em}
    \includegraphics[width=1\linewidth]{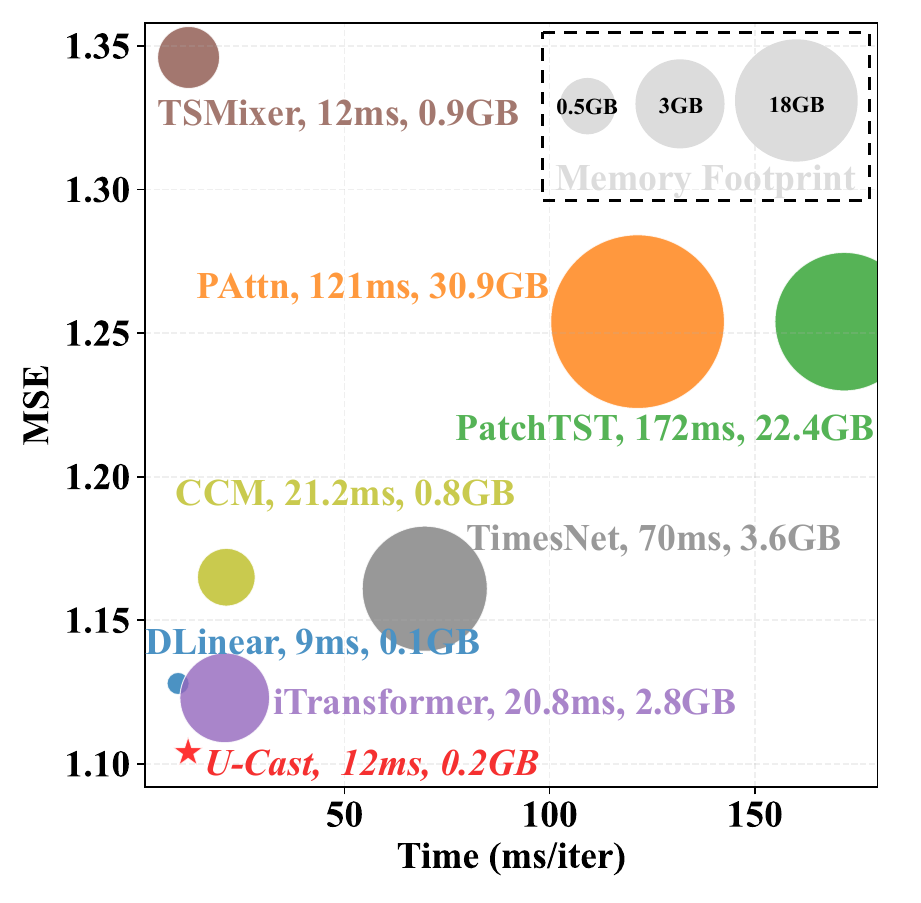}
    \vspace{-2em}
    \caption{MSE vs. average training time per batch on Wind dataset ($C = 3{,}850$); marker size reflects GPU memory usage.}
    \label{fig:efficiency}
    \vspace{-1em}
\end{wrapfigure}

\textbf{Benchmark Results.} 
The experimental setup is detailed in Appendix~\ref{sec:Implementation Details}. The results are shown in Table~\ref{tab:dataset_results}. We summarize the following observations from our benchmark results:
\textbf{First}, although CI methods cannot explicitly model channel correlations, they are designed to learn shared temporal patterns across channels. This can help reduce the risk of overfitting, enabling CI methods to achieve competitive performance. Among them, PAttn generally achieves better results, indicating the effectiveness of attention mechanisms for temporal modeling.
\textbf{Second}, among CD methods, iTransformer significantly outperforms others, suggesting that treating each channel as a token and using attention to model inter-channel dependencies is more effective than position-wise or cluster-based approaches in high-dimensional settings. We argue that the token-wise design helps mitigate noise in the temporal dimension and fully utilizes the benefits of the attention mechanism.
\textbf{Third}, CD methods do not show a consistent or significant performance advantage over CI methods. However, our theoretical analysis in Section~\ref{sec:Preliminary Study} suggests that CD can reduce the MSE risk compared to CI in multivariate time series forecasting. We believe this is due to current CD baselines either failing to effectively exploit complex channel dependencies or being unable not scale to high-dimensional time series.

% Based on these observations, there is a clear need for an effective and efficient model tailored for high-dimensional time series forecasting, one that fully exploits multivariate correlations.

\textbf{Performance Gains of \method{}.}  
As shown in Table~\ref{tab:dataset_results}, \method{} achieves top performance on the high-dimensional \dataset{} benchmark, ranking \textbf{first} on 12 datasets for MSE and 11 for MAE, and further achieves an average forecasting error reduction of
15\% compared to iTransformer (t-test p-value=1.34*10$^{-5}$).  
While iTransformer, the representative CD model that explicitly models multivariate correlations via Transformer, serves as a strong baseline, it underperforms on many datasets compared to \method{}. This highlights the advantage of learning hierarchical latent channel structures for capturing complex inter-channel dependencies.  
The limitation of iTransformer may stem from the extreme complexity of channel dependencies in high-dimensional settings, where the absence of explicitly hierarchical structure learning in its attention mechanism constrains its effectiveness.  
In contrast, by incorporating hierarchical latent queries and full-rank regularization, \method{} is better equipped to handle such complexity.  

\textbf{Model Efficiency.}
Beyond its strong forecasting performance, another notable advantage of \method{} is its lightweight architecture. Figure~\ref{fig:efficiency} shows \method{} achieves the best trade-off between training speed, memory footprint, and performance. We also find that \method{} consistently achieves favorable training efficiency across varying dimensionalities and as dimensionality increases, the efficiency advantage of \method{} over other methods becomes more pronounced (see Appendix~\ref{sec:efficiency}). Additionally, we conduct computational complexity analysis, which indicates that \method{} lowers both time and memory by a factor of $r$ compared with iTransformer, yet still retains the expressive power of attention through its latent‑query hierarchy (see Appendix~\ref{app:complexity}).

\begin{wraptable}{r}{0.4\textwidth}
\centering
\vspace{-2em}
\caption{Ablation study on different components of \method{}. The results are averaged over 16 datasets, with the complete results provided in Table~\ref{tab:ablation_full}.}
\resizebox{0.35\textwidth}{!}{%
\begin{tabular}{l|cc}
\toprule
\textbf{Variant} & \textbf{MSE} & \textbf{MAE} \\
\midrule
\rowcolor{gray!20}
\method{} & \textcolor{red}{\textbf{1.243}} & \textcolor{red}{\textbf{0.326}} \\
w/o $\mathcal{L}_{\text{cov}}$ & 1.267 & 0.341 \\
w/o Hierarchical & 1.263 & 0.332 \\
w/o Latent Query & 1.260 & 0.331 \\
w/o Upsampling & 1.269 & 0.336 \\
\bottomrule
\end{tabular}%
}
\label{tab:ablation}
\vspace{-1em}
\end{wraptable}

\textbf{Ablation Study.}
We evaluate the impact of: (1) \textbf{w/o hierarchical} by retaining only a single layer for dimensionality reduction, (2) \textbf{w/o latent query} by setting $\mathbf{Q}_\ell$ requires\_grad=False, and (3)\textbf{w/o upsampling} by using a simple linear projection to restore channel dimension. In addition, we examine the role of the covariance full-rank regularisation by \textbf{w/o $\mathcal{L}_{\text{cov}}$}, i.e., setting $\alpha=0$. Table~\ref{tab:ablation} shows that removing any component degrades \method{}'s averaged performance, confirming their necessity. For the full results on all datasets and deeper analysis, we refer reader to Appendix~\ref{sec:app_ablation_study}.

\textbf{Does full-rank regularization $\mathcal{L}_{\text{cov}}$ disentangle channels correlation effectively?}
As shown in Figure~\ref{fig:case}~(a), the top four subplots visualize the evolution of the covariance matrix $\mathbf{\Sigma}$ from a randomly initialized state (Epoch 0) to a well-optimized one (Epoch 10). The structure of the covariance matrix changes significantly across epochs, transitioning from dense to more sparse. This indicates that $\mathcal{L}_{\text{cov}}$ is effectively promoting disentanglement by reducing redundancy among channels. 

\begin{wrapfigure}{r}{0.65\textwidth}
    \centering
    \vspace{-1em}
    \includegraphics[width=1\linewidth]{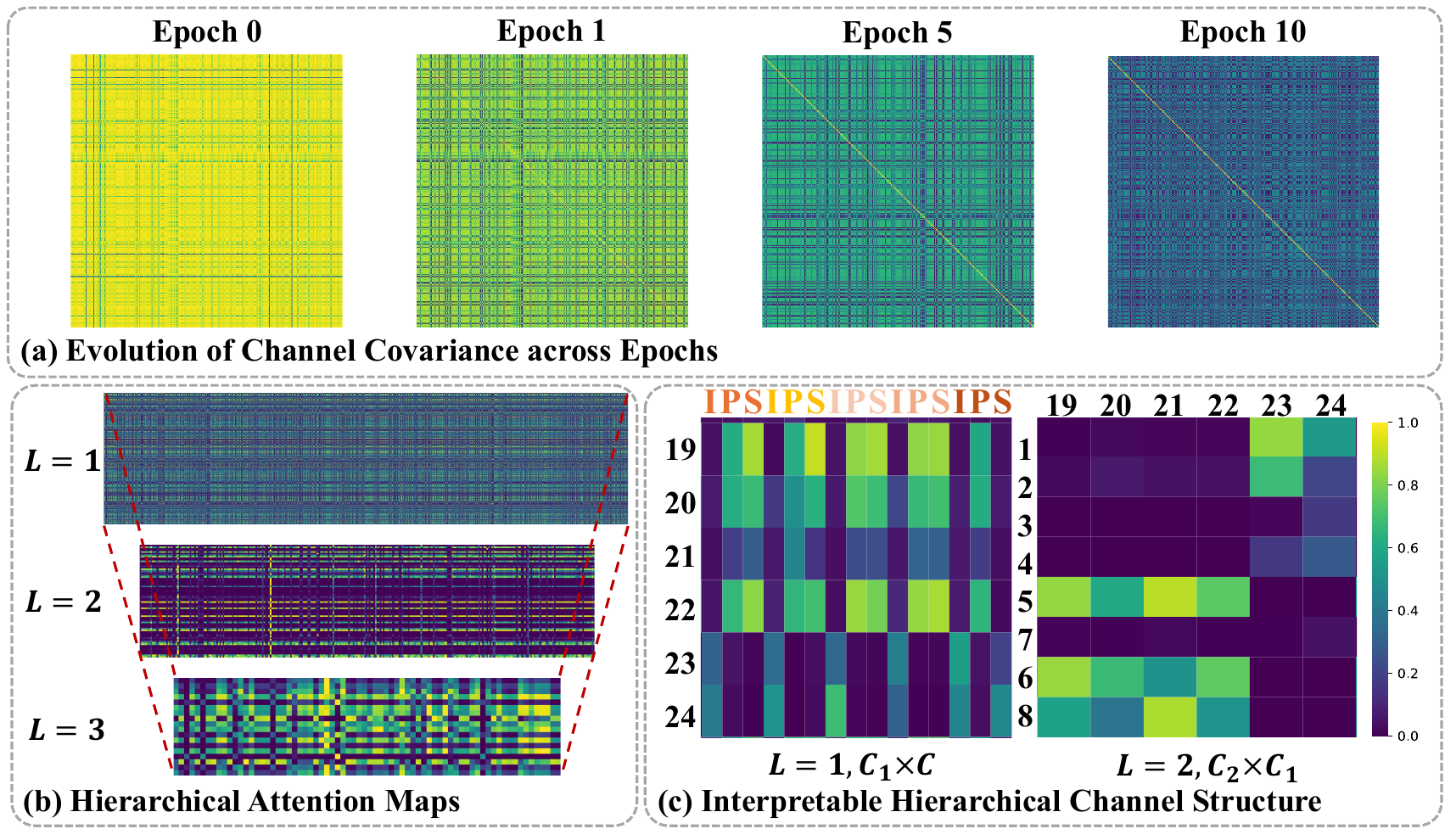}
    \vspace{-2em}
    \caption{A showcase on Measles dataset of full-rank regularization and latent hierarchical channel structure learning.}
    \vspace{-1em}
    \label{fig:case}
\end{wrapfigure}

\textbf{Can \method{} learn latent hierarchical structure among channels?}
As shown in Figure~\ref{fig:case}~(b), we visualize the attention maps at different layers. The attention focuses vary across layers, reflecting a latent hierarchical structure in the learned representations. Furthermore, assigning dimension reduction and multivariate correlation modeling to the attention mechanism improves the interpretability of attention maps. Figure~\ref{fig:case}~(c) shows attention maps from the hierarchical latent query network on the Measles dataset ($C = 1161$, 387 regions × 3 features: \textbf{I}–Inferred Infections, \textbf{P}–Population, \textbf{S}–Suspected Cases). Five regions are selected (colors indicate regions). At $L=1$, the model reduces to $C_1 = 32$ latent dimensions (visualizing 19–24), and at $L=2$, to $C_2 = 8$. At $L=1$, dimensions 19–22 mainly attend to P and S, while 23–24 focus on I. At $L=2$, dimensions 1–4 attend to outputs from 19–22, and 5–8 to those from 23–24. This pattern reflects a learned hierarchy: the model first separates features within regions, then integrates related features across regions, revealing a meaningful latent channel structure.

% \vspace{-0.5em}
\section{Conclusion}
% \vspace{-0.5em}
In this paper, we propose \method{} to address the unique challenges of effectively modeling these latent hierarchical channel structures within High-Dimensional Time Series Forecasting (HDTSF). To facilitate research in HDTSF and provide a testbed for future innovative approaches, we curate the \dataset{} benchmark. 
Through extensive experiments, we identify and highlight key research directions and opportunities in Appendix~\ref{app:Future Opportunity}. It is our aspiration that \method{}, \dataset{}, \textsc{Time-HD-lib}, and their accompanying resources will serve as a significant catalyst for innovation and progress in the time series community (see Appendix~\ref{app:broader_impact} for broader impact).

\bibliography{iclr2026_conference}

\begin{thebibliography}{60}
\providecommand{\natexlab}[1]{#1}
\providecommand{\url}[1]{\texttt{#1}}
\expandafter\ifx\csname urlstyle\endcsname\relax
  \providecommand{\doi}[1]{doi: #1}\else
  \providecommand{\doi}{doi: \begingroup \urlstyle{rm}\Url}\fi

\bibitem[Ansari et~al.(2024)Ansari, Stella, Turkmen, Zhang, Mercado, Shen, Shchur, Rangapuram, Arango, Kapoor, et~al.]{chronos}
Abdul~Fatir Ansari, Lorenzo Stella, Caner Turkmen, Xiyuan Zhang, Pedro Mercado, Huibin Shen, Oleksandr Shchur, Syama~Sundar Rangapuram, Sebastian~Pineda Arango, Shubham Kapoor, et~al.
\newblock Chronos: Learning the language of time series.
\newblock \emph{arXiv preprint arXiv:2403.07815}, 2024.

\bibitem[Bauer et~al.(2021)Bauer, Z{\"u}fle, Eismann, Grohmann, Herbst, and Kounev]{libra}
Andr{\'e} Bauer, Marwin Z{\"u}fle, Simon Eismann, Johannes Grohmann, Nikolas Herbst, and Samuel Kounev.
\newblock Libra: A benchmark for time series forecasting methods.
\newblock In \emph{Proceedings of the ACM/SPEC International Conference on Performance Engineering}, pp.\  189--200, 2021.

\bibitem[Cai et~al.(2015)Cai, Kang, Banerjee, and Wang]{cai2015stochastic}
Yongli Cai, Yun Kang, Malay Banerjee, and Weiming Wang.
\newblock A stochastic sirs epidemic model with infectious force under intervention strategies.
\newblock \emph{Journal of Differential Equations}, 259\penalty0 (12):\penalty0 7463--7502, 2015.

\bibitem[Cakan et~al.(2021)Cakan, Jajcay, and Obermayer]{cakan2021}
Caglar Cakan, Nikola Jajcay, and Klaus Obermayer.
\newblock neurolib: A simulation framework for whole-brain neural mass modeling.
\newblock \emph{Cognitive Computation}, Oct 2021.
\newblock ISSN 1866-9964.
\newblock \doi{10.1007/s12559-021-09931-9}.
\newblock URL \url{https://doi.org/10.1007/s12559-021-09931-9}.

\bibitem[Chatfield(2000)]{chatfield2000time}
Chris Chatfield.
\newblock \emph{Time-series forecasting}.
\newblock Chapman and Hall/CRC, 2000.

\bibitem[Chen et~al.(2024)Chen, Lenssen, Feng, Hu, Fey, Tassiulas, Leskovec, and Ying]{ccm}
Jialin Chen, Jan~Eric Lenssen, Aosong Feng, Weihua Hu, Matthias Fey, Leandros Tassiulas, Jure Leskovec, and Rex Ying.
\newblock From similarity to superiority: Channel clustering for time series forecasting.
\newblock \emph{Advances in Neural Information Processing Systems}, 37:\penalty0 130635--130663, 2024.

\bibitem[Chen et~al.(2023)Chen, Li, Yoder, Arik, and Pfister]{tsmixer}
Si-An Chen, Chun-Liang Li, Nate Yoder, Sercan~O Arik, and Tomas Pfister.
\newblock Tsmixer: An all-mlp architecture for time series forecasting.
\newblock \emph{arXiv preprint arXiv:2303.06053}, 2023.

\bibitem[Cohen et~al.(2009)Cohen, Huang, Chen, Benesty, Benesty, Chen, Huang, and Cohen]{cohen2009pearson}
Israel Cohen, Yiteng Huang, Jingdong Chen, Jacob Benesty, Jacob Benesty, Jingdong Chen, Yiteng Huang, and Israel Cohen.
\newblock Pearson correlation coefficient.
\newblock \emph{Noise reduction in speech processing}, pp.\  1--4, 2009.

\bibitem[Craddock et~al.(2013)Craddock, Benhajali, Chu, Chouinard, Evans, Jakab, Khundrakpam, Lewis, Li, Milham, et~al.]{craddock2013neuro}
Cameron Craddock, Yassine Benhajali, Carlton Chu, Francois Chouinard, Alan Evans, Andr{\'a}s Jakab, Budhachandra~Singh Khundrakpam, John~David Lewis, Qingyang Li, Michael Milham, et~al.
\newblock The neuro bureau preprocessing initiative: open sharing of preprocessed neuroimaging data and derivatives.
\newblock \emph{Frontiers in Neuroinformatics}, 7\penalty0 (27):\penalty0 5, 2013.

\bibitem[Dai et~al.(2024)Dai, Wu, Liu, Li, Bao, Jiang, and Xia]{pdf}
Tao Dai, Beiliang Wu, Peiyuan Liu, Naiqi Li, Jigang Bao, Yong Jiang, and Shu-Tao Xia.
\newblock Periodicity decoupling framework for long-term series forecasting.
\newblock In \emph{The Twelfth International Conference on Learning Representations}, 2024.

\bibitem[Das et~al.(2024)Das, Kong, Sen, and Zhou]{timefm}
Abhimanyu Das, Weihao Kong, Rajat Sen, and Yichen Zhou.
\newblock A decoder-only foundation model for time-series forecasting.
\newblock In \emph{Forty-first International Conference on Machine Learning}, 2024.

\bibitem[De~Gooijer \& Hyndman(2006)De~Gooijer and Hyndman]{de200625}
Jan~G De~Gooijer and Rob~J Hyndman.
\newblock 25 years of time series forecasting.
\newblock \emph{International journal of forecasting}, 22\penalty0 (3):\penalty0 443--473, 2006.

\bibitem[Godahewa et~al.(2021)Godahewa, Bergmeir, Webb, Hyndman, and Montero-Manso]{monash}
Rakshitha Godahewa, Christoph Bergmeir, Geoffrey~I Webb, Rob~J Hyndman, and Pablo Montero-Manso.
\newblock Monash time series forecasting archive.
\newblock \emph{arXiv preprint arXiv:2105.06643}, 2021.

\bibitem[Google(2020)]{GoogleMobility}
Google.
\newblock Google covid-19 community mobility reports, 2020.
\newblock URL \url{https://www.google.com/covid19/mobility/}.
\newblock Accessed: <date>.

\bibitem[Goswami et~al.(2024)Goswami, Szafer, Choudhry, Cai, Li, and Dubrawski]{moment}
Mononito Goswami, Konrad Szafer, Arjun Choudhry, Yifu Cai, Shuo Li, and Artur Dubrawski.
\newblock Moment: A family of open time-series foundation models.
\newblock \emph{arXiv preprint arXiv:2402.03885}, 2024.

\bibitem[Granger \& Newbold(2014)Granger and Newbold]{granger2014forecasting}
Clive William~John Granger and Paul Newbold.
\newblock \emph{Forecasting economic time series}.
\newblock Academic press, 2014.

\bibitem[Howard et~al.(2020)Howard, inversion, Makridakis, and vangelis]{m5-forecasting-accuracy}
Addison Howard, inversion, Spyros Makridakis, and vangelis.
\newblock M5 forecasting - accuracy.
\newblock \url{https://kaggle.com/competitions/m5-forecasting-accuracy}, 2020.
\newblock Kaggle.

\bibitem[Kingma \& Ba(2014)Kingma and Ba]{kingma2014adam}
Diederik~P Kingma and Jimmy Ba.
\newblock Adam: A method for stochastic optimization.
\newblock \emph{arXiv preprint arXiv:1412.6980}, 2014.

\bibitem[Kumar \& Toshniwal(2016)Kumar and Toshniwal]{kumar2016analysis}
Sachin Kumar and Durga Toshniwal.
\newblock Analysis of hourly road accident counts using hierarchical clustering and cophenetic correlation coefficient (cpcc).
\newblock \emph{Journal of Big Data}, 3\penalty0 (1):\penalty0 13, 2016.

\bibitem[Lai et~al.(2018)Lai, Chang, Yang, and Liu]{lstnet}
Guokun Lai, Wei-Cheng Chang, Yiming Yang, and Hanxiao Liu.
\newblock Modeling long-and short-term temporal patterns with deep neural networks.
\newblock In \emph{The 41st international ACM SIGIR conference on research \& development in information retrieval}, pp.\  95--104, 2018.

\bibitem[Li \& Anastasiu(2025)Li and Anastasiu]{li2025mc}
Yanhong Li and David~C Anastasiu.
\newblock Mc-ann: A mixture clustering-based attention neural network for time series forecasting.
\newblock \emph{IEEE Transactions on Pattern Analysis and Machine Intelligence}, 2025.

\bibitem[Li et~al.(2023)Li, Xu, and Anastasiu]{li2023extreme}
Yanhong Li, Jack Xu, and David~C Anastasiu.
\newblock An extreme-adaptive time series prediction model based on probability-enhanced lstm neural networks.
\newblock In \emph{Proceedings of the AAAI Conference on Artificial Intelligence}, volume~37, pp.\  8684--8691, 2023.

\bibitem[Liang et~al.(2022)Liang, Shao, Wang, Zhang, Sun, and Xu]{basicts}
Yubo Liang, Zezhi Shao, Fei Wang, Zhao Zhang, Tao Sun, and Yongjun Xu.
\newblock Basicts: An open source fair multivariate time series prediction benchmark.
\newblock In \emph{International symposium on benchmarking, measuring and optimization}, pp.\  87--101. Springer, 2022.

\bibitem[Lim \& Zohren(2021)Lim and Zohren]{lim2021time}
Bryan Lim and Stefan Zohren.
\newblock Time-series forecasting with deep learning: a survey.
\newblock \emph{Philosophical Transactions of the Royal Society A}, 379\penalty0 (2194):\penalty0 20200209, 2021.

\bibitem[Lin et~al.(2024)Lin, Lin, Wu, Chen, and Yang]{sparsetsf}
Shengsheng Lin, Weiwei Lin, Wentai Wu, Haojun Chen, and Junjie Yang.
\newblock Sparsetsf: Modeling long-term time series forecasting with 1k parameters.
\newblock \emph{arXiv preprint arXiv:2405.00946}, 2024.

\bibitem[Liu et~al.(2023)Liu, Xia, Liang, Hu, Wang, Bai, Huang, Liu, Hooi, and Zimmermann]{liu2023largest}
Xu~Liu, Yutong Xia, Yuxuan Liang, Junfeng Hu, Yiwei Wang, Lei Bai, Chao Huang, Zhenguang Liu, Bryan Hooi, and Roger Zimmermann.
\newblock Largest: A benchmark dataset for large-scale traffic forecasting.
\newblock \emph{Advances in Neural Information Processing Systems}, 36:\penalty0 75354--75371, 2023.

\bibitem[Liu et~al.(2022)Liu, Wu, Wang, and Long]{liu2022non}
Yong Liu, Haixu Wu, Jianmin Wang, and Mingsheng Long.
\newblock Non-stationary transformers: Exploring the stationarity in time series forecasting.
\newblock \emph{Advances in neural information processing systems}, 35:\penalty0 9881--9893, 2022.

\bibitem[Liu et~al.(2024{\natexlab{a}})Liu, Hu, Zhang, Wu, Wang, Ma, and Long]{itransformer}
Yong Liu, Tengge Hu, Haoran Zhang, Haixu Wu, Shiyu Wang, Lintao Ma, and Mingsheng Long.
\newblock itransformer: Inverted transformers are effective for time series forecasting.
\newblock In \emph{The Twelfth International Conference on Learning Representations}, 2024{\natexlab{a}}.

\bibitem[Liu et~al.(2024{\natexlab{b}})Liu, Li, Wei, Wan, Lau, and Jin]{liu2024epilearn}
Zewen Liu, Yunxiao Li, Mingyang Wei, Guancheng Wan, Max~SY Lau, and Wei Jin.
\newblock Epilearn: A python library for machine learning in epidemic modeling.
\newblock \emph{arXiv preprint arXiv:2406.06016}, 2024{\natexlab{b}}.

\bibitem[Liu et~al.(2025)Liu, Ni, Lau, and Jin]{liu2025cape}
Zewen Liu, Juntong Ni, Max~SY Lau, and Wei Jin.
\newblock Cape: Covariate-adjusted pre-training for epidemic time series forecasting.
\newblock \emph{arXiv preprint arXiv:2502.03393}, 2025.

\bibitem[Lubba et~al.(2019)Lubba, Sethi, Knaute, Schultz, Fulcher, and Jones]{lubba2019catch22}
Carl~H Lubba, Sarab~S Sethi, Philip Knaute, Simon~R Schultz, Ben~D Fulcher, and Nick~S Jones.
\newblock catch22: Canonical time-series characteristics: Selected through highly comparative time-series analysis.
\newblock \emph{Data mining and knowledge discovery}, 33\penalty0 (6):\penalty0 1821--1852, 2019.

\bibitem[Luo \& Wang(2024)Luo and Wang]{moderntcn}
Donghao Luo and Xue Wang.
\newblock Moderntcn: A modern pure convolution structure for general time series analysis.
\newblock In \emph{The Twelfth International Conference on Learning Representations}, 2024.

\bibitem[Madden et~al.(2024)Madden, Jin, Lopman, Zufle, Dalziel, E.~Metcalf, Grenfell, and Lau]{madden2024deep}
Wyatt~G Madden, Wei Jin, Benjamin Lopman, Andreas Zufle, Benjamin Dalziel, C~Jessica E.~Metcalf, Bryan~T Grenfell, and Max~SY Lau.
\newblock Deep neural networks for endemic measles dynamics: Comparative analysis and integration with mechanistic models.
\newblock \emph{PLOS Computational Biology}, 20\penalty0 (11):\penalty0 e1012616, 2024.

\bibitem[Makridakis \& Hibon(2000)Makridakis and Hibon]{m3}
Spyros Makridakis and Michele Hibon.
\newblock The m3-competition: results, conclusions and implications.
\newblock \emph{International journal of forecasting}, 16\penalty0 (4):\penalty0 451--476, 2000.

\bibitem[Makridakis et~al.(2018)Makridakis, Spiliotis, and Assimakopoulos]{m4}
Spyros Makridakis, Evangelos Spiliotis, and Vassilios Assimakopoulos.
\newblock The m4 competition: Results, findings, conclusion and way forward.
\newblock \emph{International Journal of forecasting}, 34\penalty0 (4):\penalty0 802--808, 2018.

\bibitem[Ni et~al.(2025)Ni, Liu, Wang, Jin, and Jin]{timedistill}
Juntong Ni, Zewen Liu, Shiyu Wang, Ming Jin, and Wei Jin.
\newblock Timedistill: Efficient long-term time series forecasting with mlp via cross-architecture distillation.
\newblock \emph{arXiv preprint arXiv:2502.15016}, 2025.

\bibitem[Nie et~al.(2023)Nie, Nguyen, Sinthong, and Kalagnanam]{patchtst}
Yuqi Nie, Nam~H Nguyen, Phanwadee Sinthong, and Jayant Kalagnanam.
\newblock A time series is worth 64 words: Long-term forecasting with transformers.
\newblock In \emph{The Eleventh International Conference on Learning Representations}, 2023.

\bibitem[Oreshkin et~al.(2020)Oreshkin, Carpov, Chapados, and Bengio]{nbeats}
Boris~N. Oreshkin, Dmitri Carpov, Nicolas Chapados, and Yoshua Bengio.
\newblock N-beats: Neural basis expansion analysis for interpretable time series forecasting.
\newblock In \emph{International Conference on Learning Representations}, 2020.

\bibitem[Paszke(2019)]{paszke2019pytorch}
A~Paszke.
\newblock Pytorch: An imperative style, high-performance deep learning library.
\newblock \emph{arXiv preprint arXiv:1912.01703}, 2019.

\bibitem[Qiu et~al.(2024{\natexlab{a}})Qiu, Hu, Zhou, Wu, Du, Zhang, Guo, Zhou, Jensen, Sheng, et~al.]{qiu2024tfb}
Xiangfei Qiu, Jilin Hu, Lekui Zhou, Xingjian Wu, Junyang Du, Buang Zhang, Chenjuan Guo, Aoying Zhou, Christian~S Jensen, Zhenli Sheng, et~al.
\newblock Tfb: Towards comprehensive and fair benchmarking of time series forecasting methods.
\newblock \emph{arXiv preprint arXiv:2403.20150}, 2024{\natexlab{a}}.

\bibitem[Qiu et~al.(2024{\natexlab{b}})Qiu, Wu, Lin, Guo, Hu, and Yang]{duet}
Xiangfei Qiu, Xingjian Wu, Yan Lin, Chenjuan Guo, Jilin Hu, and Bin Yang.
\newblock Duet: Dual clustering enhanced multivariate time series forecasting.
\newblock \emph{arXiv preprint arXiv:2412.10859}, 2024{\natexlab{b}}.

\bibitem[Sara{\c{c}}li et~al.(2013)Sara{\c{c}}li, Do{\u{g}}an, and Do{\u{g}}an]{saraccli2013comparison}
Sinan Sara{\c{c}}li, Nurhan Do{\u{g}}an, and {\.I}smet Do{\u{g}}an.
\newblock Comparison of hierarchical cluster analysis methods by cophenetic correlation.
\newblock \emph{Journal of inequalities and Applications}, 2013\penalty0 (1):\penalty0 203, 2013.

\bibitem[Shi et~al.(2024)Shi, Wang, Nie, Li, Ye, Wen, and Jin]{time-moe}
Xiaoming Shi, Shiyu Wang, Yuqi Nie, Dianqi Li, Zhou Ye, Qingsong Wen, and Ming Jin.
\newblock Time-moe: Billion-scale time series foundation models with mixture of experts.
\newblock \emph{arXiv preprint arXiv:2409.16040}, 2024.

\bibitem[Stock \& Watson(2001)Stock and Watson]{stock2001vector}
James~H Stock and Mark~W Watson.
\newblock Vector autoregressions.
\newblock \emph{Journal of Economic perspectives}, 15\penalty0 (4):\penalty0 101--115, 2001.

\bibitem[Tan et~al.(2024)Tan, Merrill, Gupta, Althoff, and Hartvigsen]{pattn}
Mingtian Tan, Mike Merrill, Vinayak Gupta, Tim Althoff, and Tom Hartvigsen.
\newblock Are language models actually useful for time series forecasting?
\newblock \emph{Advances in Neural Information Processing Systems}, 37:\penalty0 60162--60191, 2024.

\bibitem[Wan et~al.(2015)Wan, Zhao, Song, Xu, Lin, and Hu]{wan2015photovoltaic}
Can Wan, Jian Zhao, Yonghua Song, Zhao Xu, Jin Lin, and Zechun Hu.
\newblock Photovoltaic and solar power forecasting for smart grid energy management.
\newblock \emph{CSEE Journal of power and energy systems}, 1\penalty0 (4):\penalty0 38--46, 2015.

\bibitem[Wang et~al.(2024{\natexlab{a}})Wang, Wu, Shi, Hu, Luo, Ma, Zhang, and Zhou]{timemixer}
Shiyu Wang, Haixu Wu, Xiaoming Shi, Tengge Hu, Huakun Luo, Lintao Ma, James~Y Zhang, and Jun Zhou.
\newblock Timemixer: Decomposable multiscale mixing for time series forecasting.
\newblock \emph{arXiv preprint arXiv:2405.14616}, 2024{\natexlab{a}}.

\bibitem[Wang et~al.(2024{\natexlab{b}})Wang, Wu, Dong, Liu, Long, and Wang]{wang2024deep}
Yuxuan Wang, Haixu Wu, Jiaxiang Dong, Yong Liu, Mingsheng Long, and Jianmin Wang.
\newblock Deep time series models: A comprehensive survey and benchmark.
\newblock \emph{arXiv preprint arXiv:2407.13278}, 2024{\natexlab{b}}.

\bibitem[Wang et~al.(2024{\natexlab{c}})Wang, Wu, Dong, Qin, Zhang, Liu, Qiu, Wang, and Long]{timexer}
Yuxuan Wang, Haixu Wu, Jiaxiang Dong, Guo Qin, Haoran Zhang, Yong Liu, Yunzhong Qiu, Jianmin Wang, and Mingsheng Long.
\newblock Timexer: Empowering transformers for time series forecasting with exogenous variables.
\newblock \emph{arXiv preprint arXiv:2402.19072}, 2024{\natexlab{c}}.

\bibitem[Woo et~al.(2024)Woo, Liu, Kumar, Xiong, Savarese, and Sahoo]{moirai}
Gerald Woo, Chenghao Liu, Akshat Kumar, Caiming Xiong, Silvio Savarese, and Doyen Sahoo.
\newblock Unified training of universal time series forecasting transformers.
\newblock 2024.

\bibitem[Wu et~al.(2021)Wu, Xu, Wang, and Long]{autoformer}
Haixu Wu, Jiehui Xu, Jianmin Wang, and Mingsheng Long.
\newblock Autoformer: Decomposition transformers with auto-correlation for long-term series forecasting.
\newblock \emph{Advances in neural information processing systems}, 34:\penalty0 22419--22430, 2021.

\bibitem[Wu et~al.(2022)Wu, Hu, Liu, Zhou, Wang, and Long]{timesnet}
Haixu Wu, Tengge Hu, Yong Liu, Hang Zhou, Jianmin Wang, and Mingsheng Long.
\newblock Timesnet: Temporal 2d-variation modeling for general time series analysis.
\newblock \emph{arXiv preprint arXiv:2210.02186}, 2022.

\bibitem[Wu et~al.(2023)Wu, Zhou, Long, and Wang]{wu2023interpretable}
Haixu Wu, Hang Zhou, Mingsheng Long, and Jianmin Wang.
\newblock Interpretable weather forecasting for worldwide stations with a unified deep model.
\newblock \emph{Nature Machine Intelligence}, 5\penalty0 (6):\penalty0 602--611, 2023.

\bibitem[Xu et~al.(2023)Xu, Zeng, and Xu]{fits}
Zhijian Xu, Ailing Zeng, and Qiang Xu.
\newblock Fits: Modeling time series with $10 k $ parameters.
\newblock \emph{arXiv preprint arXiv:2307.03756}, 2023.

\bibitem[Yin et~al.(2021)Yin, Wu, Wei, Shen, Qi, and Yin]{yin2021deep}
Xueyan Yin, Genze Wu, Jinze Wei, Yanming Shen, Heng Qi, and Baocai Yin.
\newblock Deep learning on traffic prediction: Methods, analysis, and future directions.
\newblock \emph{IEEE Transactions on Intelligent Transportation Systems}, 23\penalty0 (6):\penalty0 4927--4943, 2021.

\bibitem[Zeng et~al.(2023)Zeng, Chen, Zhang, and Xu]{dlinear}
Ailing Zeng, Muxi Chen, Lei Zhang, and Qiang Xu.
\newblock Are transformers effective for time series forecasting?
\newblock In \emph{Proceedings of the AAAI conference on artificial intelligence}, volume~37, pp.\  11121--11128, 2023.

\bibitem[Zhang \& Yan(2023)Zhang and Yan]{crossformer}
Yunhao Zhang and Junchi Yan.
\newblock Crossformer: Transformer utilizing cross-dimension dependency for multivariate time series forecasting.
\newblock In \emph{The Eleventh International Conference on Learning Representations}, 2023.

\bibitem[Zhao \& Shen(2024)Zhao and Shen]{zhao2024rethinking}
Lifan Zhao and Yanyan Shen.
\newblock Rethinking channel dependence for multivariate time series forecasting: Learning from leading indicators.
\newblock \emph{arXiv preprint arXiv:2401.17548}, 2024.

\bibitem[Zhou et~al.(2021)Zhou, Zhang, Peng, Zhang, Li, Xiong, and Zhang]{informer}
Haoyi Zhou, Shanghang Zhang, Jieqi Peng, Shuai Zhang, Jianxin Li, Hui Xiong, and Wancai Zhang.
\newblock Informer: Beyond efficient transformer for long sequence time-series forecasting.
\newblock In \emph{Proceedings of the AAAI conference on artificial intelligence}, volume~35, pp.\  11106--11115, 2021.

\bibitem[Zhou et~al.(2022)Zhou, Ma, Wen, Wang, Sun, and Jin]{fedformer}
Tian Zhou, Ziqing Ma, Qingsong Wen, Xue Wang, Liang Sun, and Rong Jin.
\newblock Fedformer: Frequency enhanced decomposed transformer for long-term series forecasting.
\newblock In \emph{International conference on machine learning}, pp.\  27268--27286. PMLR, 2022.

\end{thebibliography}
\bibliographystyle{iclr2026_conference}

\appendix
\newpage
\section{Comparison with Existing Works}
\label{app:Comparison with Existing Works}
\begin{table}[ht]
\centering
\caption{Comparison of datasets in benchmarks and foundation models on four properties:
(1) \textbf{Multivariate} indicates whether the datasets involve multiple variables (more than one channel). 
(2) \textbf{High-Dimensional ($\geq$1000)} refers to the presence of datasets with at least 1000 channels, either for pretraining or evaluation. 
(3) \textbf{Temporal Alignment} denotes whether the datasets are temporally aligned. 
(4) \textbf{Evaluation} assesses whether high-dimensional datasets are used specifically for evaluation purposes.}
\resizebox{\textwidth}{!}{%
\begin{tabular}{lcccc}
\toprule
\textbf{Property} & \textbf{Multivariate} & \textbf{High-Dimensional ($\geq$1000)} & \textbf{Temporal Alignment} & \textbf{Evaluation} \\
\midrule
\textbf{Benchmark} & & & & \\
\quad M3~\citep{m3}           & \xmark & \xmark & \xmark & \xmark \\
\quad M4~\citep{m4}           & \xmark & \xmark & \xmark & \xmark \\
\quad LTSF-Linear~\citep{dlinear}  & \cmark & \xmark & \cmark & \xmark \\
\quad TSlib~\citep{timesnet}        & \cmark & \xmark & \cmark & \xmark \\
\quad BasicTS~\citep{basicts}& \cmark & \xmark & \cmark & \xmark \\
\quad BasicTS+~\citep{basicts}     & \cmark & \xmark & \cmark & \xmark \\
\quad Monash~\citep{monash}       & \xmark & \xmark & \xmark & \xmark \\
\quad Libra~\citep{libra}        & \xmark & \xmark & \xmark & \xmark \\
\quad TFB~\citep{qiu2024tfb}          & \cmark & \cmark~(only one) & \cmark~(only one) & \cmark~(only one) \\
\midrule
\textbf{Foundation Models} & & & & \\
\quad Chronos~\citep{chronos}      & \cmark & \cmark & \xmark & \xmark \\
\quad Moment~\citep{moment}       & \cmark & \cmark & \xmark & \xmark \\
\quad Moirai~\citep{moirai}       & \cmark & \cmark & \xmark & \xmark \\
\quad TimesFM~\citep{timefm}      & \cmark & \cmark & \xmark & \xmark \\
\quad Time-MoE~\citep{time-moe}     & \cmark & \cmark~(only one) & \cmark~(only one) & \cmark~(only one) \\
\midrule
\textbf{Ours} & & & & \\
\quad \textbf{\dataset{}}      & \cmark & \cmark & \cmark & \cmark \\
\bottomrule
\end{tabular}
}
\label{tab:comparison}
\end{table}

\begin{wrapfigure}[]{R}{0.25\textwidth}
% \vskip-1.2em
    \centering
    \includegraphics[width=1\linewidth]{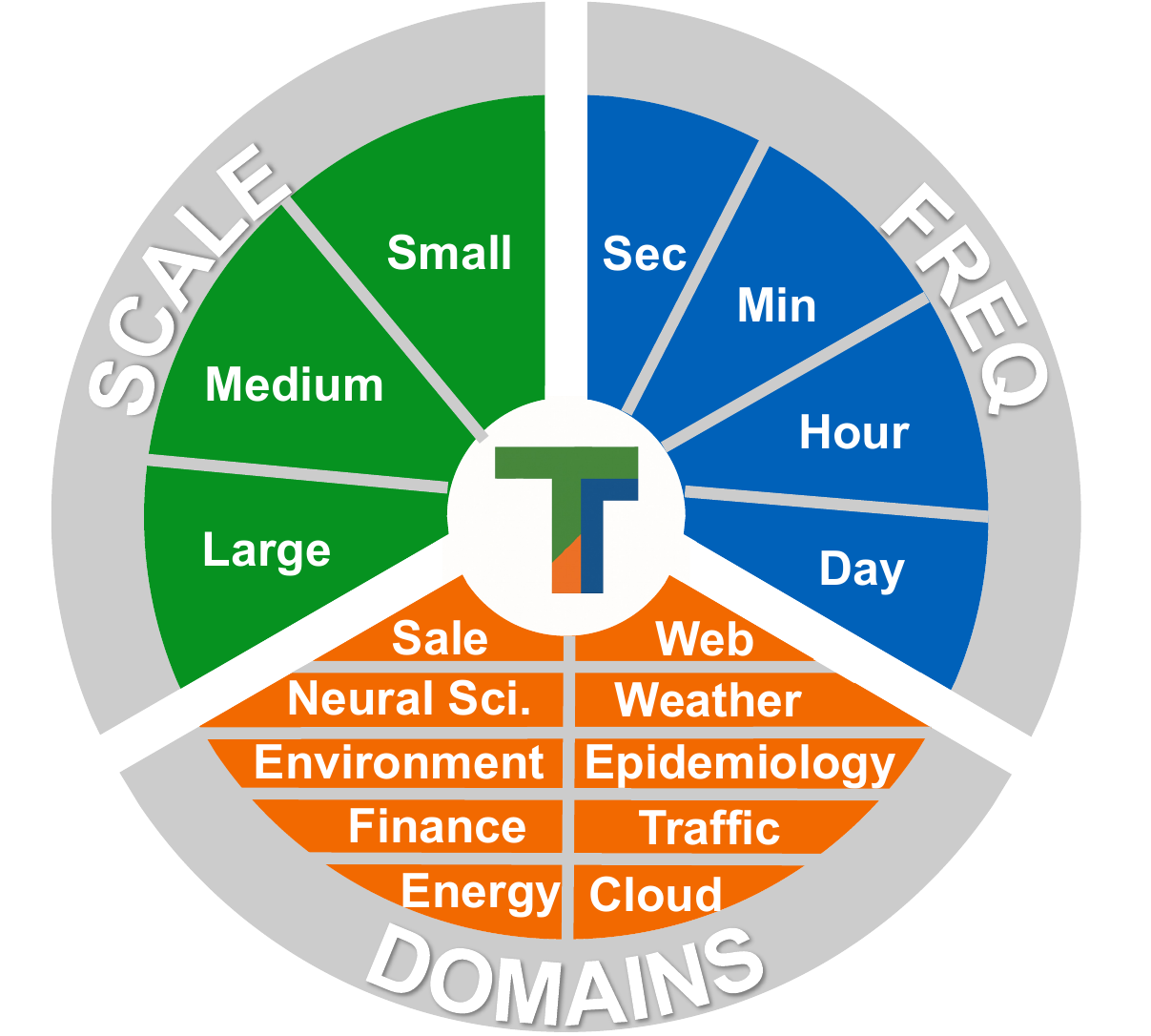}
    % \vskip-0.3em
    \caption{\dataset{} provides diverse high-dimensional datasets.}
    % \caption{\dataset{} provides high-dimensional datasets that are diverse in scale, frequency, domains.}
    \label{fig:time_hd}
\end{wrapfigure}

\textbf{High-Dimensional Time Series Datasets.} As shown in Table~\ref{tab:comparison}, most existing benchmarks do not include any high-dimensional datasets. The only exception is TFB~\citep{qiu2024tfb}, which incorporates only one dataset with 2000 channels (Wiki2000) for evaluation. Among foundation models, although all of them include high-dimensional time series datasets, these datasets are primarily used for pretraining rather than evaluation. Time-MoE~\citep{time-moe} is the only one that evaluates on a high-dimensional dataset (Global Temp) with 1000 channels. Both TFB and Time-MoE attempt to explore the high-dimensional regime, but they each include \textbf{only one} such dataset. This limited coverage constrains their effectiveness as comprehensive benchmarks.

\textbf{Multivariate Time Series Forecasting.}
\textit{Channel-independent (CI)}: The CI modeling strategy uses a shared model backbone for all channels and processes each channel independently during the forward pass~\citep{patchtst,dlinear,fits,timemixer,pattn,pdf,sparsetsf}. This design generally results in lower model capacity but offers greater robustness. However, since CI models treat each channel independently, they are unable to capture inter-channel correlations explicitly.

\textit{Channel-dependent (CD)}:
In contrast, the CD modeling strategy incorporates modules that are specifically designed to capture dependencies across channels. These methods tend to offer higher model capacity and can leverage cross-channel information more effectively. Based on how the inter-channel correlations are handled, CD approaches can be categorized as follows:
(1) Position-wise~\citep{tsmixer,timesnet,autoformer,informer,fedformer,crossformer}: These methods first project the channel dimension of each time step into a hidden embedding space, resulting in one embedding per time step. This approach may ignore correlations across channels at different time steps and could introduce noise. After generating the embeddings, these methods mainly focus on modeling temporal dependencies.
(2) Token-wise~\citep{itransformer,timexer}: These methods treat each channel as a token and input the sequence of tokens into a Transformer model. This helps reduce irrelevant noise for forecasting. Although multi-layer Transformers have the potential to capture complex hierarchical inter-channel structures, they are extremely time-consuming and do not scale well to high-dimensional time series. Moreover, the lack of an explicit mechanism for learning hierarchical channel structure may limit their effectiveness.
(3) Cluster-wise~\citep{ccm,duet}: These approaches divide the multivariate time series into disjoint channel clusters using clustering algorithms. However, using a single-layer clustering step may overlook the hierarchical organization present in high-dimensional time series, which can lead to suboptimal performance.

\section{Future Opportunity}
\label{app:Future Opportunity}
\textbf{Scaling Forecasting Models to High-Dimensional Time Series.}
Most existing forecasting models are developed and benchmarked on low-dimensional datasets, which restricts their practical relevance in high-dimensional applications. As the number of variables increases, computational and memory bottlenecks become more prominent, especially for models employing attention mechanisms or dense channel-wise operations. Future research can prioritize the development of architectures that maintain accuracy while achieving sub-quadratic complexity in the number of channels. In particular, scalable attention mechanisms, parameter-efficient representations, and adaptive channel selection strategies are promising directions.

\textbf{Improving the Ability of Models to Capture Inter-Channel Correlations.}
High-dimensional time series often exhibit rich and structured inter-channel dependencies, including spatial, semantic, and hierarchical patterns. While \method{} demonstrates that learning latent hierarchical structures can be effective, there remains room to improve in two aspects: (1) dynamic correlation modeling that adapts to temporal shifts in inter-variable relationships, and (2) domain-aware inductive biases that can leverage prior knowledge (e.g., spatial topology, sectoral information) to guide the learning of channel relationships. Better capturing these correlations can lead to models that are both more interpretable and more robust to distributional changes.

\textbf{Benchmarking and Standardization for High-Dimensional Forecasting.}
Despite the introduction of the \dataset{} benchmark, the evaluation of forecasting models in high-dimensional regimes remains underdeveloped. Future work should explore more comprehensive evaluation protocols, including robustness to missing channels, adaptation to distribution drift, and generalization to unseen domains. Additionally, standardized metrics for scalability and memory efficiency, beyond predictive accuracy, are needed to holistically compare models in real-world deployment scenarios.

\section{Broader Impact}
\label{app:broader_impact}
This work reframes time series forecasting (TSF) around high-dimensional TSF (HDTSF), where models must reason over thousands of channels and their structure. By curating \textsc{Time-HD} across more than ten domains and large channel counts, we move beyond the past practice of training and testing only on small, low-dimensional benchmarks. This shift makes evaluation more reliable, reveals failure modes that only appear at scale, and brings assessment closer to real deployments in sensor networks, markets, and geoscience. It also encourages the community to measure real gains rather than benchmark-specific gains.

\dataset{} and \textsc{Time-HD-Lib} are released to lower the barrier to entry and raise reproducibility. The dataset suite standardizes timestamps, horizons, and metrics; the library offers a distributed train–validate–test pipeline with common preprocessing and reference configs. Together, they let researchers and practitioners start quickly, scale experiments without bespoke infrastructure, and add new datasets and tasks that broaden HDTSF coverage. This shared stack supports fair comparisons and faster iteration on modeling ideas.

Methodologically, \method{} serves as a strong baseline for future work on HDTSF. Its design targets cross-channel structure and efficiency, yielding competitive accuracy with favorable speed and memory use. We recommend using \method{} as a starting point for ablations and new architectures, reporting uncertainty alongside point forecasts, and disaggregating results by domain. Users should watch for dataset bias and compute cost; the released pipeline helps by making runs reproducible and efficient. Overall, the task, resources, and baseline aim to help the TSF community build and evaluate models that scale to realistic settings.

\section{Datasets Description}
\label{app:Dataset Description}
% \subsection{Long-Term Time Series Forecasting}
\underline{\textbf{Temp} (Global Temp) and \textbf{Wind} (Global Wind).}
We adapt this dataset from Corrformer~\citep{wu2023interpretable}. Global Temp\&Wind dataset is from the National Centers for Environmental Information (NCEI)~\footnote{https://www.ncei.noaa.gov/}. This dataset contains the hourly averaged wind speed and
hourly temperature of 3,850 stations (7,700 dimensions) around the world from 01/01/2019 to 12/31/2020, 17,544 timesteps in total. This dataset is well established. However, we find there are some constant columns in the datasets, which means the temp or wind speed never change. We remove these columns.

\underline{\textbf{Solar} (NREL Solar Power).} NREL Solar Power is collected from National Renewable Energy Laboratory (NREL)~\footnote{https://www.nrel.gov/grid/solar-power-data.html}. One of the most popular datasets for time series forecasting is Solar-Energy~\citep{lstnet}, which is sourced from NREL. However, it includes data from only 137 solar photovoltaic (PV) plants in Alabama, limiting the dataset to 137 dimensions. To increase the dimensionality, we collect data from 5,166 PV plants across 47 U.S. states, excluding Alaska, Hawaii, and North Dakota, using data from NREL. The data source is clean so we don't conduct further preprocessing on it.

\underline{\textbf{Meters} (Smart meters in London).} We process half-hourly smart meter energy consumption data by loading data from multiple blocks, where each block represents a subset of meters, and aggregating energy consumption data for each meter. The energy unit is kWh/hh. The data is sorted chronologically and merged into a single dataset, consisting of \textit{5,567 meters and 40,405 timesteps}.
\textbf{First}, to address missing values introduced by the staggered installation of meters, we filter out records before 07/14/2012, as meters were installed at different times starting from 11/23/2011. We also ensure that only records with exact half-hour intervals are retained, resulting in a dataset of \textit{5,567 meters and 28,512 timesteps}.
\textbf{Second}, we remove meters that became operational after 07/14/2012 or stopped reporting before 02/28/2014, reducing the dataset to \textit{4,213 meters while maintaining 28,512 timesteps}.
\textbf{Third}, we count missing values and remove meters with more than 10 missing values. The remaining missing values are replaced with zero, yielding a final dataset of \textit{2,899 meters and 28,512 timesteps}.
The cleaned and processed dataset is then saved for further analysis.

\underline{\textbf{Traffic-CA} (LargeST-CA), \textbf{Traffic-GLA} (LargeST-GLA), \textbf{Traffic-GBA} (LargeST-GBA).} The source of this dataset is the Caltrans Performance Measurement System (PeMS)\footnote{https://pems.dot.ca.gov/}, and we utilize its preprocessed version\citep{liu2023largest}. 
\textbf{LargeST-CA} encompassing a total of 8,600 sensors in California and each sensor contains five years of traffic flow data (from 2017 to 2021) with a 5-minute
interval, resulting in a total of 525,888 time frames. 
\textbf{LargeST-CA consists of three subsets: }
\textbf{(1)} \textbf{LargeST-GLA} includes 3,834 sensors installed across five counties in the Greater Los Angeles (GLA) area: Los Angeles, Orange, Riverside, San Bernardino, and Ventura. 
\textbf{(2)} \textbf{LargeST-GBA} contains 2,352 sensors distributed across 11 counties in the Greater Bay Area (GBA): Alameda, Contra Costa, Marin, Napa, San Benito, San Francisco, San Mateo, Santa Clara, Santa Cruz, Solano, and Sonoma.
\textbf{(3)} \textbf{LargeST-SD}, the smallest subset, consists of 716 sensors located exclusively in San Diego County.
Since LargeST-SD has only 716 dimensions, we do not include it in our benchmark.
For the 525,888 original time frames, we aggregate the data into 1-hour intervals, resulting in 43,824 time frames. We then remove sensors with more than 100 missing (NaN) values, and apply linear interpolation to fill the remaining missing values.

% \textbf{Rebound.} Simulations in this paper made use of the REBOUND N-body code \citepp{rebound}. The simulations were integrated using WHFast, a symplectic Wisdom-Holman integrator \citepp{reboundwhfast,wh}. 

\underline{\textbf{Neurolib.}} Simulations made use of the Neurolib code \citep{cakan2021}. We utilize the brain structure in ABIDE dataset~\citep{craddock2013neuro}.

\underline{\textbf{SIRS.}} Simulations extend the classical SIR model \citep{cai2015stochastic} by incorporating temporary immunity, resulting in dynamic transitions among Susceptible, Infected, and Recovered groups within a single region. To enhance realism, parameter noise, and seasonal variation are included. For multiple regions, each compartment has a fixed transfer rate between regions to simulate commuting and spatial transmission dynamics.

% \subsection{Short-Term Time Series Forecasting}
\underline{\textbf{Air Quality} (Chinese Air Quality).} This dataset includes 7 air quality indicators: $AQI$, $CO$, $NO_2$, $O_3$, $PM10$, $PM2.5$, and $SO_2$ of 786 distinct cities in China. We remove dimensions with more than 100 missing values and interpolate other dimensions linearly, resulting in 1,106 dimensions.

\underline{\textbf{SP500} (S\&P 500 Index).} The S\&P 500 is a stock market index maintained by S\&P Dow Jones Indices~\footnote{https://www.spglobal.com/spdji/en/}. It consists of 503 common stocks issued by 500 large-cap companies traded on American stock exchanges. We use yfinance~\footnote{https://github.com/ranaroussi/yfinance}, a Python package for financial data retrieval, to download market data for these stocks from Yahoo Finance~\footnote{https://finance.yahoo.com/}. We select the latest S\&P 500 company list (as of 03/14/2025) and extract daily market data spanning the past 30 years (7,553 days). To ensure data consistency, we retain only the companies that were publicly traded 30 years ago, reducing the dataset to 295 companies. For each company, we extract five key market variables: Open, Close, High, Low, and Volume, resulting in a total of 1,475 dimensions.

\underline{\textbf{M5.}} We adapt this dataset from a Kaggle competition~\citep{m5-forecasting-accuracy}. It covers stores in three U.S. states, i.e. California, Texas, and Wisconsin, and includes item-level sales data, department and product category details, and store information. The dataset records daily sales for each item from 07/01/2015 to 06/30/2022 (a total of 1,947 days) across different departments and stores, resulting in 30,490 time series. Due to its sparsity (i.e., a large number of zero values), we aggregate sales by summing each item's sales across all departments and stores, reducing the dataset to 3,049 aggregated items.

\underline{\textbf{Atec.}} The dataset consists of 1589 traffic data collected every 10 minutes for different zones of different applications. We remove the dimensions that include more than 10 missing values and conduct interpolation for the remaining dimensions.

\underline{\textbf{Measles} (Measles England).} The measles dataset \citep{madden2024deep, liu2025cape, liu2024epilearn}, contains biweekly measles infections from 387 regions across England and Wales from 1944–1965. For each region, the dataset includes three features: inferred infections, population, and suspected cases, resulting in 1,161 dimensions.

\underline{\textbf{Wiki-20k} (Wikipedia Web Traffic).} The dataset is adapted from the Monash Time Series Forecasting Repository~\footnote{https://zenodo.org/records/7370977}. It comprises 145,063 time series representing web traffic for various Wikipedia pages from July 1, 2015, to June 30, 2022 (2,557 days). This dataset is an extended version of the one used in the Kaggle Wikipedia Web Traffic Forecasting Competition~\footnote{https://www.kaggle.com/c/web-traffic-time-series-forecasting/data}. Due to data sparsity, we remove the Wikipedia pages containing NaN values, resulting in a final dataset of 112,333 Wikipedia pages, each corresponding to one dimension. Due to the original dataset's dimension is too large and is not correlated, we select the first 20k dimensions to construct Wiki-20k dataset for evaluation.

\underline{\textbf{Mobility} (Google Community Mobility).} The Community Mobility dataset~\citep{GoogleMobility} aims to assist in analyzing the impact of COVID-19 on community movement patterns. 
The dataset covers movement trends over time (from February 15, 2020, to October 15, 2022, spanning 974 days) across 4,334 regions in different countries. It categorizes mobility data into \textbf{six} place types: \textbf{retail and recreation} (restaurants, cafes, shopping centers, theme parks, museums, libraries, and movie theaters), \textbf{groceries and pharmacies} (grocery stores, food warehouses, farmers markets, specialty food shops, drug stores, and pharmacies), \textbf{parks} (local and national parks, public beaches, marinas, dog parks, plazas, and public gardens), \textbf{transit stations} (subway, bus, and train stations), \textbf{workplaces} (places of employment), and \textbf{residential} (places of residence). 
The dataset initially consists of 26,004 dimensions, derived from the six place categories across 4,334 regions. However, due to data sparsity, we exclude regions containing NaN values, reducing the dataset to 971 regions and 5,826 dimensions.
The values in the dataset represent changes in visit frequency and duration of stay at various locations relative to the baseline (i.e. the median value for the same day of the week during the five-week period from January 3 to February 6, 2020.). These changes are computed using aggregated and anonymized data, similar to the methodology employed in Google Maps to estimate popular visit times.

\paragraph{Prediction Length $S$}
Earlier benchmarks commonly evaluated the same four horizons (96, 192, 336, 720) across all datasets, regardless of sampling frequency. 
This practice can yield scenarios with limited practical relevance, for example, requiring a weather model to predict 720 days ahead. 
In contrast, we align the prediction horizon with the data resolution so that each task corresponds to a \textit{realistic operational need}:

\begin{itemize}[]
    \item 1\,ms–30\,min data $\;\to\;$ $S=336$ steps (0.34\,s to 28\,h)
    \item 1\,h data $\;\to\;$ $S=168$ steps (7\,days)
    \item 6\,h data $\;\to\;$ $S=28$ steps (7\,days)
    \item 1\,day data $\;\to\;$ $S=7$ steps (1\,week)
\end{itemize}

These settings, listed under ``\textit{Pred Length}'' in Table~\ref{tab:dataset}, reflect forecasting horizons commonly adopted by domain experts (e.g., weekly planning in retail or epidemiology).

\section{Analysis of Channel Correlation in \dataset{}}
\label{sec:Channel Correlation Statistic}
We follow TFB~\citep{qiu2024tfb} to compute channel correlation. In multivariate time series, different variables often share common temporal patterns due to underlying shared factors. To quantify this, we extract 22 features per channel using Catch22~\citep{lubba2019catch22}. These features serve as fixed-length representations of each time series channel. We then compute the pairwise Pearson correlation coefficients~\citep{cohen2009pearson} between all channel representations and report the mean and variance. 

The results reported in the ``\textit{Correlation}'' column of Table~\ref{tab:dataset} show that all datasets exhibit \textbf{high average channel correlation}, reflecting strong \textit{inter-variable dependencies} and a dominant \textit{global component} shared across channels. While this global component is highly predictable, it provides little additional information for forecasting, since it largely repeats patterns common to all variables. When such shared trends dominate, they can overshadow the subtle yet critical channel-specific dynamics that ultimately determine predictive accuracy. Consequently, a model that focuses primarily on global trends will miss the fine-grained deviations that differentiate one channel from another. These findings motivate the need for models that explicitly disentangle and separate the global component from channel-specific signals while also capturing structured inter-channel dependencies.

\paragraph{Modeling Strategy in \method{}.}
\method{} addresses this challenge through an architecture that separates the modeling of unique channel dynamics from that of shared global trends (see Figure~\ref{fig:method}). Specifically:
\begin{itemize}
    \item \textbf{Modeling Hierarchical Inter-Channel Dependency:} The \textit{Hierarchical Latent Query Network} compresses the input into a low-dimensional latent space, where the \textit{Full-Rank Regularization} (Section~\ref{sec:Optimization Objective}) enforces disentanglement and removes redundancy. The subsequent \textit{Temporal Alignment} module then captures the essential channel-specific temporal patterns within this compact space.
    \item \textbf{Modeling Shared Trends:} After unique inter-channel dynamics are learned, the \textit{Hierarchical Upsampling Network} and final \textit{Output Projection} layer (Equation~\ref{eq:output_projection}) reconstruct the forecasts for all channels. With a residual connection to the original embeddings, the projection layer re-applies the simple global trends on top of the channel-specific predictions from the latent space.
\end{itemize}

In essence, \method{} decomposes the problem: it compresses and disentangles redundant global components to focus on informative inter-channel dynamics, while leveraging the output projection to efficiently recover global trends. This separation is not a loss of information but a deliberate design choice that enables accurate and robust forecasting in high-dimensional, highly correlated time series.

We also want to highlight that, in the real world, there indeed exist time series collections with \textbf{low inter-channel correlations}. However, merging such weakly related time series into a single dataset is usually not meaningful because they do not provide complementary predictive information to each other. In contrast, the datasets in \dataset{} are specifically collected from scenarios with strong inter-channel correlations, where exploiting these correlations can meaningfully improve forecasting performance.

\section{Implementation Details}
\label{sec:Implementation Details}
\textbf{Metric Details.} We use Mean Square Error (MSE) and Mean Absolute Error (MAE) as our evaluation metrics, following~\citep{itransformer, moderntcn, timemixer, timesnet, autoformer, informer, dlinear}:
\begin{equation}
    \text{MSE} = \frac{1}{S \times C}\sum_{i=1}^S\sum_{j=1}^C(\mathbf{Y}_{ij}-\hat{\mathbf{Y}}_{ij})^2,
\end{equation}
\begin{equation}
    \text{MAE} = \frac{1}{S \times C}\sum_{i=1}^S\sum_{j=1}^C|\mathbf{Y}_{ij}-\hat{\mathbf{Y}}_{ij}|.
\end{equation}
Here, $\mathbf{Y} \in \mathbb{R}^{S \times C}$ represents the ground truth, and $\hat{\mathbf{Y}} \in \mathbb{R}^{S \times C}$ represents the predictions. $S$ denotes the future prediction length, $C$ is the number of channels, and $Y_{ij}$ indicates the value at the $i$-th future time point for the $j$-th channel.

\textbf{Experiment Details.} All experiments are implemented in PyTorch~\citep{paszke2019pytorch} and conducted on a cluster equipped with 8 NVIDIA A100 GPUs and several NVIDIA H100 GPUs. 
The baselines are trained using their default configurations as reported in their respective papers, and with further hyperparameter searching (Learning Rate in \{0.01, 0.001, 0.0001\}, Input Sequence Length $T$ in \{$3\times S$, $4\times S$, $5\times S$\}), where $S$ is prediction length.
ADAM optimizer~\citep{kingma2014adam} is used with MSE loss for the training of all the models. 
We apply early stopping with a patience value of 5 epochs. 
The batch size is initially set to 32. If an out-of-memory (OOM) error occurs, the batch size is automatically halved until the issue is resolved.
For \method{}, we default to using non-stationary~\citep{liu2022non} normalization for all datasets. We default to set the number of layers $L$ to 2, the hidden dimension $d$ to 512, and the reduction ratio $r$ to 16 for all the datasets.
Additional detailed \method{} model configuration information is presented in Table~\ref{tab:hyperparams}.
We conduct hyperparameter sensitivity analysis in Appendix~\ref{app:hyperparameter_sensitivity}. 
For robustness, we perform five-seed runs on all datasets and calculate the standard deviation. As shown in Table~\ref{tab:ucast_results_sd}, \method{} trains reliably across all cases.

\begin{table}[ht]
\centering
\caption{Performance of \method{} (five-seed runs) on all datasets. Results are reported as mean $\pm$ standard deviation.}
\resizebox{0.5\textwidth}{!}{%
\begin{tabular}{lcc}
\toprule
\textbf{Dataset} & \textbf{MSE} & \textbf{MAE} \\
\midrule
Atec         & $0.290 \pm 0.001$ & $0.278 \pm 0.002$ \\
Air Quality  & $0.420 \pm 0.005$ & $0.411 \pm 0.004$ \\
Temp         & $0.263 \pm 0.002$ & $0.384 \pm 0.001$ \\
Wind         & $1.078 \pm 0.003$ & $0.686 \pm 0.002$ \\
Mobility     & $0.317 \pm 0.004$ & $0.324 \pm 0.003$ \\
Traffic-CA   & $0.060 \pm 0.001$ & $0.134 \pm 0.002$ \\
Traffic-GBA  & $0.059 \pm 0.000$ & $0.128 \pm 0.000$ \\
Traffic-GLA  & $0.059 \pm 0.000$ & $0.135 \pm 0.001$ \\
M5           & $3.472 \pm 0.001$ & $0.843 \pm 0.000$ \\
Measles      & $0.014 \pm 0.005$ & $0.050 \pm 0.004$ \\
Neurolib     & $1.747 \pm 0.005$ & $0.349 \pm 0.006$ \\
Solar        & $0.159 \pm 0.001$ & $0.220 \pm 0.002$ \\
SIRS         & $0.006 \pm 0.000$ & $0.048 \pm 0.001$ \\
Meters       & $0.947 \pm 0.011$ & $0.550 \pm 0.009$ \\
SP500        & $0.501 \pm 0.008$ & $0.305 \pm 0.006$ \\
Wiki-20k     & $10.275 \pm 0.005$ & $0.302 \pm 0.004$ \\
\bottomrule
\end{tabular}
}
\label{tab:ucast_results_sd}
\end{table}

\begin{table}[h]
\centering
\caption{\method{} hyperparameters used for each dataset.}
\resizebox{0.7\textwidth}{!}{%
\begin{tabular}{lccc}
\toprule
\textbf{Dataset} & \textbf{Learning Rate} & \textbf{Input Sequence Length $T$} & \boldmath$\alpha$ \\
\midrule
Atec                             & 0.0005 & $3\times S$ & 0.001  \\
Air Quality              & 0.0005 & $3\times S$ & 0.100  \\
Temp                     & 0.001 & $3\times S$ & 0.001  \\
Wind                     & 0.0005 & $4\times S$ & 0.010  \\
Mobility     & 0.0005 & $4\times S$ & 0.010  \\
Traffic-CA                      & 0.002 & $4\times S$ & 0.010  \\
Traffic-GBA                     & 0.002 & $4\times S$ & 0.010  \\
Traffic-GLA                     & 0.002 & $4\times S$ & 0.010  \\
M5                               & 0.001 & $4\times S$ & 10.000 \\
Measles                & 0.0005 & $4\times S$ & 0.010  \\
Neurolib                         & 0.0005 & $4\times S$ & 0.010  \\
Solar              & 0.001 & $4\times S$ & 0.010  \\
SIRS                             & 0.001 & $4\times S$ & 0.001  \\
Meters       & 0.0005 & $4\times S$ & 0.010  \\
SP500                            & 0.002 & $4\times S$ & 0.001  \\
Wiki-20k  & 0.001 & $4\times S$ & 0.001  \\
\bottomrule
\end{tabular}
}
\label{tab:hyperparams}
\end{table}

\section{Theorem~\ref{thm:risk_reduction} Proof}
\label{app:Theorem Proof1}
\begin{proof}
Our goal is to forecast 
$z_t = \begin{pmatrix} z_t^{(1)} \\ z_t^{(2)} \end{pmatrix}.$
For simplicity, we focus on forecasting \( z_{t+1}^{(1)} \); forecasting \( z_{t+1}^{(2)} \) is analogous. In a single forward process, the CI forecaster observes only \( x = z_t^{(1)} \) and does \textbf{not} observe \( w = z_t^{(2)} \). The optimal forecast under squared-error loss is the conditional expectation $\mathbb E[z_{t+1}^{(1)} \mid x]$. From the Equation~\ref{eq:var_system}, $z_{t+1}^{(i)} = a_{11}x + a_{12}w + \varepsilon$, where $\varepsilon = \varepsilon_{t+1}^{(1)} \sim \mathcal N(0, \sigma_{11})$.
The conditional distribution of $w$ given $x$ is Gaussian. The optimal forecast is:
\[
  \mathbb E[z_{t+1}^{(1)} \mid x] = \mathbb E[a_{11}x + a_{12}w + \varepsilon \mid x] = a_{11}x + a_{12}\mathbb E[w \mid x] + \mathbb E[\varepsilon \mid x].
\]
Since $\varepsilon$ is independent of $x$ (and $w$), $\mathbb E[\varepsilon \mid x] = \mathbb E[\varepsilon] = 0$ (as $\varepsilon$ has zero mean). Thus,
\[
  \mathbb E[z_{t+1}^{(1)} \mid x] = a_{11}x + a_{12}\mathbb E[w \mid x].
\]
The CI Bayes risk is the expected squared error of this optimal forecast:
\begin{align*}
  R_{\text{CI}} &= \mathbb E[(z_{t+1}^{(1)} - \mathbb E[z_{t+1}^{(1)} \mid x])^2] \\
  &= \mathbb E[(a_{11}x + a_{12}w + \varepsilon - (a_{11}x + a_{12}\mathbb E[w \mid x]))^2] \\
  &= \mathbb E[(a_{12}(w - \mathbb E[w \mid x]) + \varepsilon)^2] \\
  &= \mathbb E[a_{12}^2(w - \mathbb E[w \mid x])^2] + \mathbb E[\varepsilon^2] + 2a_{12}\mathbb E[(w - \mathbb E[w \mid x])\varepsilon].
\end{align*}
Again, using the independence of $\varepsilon$ from $w$ and $x$ (and thus from $\mathbb E[w \mid x]$), the cross-term expectation is zero: $\mathbb E[(w - \mathbb E[w \mid x])\varepsilon] = \mathbb E[w - \mathbb E[w \mid x]]\mathbb E[\varepsilon] = 0$. The first term is $a_{12}^2$ times the definition of the conditional variance, $\Var(w \mid x) = \mathbb E[(w - \mathbb E[w \mid x])^2]$, since $w$ has zero mean (due to normalization) and $\mathbb E[w \mid x]$ is the conditional mean. Therefore,
\begin{align*}
  R_{\text{CI}} &= a_{12}^2 \Var(w \mid x) + \mathbb E[\varepsilon^2] \\
  &= \sigma_{11} + a_{12}^2 \Var(w \mid x).
\end{align*}
Note that our assumption that the time series are normalized (zero mean) simplifies the interpretation but the structure of the result holds more generally.

In a single forward process, the CD forecaster observes \textbf{both} $x = z_t^{(1)}$ and $w = z_t^{(2)}$. The optimal forecast is $\mathbb E[z_{t+1}^{(1)} \mid x, w]$.
\[
  \mathbb E[z_{t+1}^{(1)} \mid x, w] = \mathbb E[a_{11}x + a_{12}w + \varepsilon \mid x, w] = a_{11}x + a_{12}w + \mathbb E[\varepsilon \mid x, w].
\]
Since $\varepsilon = \varepsilon_{t+1}^{(1)}$ is independent of both $x=z_t^{(1)}$ and $w=z_t^{(2)}$, $\mathbb E[\varepsilon \mid x, w] = \mathbb E[\varepsilon] = 0$. Thus,
\[
 \mathbb E[z_{t+1}^{(1)} \mid x, w] = a_{11}x + a_{12}w.
\]
The CD Bayes risk is:
\begin{align*}
  R_{\text{CD}} &= \mathbb E[(Y - \mathbb E[z_{t+1}^{(1)} \mid x, w])^2] \\
  &= \mathbb E[(a_{11}x + a_{12}w + \varepsilon - (a_{11}x + a_{12}w))^2] \\
  &= \mathbb E[\varepsilon^2] = \sigma_{11}.
\end{align*}

Comparing the two risks, we see that the reduction in risk from using the CD compared to the CI one is:
\[
  R_{\text{CI}} - R_{\text{CD}} = (\sigma_{11} + a_{12}^2 \Var(w \mid x)) - \sigma_{11} = a_{12}^2 \Var(w \mid x).
\]
This difference is non-negative since $a_{12}^2 \ge 0$ and the conditional variance $\Var(w \mid x) \ge 0$. The risk reduction is strictly positive if $a_{12} \neq 0$ and $\Var(w \mid x) > 0$. This occurs when the second channel ($w$) has a non-zero direct influence on the first channel in the next period ($a_{12} \neq 0$) and when the second channel carries some information not already present in the first channel (i.e., $w$ is not perfectly predictable from $x$, ensuring $\Var(w \mid x) > 0$). If $a_{12} = 0$, or if $w$ is perfectly predictable from $x$, the risks are equal. Similarly, if we focus on forecasting $z_{t+1}^{(2)}$, we have:
\[
  R_{\text{CI}} - R_{\text{CD}} = (\sigma_{22} + a_{21}^2 \Var(w \mid x)) - \sigma_{22} = a_{21}^2 \Var(x \mid w).
\]

Thus, $\sum_i^{2}(R_{\text{CI}} - R_{\text{CD}})=a_{12}^2 \Var(w \mid x)  + a_{21}^2 \Var(x \mid w)>0$
This example demonstrates that, under squared-error loss for a linear Gaussian VAR model, using the CD leads to a Bayes risk that is always less than or equal to the Bayes risk obtained using only the CI.

\end{proof}

\section{Theorem~\ref{thm:monotonicity} Proof}
\label{app:Theorem Proof2}

\begin{proof}
Let $\mathcal{F}_p = \sigma(z_t^{(1)}, \dots, z_t^{(p)})$ denote the sigma-field generated by the first $p$ channels at time $t$. The Bayes risk under squared-error loss using information set $\mathcal{F}_p$ is:
\[
  R_p = \mathbb{E}\left[(Y - \mathbb{E}[Y \mid \mathcal{F}_p])^2\right].
\]

Since conditional expectation minimizes mean squared error, and since the information sets are nested ($\mathcal{F}_1 \subseteq \dots \subseteq \mathcal{F}_P$), we have:
\[
  R_1 \ge R_2 \ge \dots \ge R_P = \mathbb{E}[(\varepsilon_{t+1}^{(1)})^2] = \sigma_{11},
\]
where the last equality holds because $\mathcal{F}_P$ fully determines the conditional mean of $Y$, and the only uncertainty is from the innovation $\varepsilon_{t+1}^{(1)}$.

Define the risk reduction relative to the univariate forecast as:
\[
  \Delta_p = R_1 - R_p.
\]
The monotonicity of $R_p$ implies that $\Delta_p$ is non-decreasing:
\[
  0 = \Delta_1 \le \Delta_2 \le \dots \le \Delta_P = R_1 - \sigma_{11}.
\]

To characterize when the inequality is strict, note that $\Delta_p > \Delta_{p-1}$ if and only if:
\[
  R_p < R_{p-1}.
\]
This occurs if the $p$-th channel $z_t^{(p)}$ provides additional information about $Y$ beyond what is already available from $z_t^{(1)}, \dots, z_t^{(p-1)}$. In the VAR(1) context, this requires:
\begin{itemize}
  \item $a_{1p} \ne 0$, so that $z_t^{(p)}$ directly affects $Y$;
  \item $\Var(z_t^{(p)} \mid z_t^{(1)}, \dots, z_t^{(p-1)}) > 0$, so that $z_t^{(p)}$ is not a deterministic function of the previous channels.
\end{itemize}

If either condition fails, then $z_t^{(p)}$ does not improve the forecast of $Y$, and $R_p = R_{p-1}$. Therefore, each strictly decreasing step in Bayes risk corresponds to the inclusion of a new, informative, and non-redundant channel.
\end{proof}

\section{Theorem~\ref{thm:fullrank} Proof}
\label{app:Theorem Proof3}
\begin{proof}
Write the compact singular–value decomposition of
\( \mathbf{H}^{(\ell-1)} \) as
\(
\mathbf{H}^{(\ell-1)}
   = \mathbf{U}\,\mathbf{\Sigma}\,\mathbf{V}^{\top},
\)
where
\(
\mathbf{U}\in\mathbb{R}^{C\times r},\;
\mathbf{\Sigma}\in\mathbb{R}^{r\times r}\;
\text{is diagonal with positive entries},\;
\mathbf{V}\in\mathbb{R}^{d\times r},
\)
and the columns of \( \mathbf{U} \) form an orthonormal basis of the
row‑space of \( \mathbf{H}^{(\ell-1)} \).
Because \( r\ge C' \) by assumption, we can pick a
matrix \( \mathbf{Q}\in\mathbb{R}^{C'\times C} \) of full row rank
whose rows are any \( C' \) independent linear
combinations of the rows of \( \mathbf{U}^{\top} \).
Formally, let
\( \mathbf{S}\in\mathbb{R}^{C'\times C} \)
select those \( C' \) rows (so \( \mathbf{S}\mathbf{S}^{\top}=\mathbf{I}_{C'} \)),
and set \( \mathbf{Q}= \mathbf{S} \).
Then
\[
\mathbf{H}^{(\ell)}
  = \mathbf{Q}\mathbf{H}^{(\ell-1)}
  = \mathbf{S}\mathbf{U}\,\mathbf{\Sigma}\,\mathbf{V}^{\top}.
\]
Since \( \mathbf{S}\mathbf{U}\in\mathbb{R}^{C'\times r} \)
has full row rank \( C' \),
\[
\operatorname{rank}\!\bigl(\mathbf{H}^{(\ell)}\bigr)
  = \operatorname{rank}\!\bigl(\mathbf{S}\mathbf{U}\bigr)
  = C',
\]
which proves the first claim.
Assume further that \( r\ge d \) and \( C'\ge d \),
so the maximal attainable row rank is \( d \).
Define
\(
\mathbf{\Sigma}^{(\ell)}
      = \tfrac{1}{d}\mathbf{H}^{(\ell)}\mathbf{H}^{(\ell)\!\top}
      \in\mathbb{R}^{C'\times C'}.
\)
Adding the penalty
\(
\mathcal{L}_{\mathrm{cov}}
      = -\tfrac{1}{C'}\log\det\!\bigl(\mathbf{\Sigma}^{(\ell)}+\varepsilon
      \mathbf{I}_{C'}\bigr)
\)
in training maximises
\(
\det\!\bigl(\mathbf{\Sigma}^{(\ell)}\bigr)
      = \prod_{i=1}^{C'}\sigma_i^{2},
\)
where \( \sigma_i \) are the singular values of
\( \mathbf{H}^{(\ell)} \).
Because the product of the singular values grows when any
\( \sigma_i \) that is near zero increases,
gradient descent on
\(
-\log\det(\cdot)
\)
pushes all
\( \sigma_i \)
away from zero,
thereby driving \( \mathbf{H}^{(\ell)} \)
toward full row rank
\( \min(C',d) \).

\paragraph{Compatibility with HLQN attention.}
In HLQN the $\ell$-th block produces
\(
\mathbf{H}^{(\ell)}
      = W_o
        \!\bigl(
        \operatorname{softmax}(
          W_q\mathbf{Q}_\ell\,
          (W_k\mathbf{H}^{(\ell-1)})^{\!\top}/\sqrt{d_h})
        \,W_v\mathbf{H}^{(\ell-1)}
        \bigr).
\)
Write the attention weights as
\(
\mathbf{A}(\mathbf{H}^{(\ell-1)}) \;=\;
\operatorname{softmax}\bigl(
          W_q\mathbf{Q}_\ell\,
          (W_k\mathbf{H}^{(\ell-1)})^{\!\top}/\sqrt{d_h}\bigr)
      \in\mathbb{R}^{C_\ell\times C}.
\)
Then
\[
\mathbf{H}^{(\ell)}
  = \underbrace{%
       W_o\mathbf{A}(\mathbf{H}^{(\ell-1)})W_v}_{\displaystyle
       \mathbf{Q}(\mathbf{H}^{(\ell-1)})\in\mathbb{R}^{C_\ell\times C}}
     \,\mathbf{H}^{(\ell-1)}.
\]
Hence each forward pass is algebraically identical to
left-multiplying \(\mathbf{H}^{(\ell-1)}\) by a
sample-dependent matrix \(\mathbf{Q}(\mathbf{H}^{(\ell-1)})\).
As long as \(\mathbf{Q}(\mathbf{H}^{(\ell-1)})\) has full row rank
\(C_\ell\)—a property encouraged by the LogDet regulariser—the rank
analysis above applies verbatim:
\(
\operatorname{rank}\bigl(\mathbf{H}^{(\ell)}\bigr)=C_\ell
\)
and the gradient of the regulariser still pushes every singular value
of \(\mathbf{H}^{(\ell)}\) away from zero.
Therefore the conclusions of Theorem~\ref{thm:fullrank}
remain valid for the Hierarchical Latent Query Network attention block.

\end{proof}

\section{Theorem of Information Reduction}
\label{app:Theorem_information}
\begin{theorem}[Entropy Monotonicity of the Covariance Loss]\label{thm:entropy}
Assume the rows of 
\(
\mathbf{H}^{(\ell)}\in\mathbb{R}^{C'\times d}
\)
follow an i.i.d.\ zero-mean Gaussian distribution
\(
\mathcal{N}\bigl(\mathbf{0},\mathbf{\Sigma}^{(\ell)}\bigr)
\)
with
\(
\mathbf{\Sigma}^{(\ell)}=
\frac{1}{d}\mathbf{H}^{(\ell)}\mathbf{H}^{(\ell)\!\top}\succ\mathbf{0}.
\)
Define the differential entropy
\[
h^{(\ell)} \;=\; 
\frac{1}{2}\log\!\Bigl((2\pi e)^{C'}\det(\mathbf{\Sigma}^{(\ell)})\Bigr).
\]
Let
\[
\mathcal{L}_{\mathrm{cov}}^{(\ell)}
= -\frac{1}{C'}\log\det\!\bigl(\mathbf{\Sigma}^{(\ell)}+\varepsilon I_{C'}\bigr),
\qquad
\varepsilon>0.
\]
Then, along any differentiable optimisation trajectory
\(t\mapsto\mathbf{H}^{(\ell)}(t)\) that \emph{decreases}
\(\mathcal{L}_{\mathrm{cov}}^{(\ell)}\),
the entropy satisfies
\[
\frac{\mathrm{d}}{\mathrm{d}t}h^{(\ell)}(t)
= -\frac{C'}{2}\,\frac{\mathrm{d}}{\mathrm{d}t}
      \mathcal{L}_{\mathrm{cov}}^{(\ell)}(t)
      \;\ge\;0,
\]
with equality if and only if
\(
\nabla_{\mathbf{H}^{(\ell)}}\mathcal{L}_{\mathrm{cov}}^{(\ell)}= \mathbf{0}.
\)
Hence every strict descent step on \(\mathcal{L}_{\mathrm{cov}}^{(\ell)}\)
strictly increases the Shannon differential entropy of the latent
channel distribution.
\end{theorem}

\begin{proof}
For brevity drop the layer index \( (\ell) \) and write
\( \mathbf{\Sigma}=\mathbf{\Sigma}^{(\ell)}(t) \).
Since \( \mathbf{\Sigma}\succ\mathbf{0} \) and \( \varepsilon>0 \),
\( \mathbf{\Sigma}+\varepsilon I_{C'}\succ\mathbf{0}. \)
\(
\frac{\mathrm{d}}{\mathrm{d}t}\log\det\mathbf{\Sigma}
        = \operatorname{tr}\!\bigl(\mathbf{\Sigma}^{-1}\dot{\mathbf{\Sigma}}\bigr)
\),
we obtain
\[
\dot{h}
= \frac{1}{2}\operatorname{tr}\bigl(\mathbf{\Sigma}^{-1}\dot{\mathbf{\Sigma}}\bigr).
\tag{S1}
\]
By the same identity,
\[
\dot{\mathcal{L}}_{\mathrm{cov}}
= -\frac{1}{C'}
  \operatorname{tr}\!\Bigl((\mathbf{\Sigma}+\varepsilon I_{C'})^{-1}\dot{\mathbf{\Sigma}}\Bigr).
\tag{S2}
\]
Because
\(
(\mathbf{\Sigma}+\varepsilon I_{C'})^{-1}
  \preceq \mathbf{\Sigma}^{-1}
\)
in the Löwner order (for \(\varepsilon>0\)),
we have
\[
\operatorname{tr}\!\Bigl((\mathbf{\Sigma}+\varepsilon I_{C'})^{-1}\dot{\mathbf{\Sigma}}\Bigr)
\le
\operatorname{tr}\!\bigl(\mathbf{\Sigma}^{-1}\dot{\mathbf{\Sigma}}\bigr).
\tag{S3}
\]
Combining (S1)–(S3) yields
\(
\dot{h}\ge -\frac{C'}{2}\dot{\mathcal{L}}_{\mathrm{cov}}.
\)

If the optimisation rule ensures
\(
\dot{\mathcal{L}}_{\mathrm{cov}}\le0
\)
(e.g.\ gradient descent or any monotone line search),
then the inequality gives
\(
\dot{h}\ge0.
\)
Equality holds only when
\(
\dot{\mathcal{L}}_{\mathrm{cov}}=0
\),
which, under differentiability,
implies the gradient of the loss is zero.
Therefore entropy increases strictly whenever the loss decreases strictly.
\end{proof}

\section{Detailed Experimental Setup for Empirical Analysis}
\label{app:experiment_details_empirical_analysis}

This appendix provides additional details on the synthetic experiments described in Section~\ref{sec:empirical_analysis}, including data generation, model architectures, training procedure, and evaluation.

\subsection{Synthetic Data Generation}
We generate multivariate time series using a Vector Autoregression process:
\[
z_{t+1} = A z_t + \varepsilon_{t+1}, \quad \varepsilon_{t+1} \stackrel{i.i.d.}{\sim} \mathcal{N}(\mathbf{0}, I),
\]
where \( z_t \in \mathbb{R}^C \), and the coefficient matrix \( A \in \mathbb{R}^{C \times C} \) controls the temporal and inter-channel dependencies.

We consider the following dependency structures:
\begin{itemize}
  \item \textbf{Independent}: \( A \) is a diagonal matrix with entries drawn independently from \([0.8, 1.0]\). Each channel evolves independently without influence from others.
  \item \textbf{Anti-Self}: \( A \) has zeros on the diagonal and off-diagonal entries drawn from \([0.5, 1.0]\). This ensures each channel is fully dependent on other channels but not on itself.
\end{itemize}

To ensure stability, we scale \( A \) so that its spectral radius is less than 1. Time series are simulated with $T=100$ time steps and dimensionality $C \in \{100, 250, 2000\}$. The first 10 steps are discarded to reduce initialization bias.

\subsection{Forecasting Task and Data Preparation}
We define an autoregressive forecasting task where the model observes the past $L=4$ time steps and predicts the next $H=4$ steps:
\[
\{z_{t-L+1}, \dots, z_t\} \mapsto \{z_{t+1}, \dots, z_{t+H}\}.
\]
For each sequence, we construct overlapping sliding windows. The dataset is split into $80\%$ training and $20\%$ testing samples.

Before training, data is normalized using a channel-wise z-score computed over the training set:
\[
\tilde{z}_{t}^{(i)} = \frac{z_t^{(i)} - \mu_i}{\sigma_i},
\]
where \(\mu_i\) and \(\sigma_i\) are the mean and standard deviation of channel \(i\).

\subsection{Model Architectures}
We compare the following two models:
\begin{itemize}
  \item \textbf{CI}: Each channel is modeled independently using a shared linear model. Input shape is \([B, L, C]\), and output shape is \([B, H, C]\). Each channel uses a shared linear layer: \( \mathbb{R}^L \to \mathbb{R}^H \).
  \item \textbf{CD}: Each channel is first forecasted independently (as in CI), and the resulting output is then transformed using a shared linear layer across the channel dimension: \( \mathbb{R}^{H \times C} \to \mathbb{R}^{H \times C} \).
\end{itemize}

Both models are implemented in PyTorch and trained using the Adam optimizer.

\subsection{Training Details}
Each model is trained for 100 epochs with a learning rate of 0.01 and batch size 32. We use mean squared error (MSE) as the loss function. To stabilize training, gradients are clipped with a maximum norm of 5. Models are trained on a single V100 GPU.

\subsection{Evaluation Metrics}
We report the mean squared error (MSE) averaged over all channels and forecast steps:
\[
\text{MSE} = \frac{1}{H C N} \sum_{n=1}^N \sum_{h=1}^H \sum_{c=1}^C \left( \hat{z}_{t+h}^{(c)} - z_{t+h}^{(c)} \right)^2,
\]
where \(N\) is the number of test samples.

\section{Computational Complexity}\label{app:complexity}

\paragraph{Notation}
$C$ – number of channels (tokens in the \emph{channel} dimension);  
$T$ – look‑back length (time steps);  
$d$ – model width;  
$h$ – number of heads;  
$r>1$ – reduction ratio;  
$L=\lceil\log_r C\rceil$ – number of encoder/decoder stages.  
Before the first attention layer, both UCast and iTransformer apply a \emph{single} linear map  
$\mathbf{W}_{\text{in}}\!\in\!\mathbb{R}^{T\times d}$  
that compresses each length‑$T$ series to a $d$‑dimensional embedding.  
Consequently, attention thereafter depends only on \emph{channel counts}, not on $T$.

\paragraph{Hierarchical Latent Query Network.}
At stage $\ell$ the query length is $M_\ell=C/r^{\,\ell}$ and the key/value length is
$C_{\ell-1}=C/r^{\,(\ell-1)}$.  
Scaled‑dot‑product attention therefore costs
$$
\mathcal{T}^{(\ell)}_{\text{attn}}
 = \mathcal{O}\!\bigl(M_\ell\,C_{\ell-1}\,d/h\bigr)
 = \mathcal{O}\!\Bigl(\frac{C^{2}d}{h\,r^{\,2\ell-1}}\Bigr).
$$
Summing the geometric series over $\ell=1,\dots,L$ gives the encoder cost
$$
\mathcal{T}_{\text{enc}}
 = \mathcal{O}\!\Bigl(\frac{C^{2}d}{h\,r}\Bigr).
$$

\paragraph{Hierarchical Upsampling Network.}
The decoder is symmetric, so  
$\mathcal{T}_{\text{dec}}=\mathcal{O}\!\bigl(C^{2}d/(h\,r)\bigr)$.  
Hence the \emph{overall} time complexity of UCast is
$$
\mathcal{T}_{\text{UCast}}
       = \mathcal{O}\!\Bigl(\tfrac{C^{2}d}{h\,r}\Bigr).
$$

Peak attention memory occurs in the first stage with
$M_1\times C = C^{2}/r$ score entries per head:
$$
\mathcal{M}_{\text{UCast}}
       = \mathcal{O}\!\Bigl(\tfrac{C^{2}}{r}\Bigr).
$$

\paragraph{iTransformer}
iTransformer also compresses the time axis to width $d$ and then applies full self‑attention \emph{once} over the $C$ channel tokens:
$$
\mathcal{T}_{\text{iTrans}}
       = \mathcal{O}\!\Bigl(\tfrac{C^{2}d}{h}\Bigr),\qquad
\mathcal{M}_{\text{iTrans}}
       = \mathcal{O}\!\bigl(C^{2}\bigr).
$$

\paragraph{Comparison}
The reduction ratio $r$ governs the savings:
$$
\frac{\mathcal{T}_{\text{UCast}}}{\mathcal{T}_{\text{iTrans}}}
 = \frac{1}{r},\qquad
\frac{\mathcal{M}_{\text{UCast}}}{\mathcal{M}_{\text{iTrans}}}
 = \frac{1}{r}.
$$
With a typical choice $r=16$, UCast lowers both time and memory by a factor of $16$ compared with a flat channel Transformer, yet still retains the expressive power of attention through its latent‑query hierarchy.

\section{Empirical Evidence of Hierarchical Structures}
\label{app:Empirical Evidence of Hierarchical Structures}

\begin{table}[ht]
\centering
\caption{CPCC values of hierarchical clustering applied to real-world datasets. 
Higher values (closer to $1$) indicate stronger hierarchical structure.}
\resizebox{\textwidth}{!}{%
\begin{tabular}{lccccccccc}
\toprule
\textbf{Dataset} & M5 & Meter & Temp & Wind & Atec & Mobility & SP500 & Air Quality & Measles \\
\midrule
\textbf{CPCC} & 0.7852 & 0.7705 & 0.9057 & 0.7478 & 0.8866 & 0.8338 & 0.9484 & 0.7429 & 0.8950 \\
\bottomrule
\end{tabular}
}
\label{tab:cpcc}
\end{table}

% We emphasize that hierarchical latent structures are \textbf{common} in \textbf{large real-world systems}, 
% where channels exhibit implicit groupings and multi-level correlations based on factors such as spatial proximity 
% (e.g., nested geographical regions in climate data) or semantic relationships 
% (e.g., stocks within related economic sub-sectors).

To empirically evaluate the hierarchical structure of real-world datasets, 
we apply hierarchical clustering to the data and assess the resulting structure using 
the \textbf{cophenetic correlation coefficient (CPCC)}~\citep{saraccli2013comparison}. 
A CPCC value closer to $1$ indicates that the dendrogram generated by hierarchical clustering 
preserves the original pairwise distances between data points more accurately~\citep{kumar2016analysis}, 
which implies that the features in the dataset exhibit a \textbf{stronger hierarchical structure}.

\section{Empirical Validation for Channel Dependency Findings on Real-World Data}
\label{app:Empirical Validation for Channel Dependency Findings on Real-World Data}

\begin{table}[ht]
\centering
\caption{MAE performance of CI model DLinear and CD model \method{} on Wiki datasets with different channel subsets.}
\begin{tabular}{lccc}
\toprule
& \textbf{Channels} & \textbf{DLinear} & \textbf{\method{}} \\
\midrule
Wiki-20k   & 20,000 & 0.394 & \textbf{0.302} \\
Wiki-10k   & 10,000 & 0.489 & \textbf{0.385} \\
Wiki-2k    & 2,000  & 0.669 & \textbf{0.659} \\
Wiki-0.2k  & 200    & \textbf{0.697} & 0.782 \\
\bottomrule
\end{tabular}
\label{tab:wiki_ablation}
\end{table}

To bridge the synthetic analysis to a real-world setting, we conducted an additional experiment. 
To reduce the effect of dataset-specific bias, we fixed the dataset to Wiki-20k, which is a \textbf{real-world} dataset, 
and randomly selected different numbers of channels to construct Wiki-10k, Wiki-2k, and Wiki-0.2k. 
We then evaluated the MAE performance of the CI model DLinear and the CD model \method{} on these datasets. 
As shown in Table~\ref{tab:wiki_ablation}, CI models perform better when the dimensionality is \textbf{low}, 
but as the dimensionality increases, CD models gradually \textbf{outperform} CI models. 
This observation is consistent with the conclusion in Section~\ref{sec:empirical_analysis}.

\section{Hyperparameter Sensitivity Analysis}
\label{app:hyperparameter_sensitivity}
\begin{figure}[htbp]
    \centering

    \subfigure[$\alpha$]{
        \includegraphics[width=0.31\textwidth]{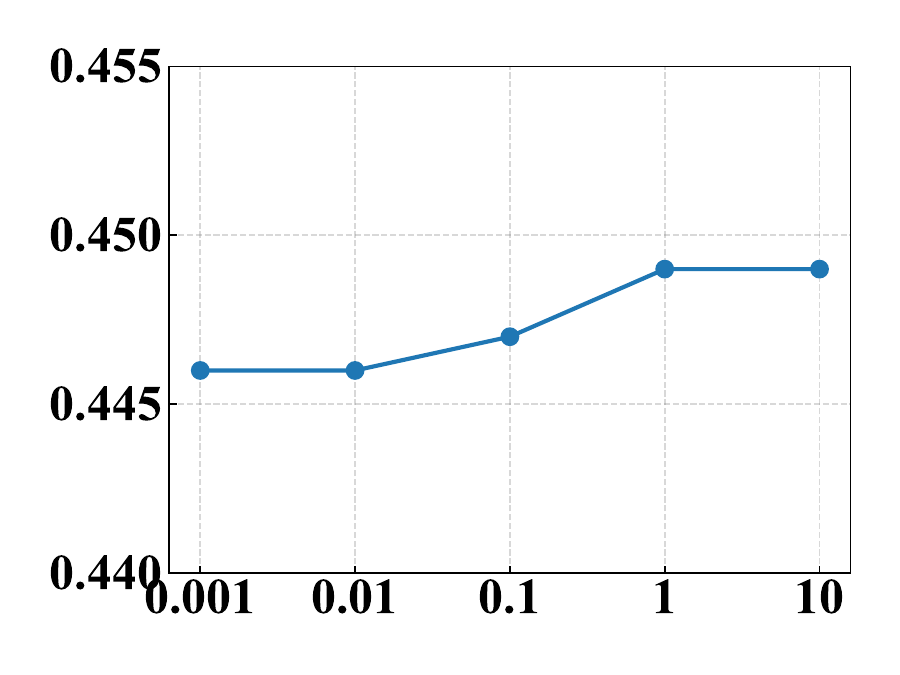}
        \label{fig:alpha}
    }
    \hfill
    \subfigure[$L$]{
        \includegraphics[width=0.31\textwidth]{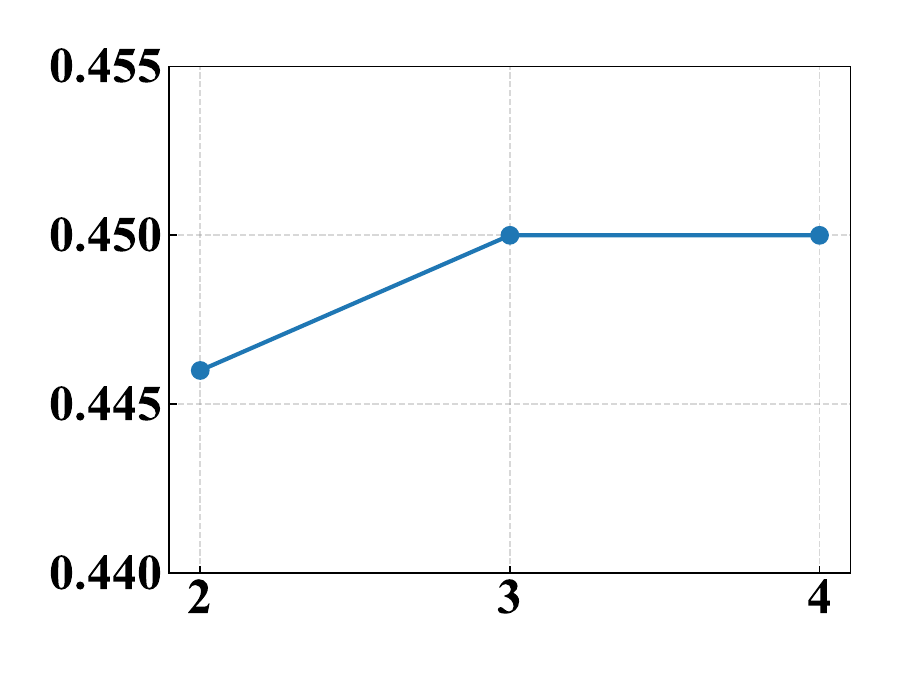}
        \label{fig:e_layers}
    }
    \hfill
    \subfigure[$r$]{
        \includegraphics[width=0.31\textwidth]{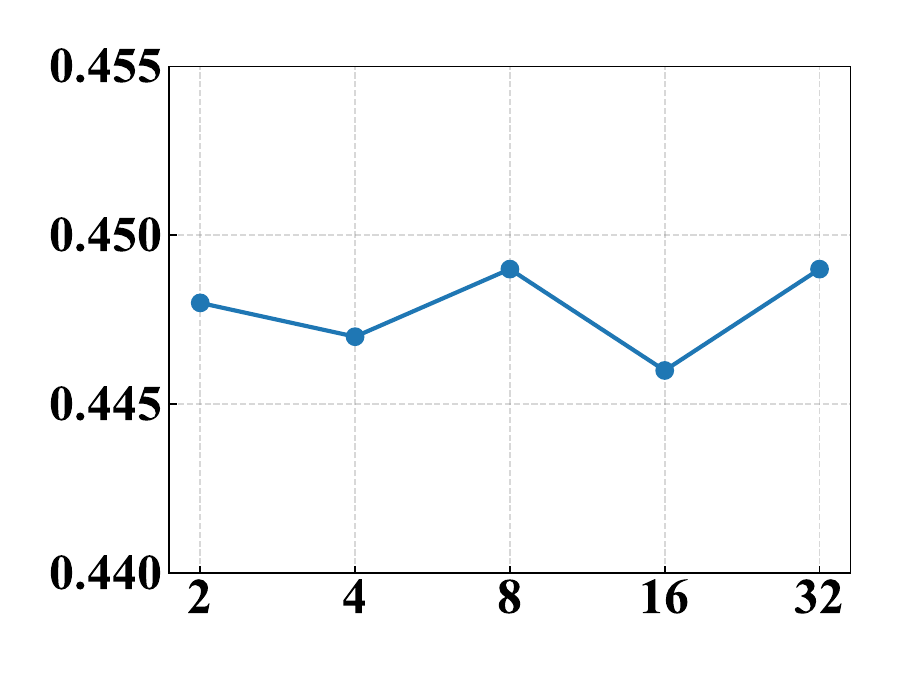}
        \label{fig:r}
    }

    \caption{Performance (MSE) on Air Quality dataset under different hyperparameters.}
    \label{fig:china_air_quality_hyperparams}
\end{figure}

\begin{figure}[htbp]
    \centering
    \begin{minipage}{0.48\linewidth}
        \centering
        \includegraphics[width=\linewidth]{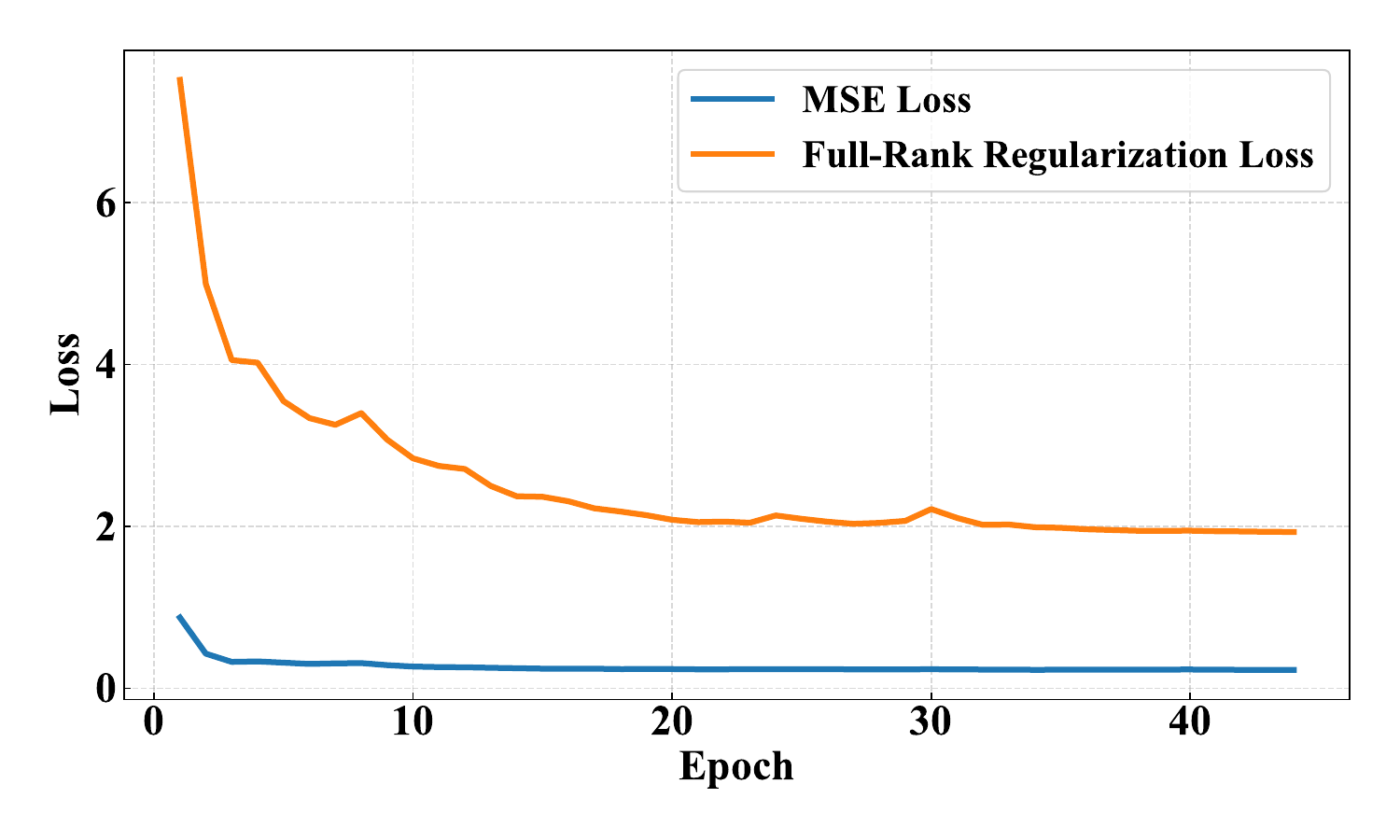}
    \end{minipage}
    \hfill
    \begin{minipage}{0.48\linewidth}
        \centering
        \includegraphics[width=\linewidth]{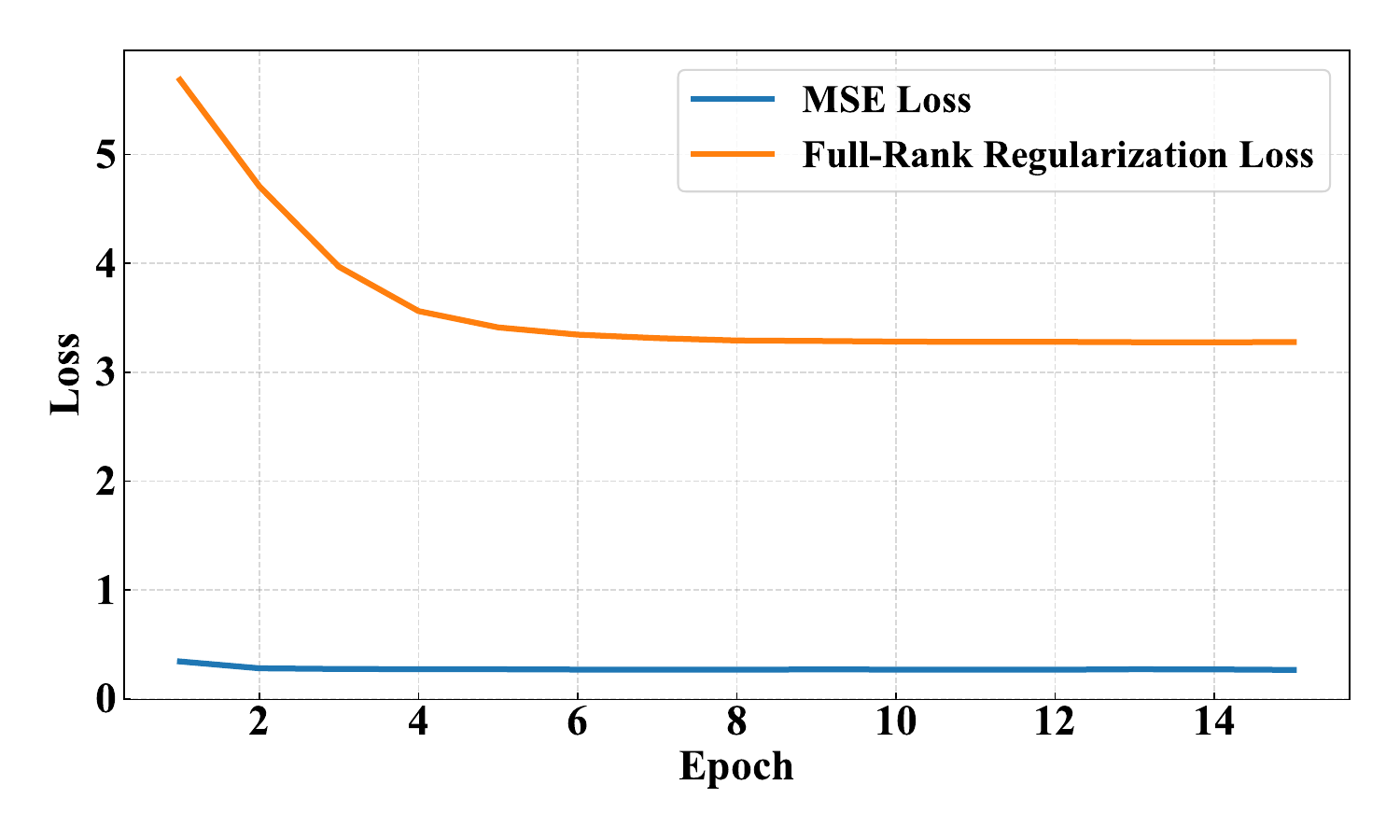}
    \end{minipage}
    
    \caption{Comparison of loss scales between the MSE loss $\mathcal{L}_{mse}$ and the full-rank regularization loss $\mathcal{L}_{cov}$. \textbf{Left:} Atec, \textbf{Right:} Temp.}
    \label{fig:loss_scale}
\end{figure}

\textbf{Sensitivity Analysis of Full-Rank Regularization Strength $\alpha$.}
We investigate how the regularization coefficient $\alpha$ influences the performance of \method{} by varying it across several orders of magnitude. As shown in Figure~\ref{fig:loss_scale}~(a), the best performance is achieved when $\alpha$ is set to a small value (0.001–0.01). This behavior can be explained by \textit{the imbalance in the scales of the two losses}: the full-rank regularization loss $\mathcal{L}_{cov}$ is roughly an order of magnitude larger than the MSE loss $\mathcal{L}_{mse}$ (see Figure~\ref{fig:loss_scale}). If $\alpha$ is too large, the optimization is dominated by $\mathcal{L}_{cov}$, which may overconstrain the latent channel representations and reduce the model’s ability to capture informative structures. Using a smaller $\alpha$ helps balance the two objectives, stabilizes gradient updates, and leads to better forecasting accuracy. Based on this analysis, we set $\alpha = 0.01$ to achieve a good trade-off between disentanglement and predictive performance.

\textbf{Sensitivity Analysis of Number of Hierarchical Layers $L$.}  
We evaluate the effect of the number of hierarchical latent query layers $L$ on forecasting performance, where a lower mean squared error (MSE) indicates better accuracy. As shown in Figure~\ref{fig:china_air_quality_hyperparams}~(b), increasing $L$ from 2 to 3 reduces the MSE, indicating worse performance. This suggests that deeper hierarchies may not capture more useful structure and could lead to redundancy or overfitting. Based on this analysis, we use a unified setting of layer number $L=2$ across all 16 datasets from different domains, which consistently achieves a stable trade-off between accuracy and efficiency without the need for per-dataset hyperparameter tuning.

\textbf{Sensitivity Analysis of Reduction Ratio $r$.}  
We investigate how the reduction ratio $r$, which determines the compression of channel representations at each layer, influences performance. As shown in Figure~\ref{fig:china_air_quality_hyperparams}~(c), the model achieves relatively stable performance across a wide range of $r$ values. However, extremely large or small values can lead to slight degradation. For instance, very low reduction (e.g., $r = 2$) leads to higher computational cost without performance gain, while aggressive reduction (e.g., $r = 32$) may discard useful information. We find that $r = 16$ offers a favorable trade-off between efficiency and accuracy. Thus, we use a unified setting of layer number $r = 16$ across all 16 datasets.

\begin{figure}[htbp]
    \centering
    \begin{minipage}{0.46\linewidth}
        \centering
        \includegraphics[width=\linewidth]{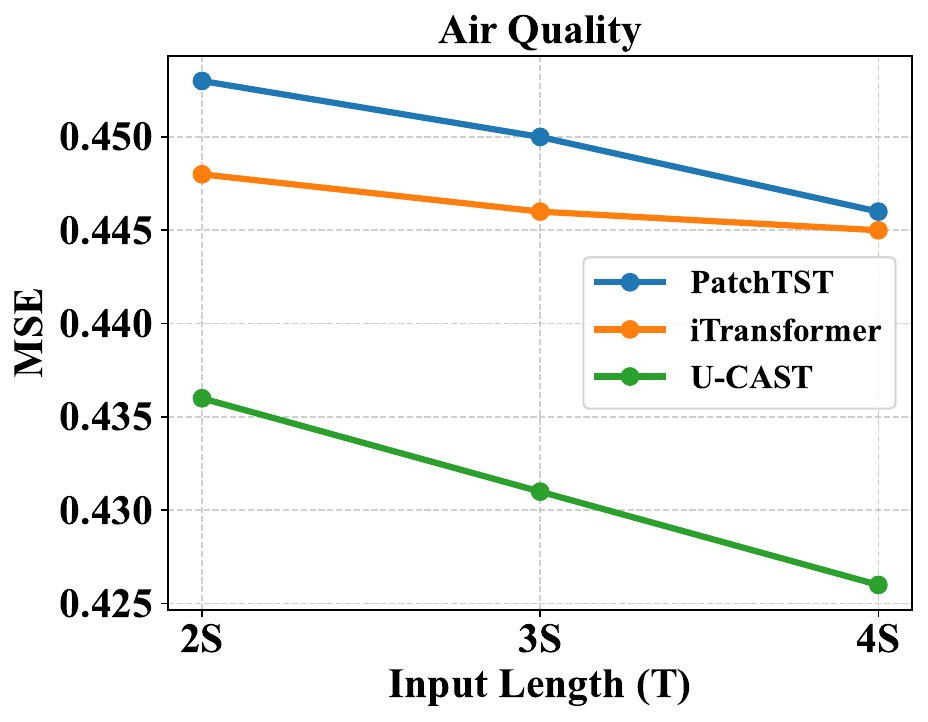}
    \end{minipage}
    \hfill
    \begin{minipage}{0.46\linewidth}
        \centering
        \includegraphics[width=\linewidth]{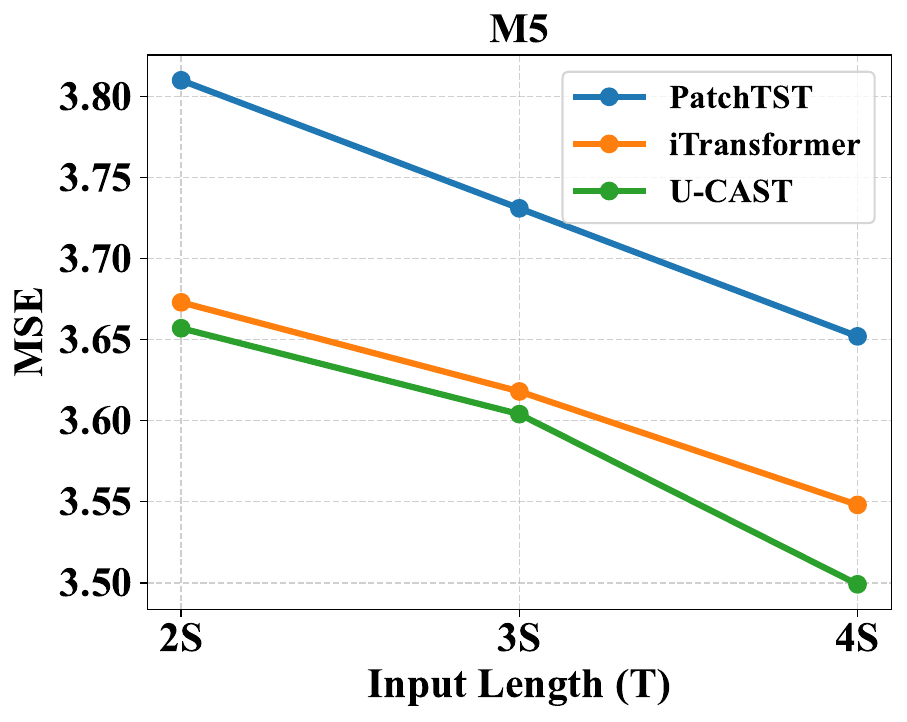}
    \end{minipage}
    \\
    \begin{minipage}{0.46\linewidth}
        \centering
        \includegraphics[width=\linewidth]{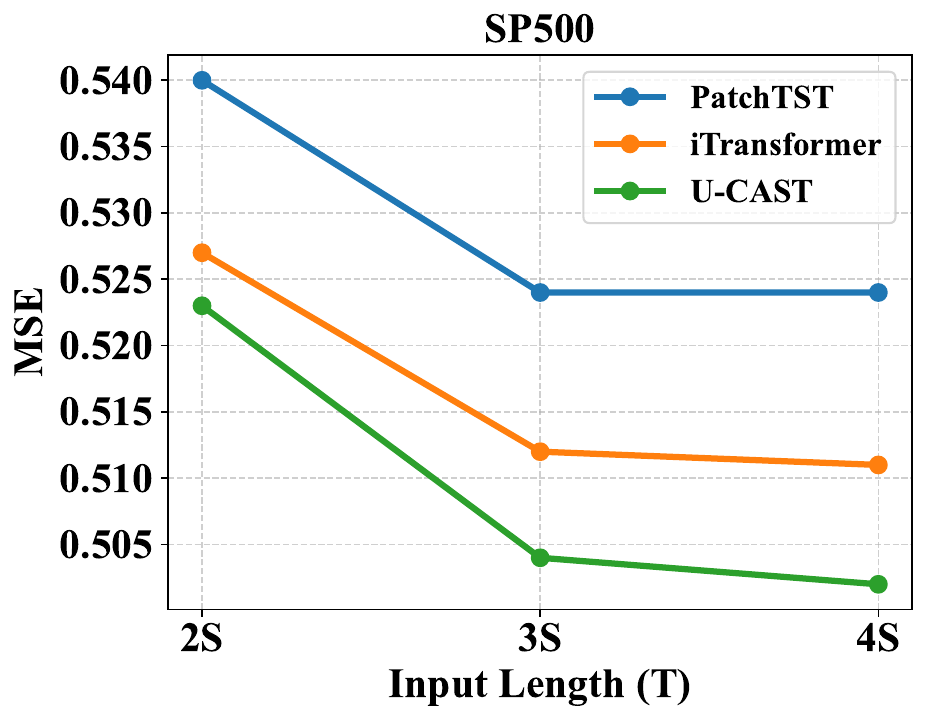}
    \end{minipage}
    \hfill
    \begin{minipage}{0.46\linewidth}
        \centering
        \includegraphics[width=\linewidth]{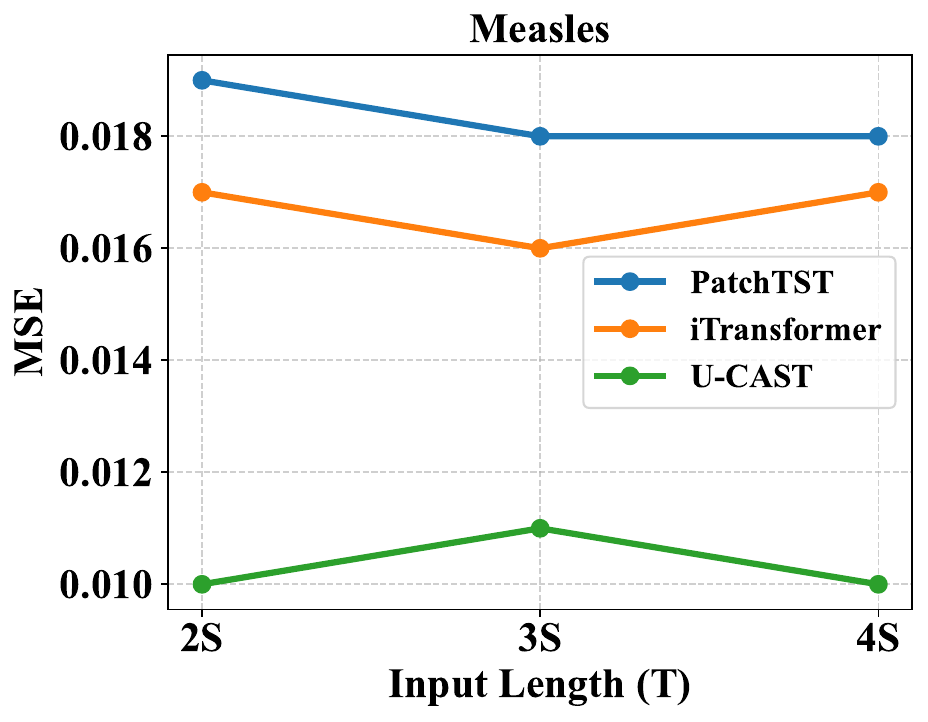}
    \end{minipage}
    
    \caption{Comparison of performance (MSE) among \method{}, PatchTST, and iTransformer on three datasets (Air Quality, M5, SIRS) under varying input lengths $T$.}
    \label{fig:different_input_lengths}
\end{figure}

\textbf{Sensitivity Analysis of Input Length $T$.}
To directly assess the impact of input length and to further validate the findings from our \dataset{} benchmark, we conduct an extended experiment. We evaluate PatchTST, iTransformer, and \method{} on three datasets across input lengths set to $2\times$, $3\times$, and $4\times$ the prediction horizon ($S$). The results are presented in Figure~\ref{fig:different_input_lengths}. These experiments lead to three main observations. \textbf{First}, longer input sequences generally improve performance in both CI and CD models on most datasets. \textbf{Second}, the effect varies by dataset. For instance, in the SIRS dataset, all models achieve their best results with the shortest input length, likely because its short-term epidemiological dynamics make longer histories add noise. This variability highlights the usefulness of \dataset{} as a benchmark for testing robustness under diverse temporal structures. \textbf{Third}, \method{} remains competitive or superior across most datasets and horizons, and CD models (\method{} and iTransformer) as a group continue to perform strongly when cross-channel structure provides informative signals.

% \section{Full Attention Score Map from hierarchical latent query network.}
% \label{app:Full Attention Score Map from hierarchical latent query network}
% Figure~\ref{fig:app_Measles1} and Figure~\ref{fig:app_Measles2} show the full attention score map used as a case in Figure~\ref{fig:case}. In Figure~\ref{fig:app_Measles1}, there is a clear period=3, where latent dimensions focus on different features in each region. In Figure~\ref{fig:app_Measles2}, the attention map further captures high-level latent hierarchical structure among channels.
% \begin{figure*}[t!]
%     \centering
%     \includegraphics[width=0.98\textwidth]{figures/Measles1.pdf}
%     \caption{The attention score map on Measles dataset at $L=1$. Shape: $32\times1161$}
%     \label{fig:app_Measles1}
% \end{figure*}

% \begin{figure*}[t!]
%     \centering
%     \includegraphics[width=0.98\textwidth]{figures/Measles2.pdf}
%     \caption{The attention score map on Measles dataset at $L=2$. Shape: $8\times32$}
%     \label{fig:app_Measles2}
% \end{figure*}

\section{Ablation Study}
\label{sec:app_ablation_study}
\begin{table*}[t!]
\centering
\caption{Ablation study on different components of \method{}. Each model is evaluated on 16 datasets using MSE and MAE. The best performance is highlighted in \textcolor{red}{\textbf{red}}, and the second-best is \textcolor{blue}{\underline{underlined}}.}
\resizebox{\textwidth}{!}{%
\begin{tabular}{c|rr|rr|rr|rr|rr}
\toprule
\multirow{2}{*}{\textbf{Dataset}} & \multicolumn{2}{c|}{\textbf{\method{}}} & \multicolumn{2}{c|}{\textbf{w/o $\mathcal{L}_{\text{cov}}$}} & \multicolumn{2}{c|}{\textbf{w/o Hierarchical}} & \multicolumn{2}{c|}{\textbf{w/o Latent Query}} & \multicolumn{2}{c}{\textbf{w/o Upsampling}} \\
& MSE & MAE & MSE & MAE & MSE & MAE & MSE & MAE & MSE & MAE \\
\midrule
Atec                        & \textcolor{red}{\textbf{0.287}}  & \textcolor{red}{\textbf{0.280}} & \textcolor{blue}{\underline{0.311}}  & \textcolor{blue}{\underline{0.297}} & 0.319  & 0.306 & 0.308  & 0.296 & 0.306  & 0.292 \\
Air Quality           & \textcolor{red}{\textbf{0.446}}  & \textcolor{red}{\textbf{0.430}} & \textcolor{blue}{\underline{0.461}}  & \textcolor{blue}{\underline{0.444}} & 0.447  & 0.431 & 0.451  & 0.434 & 0.451  & 0.434 \\
Temp                 & \textcolor{red}{\textbf{0.262}}  & \textcolor{red}{\textbf{0.383}} & \textcolor{blue}{\underline{0.273}}  & \textcolor{blue}{\underline{0.395}} & 0.272  & 0.392 & 0.280  & 0.398 & 0.294  & 0.410 \\
Wind                 & \textcolor{red}{\textbf{1.104}}  & \textcolor{red}{\textbf{0.692}} & 1.115  & 0.701 & \textcolor{blue}{\underline{1.110}}  & \textcolor{blue}{\underline{0.695}} & 1.106  & \textcolor{red}{\textbf{0.692}} & 1.109  & \textcolor{blue}{\underline{0.693}} \\
Mobility   & \textcolor{red}{\textbf{0.315}}  & \textcolor{red}{\textbf{0.317}} & 0.336  & 0.332 & \textcolor{blue}{\underline{0.316}}  & \textcolor{blue}{\underline{0.319}} & 0.326  & 0.323 & 0.327  & 0.324 \\
Traffic-CA                 & \textcolor{red}{\textbf{0.061}}  & \textcolor{red}{\textbf{0.131}} & 0.072  & 0.144 & \textcolor{blue}{\underline{0.062}}  & 0.138 & \textcolor{blue}{\underline{0.062}}  & 0.136 & \textcolor{blue}{\underline{0.062}}  & \textcolor{blue}{\underline{0.133}} \\
Traffic-GBA                & \textcolor{red}{\textbf{0.059}}  & \textcolor{red}{\textbf{0.126}} & 0.070  & 0.137 & \textcolor{blue}{\underline{0.060}}  & 0.130 & \textcolor{blue}{\underline{0.060}}  & 0.130 & \textcolor{blue}{\underline{0.060}}  & \textcolor{blue}{\underline{0.127}} \\
Traffic-GLA                & \textcolor{red}{\textbf{0.060}}  & \textcolor{red}{\textbf{0.132}} & 0.071  & 0.144 & \textcolor{blue}{\underline{0.061}}  & 0.136 & \textcolor{blue}{\underline{0.061}}  & 0.135 & \textcolor{blue}{\underline{0.061}}  & \textcolor{blue}{\underline{0.134}} \\
M5                          & \textcolor{red}{\textbf{3.501}}  & \textcolor{red}{\textbf{0.849}} & \textcolor{blue}{\underline{3.520}}  & 0.861 & 3.566  & \textcolor{blue}{\underline{0.854}} & 3.516  & 0.852 & 3.597  & 0.860 \\
Measles            & \textcolor{red}{\textbf{0.010}}  & \textcolor{red}{\textbf{0.042}} & 0.022  & 0.061 & 0.019  & 0.054 & \textcolor{blue}{\underline{0.011}}  & \textcolor{blue}{\underline{0.051}} & 0.014  & 0.063 \\
Neurolib                    & \textcolor{red}{\textbf{1.750}}  & \textcolor{red}{\textbf{0.350}} & \textcolor{blue}{\underline{1.770}}  & 0.362 & 1.818  & 0.370 & 1.813  & \textcolor{blue}{\underline{0.358}} & \textcolor{blue}{\underline{1.756}}  & 0.355 \\
Solar            & \textcolor{blue}{\underline{0.172}}  & 0.246 & 0.182  & \textcolor{blue}{\underline{0.242}} & 0.173  & 0.239 & 0.175  & 0.237 & \textcolor{red}{\textbf{0.167}}  & \textcolor{red}{\textbf{0.222}} \\
SIRS                       & \textcolor{red}{\textbf{0.007}}  & \textcolor{red}{\textbf{0.052}} & 0.038  & 0.119 & 0.008  & \textcolor{blue}{\underline{0.056}} & \textcolor{blue}{\underline{0.008}}  & 0.057 & 0.035  & 0.126 \\
Meters      & \textcolor{red}{\textbf{0.943}}  & \textcolor{blue}{\underline{0.551}} & 0.957  & 0.563 & 0.950  & 0.558 & \textcolor{blue}{\underline{0.945}}  & \textcolor{red}{\textbf{0.552}} & 0.946  & 0.550 \\
SP500                       & \textcolor{red}{\textbf{0.502}}  & \textcolor{red}{\textbf{0.301}} & \textcolor{blue}{\underline{0.561}}  & 0.338 & 0.562  & \textcolor{blue}{\underline{0.333}} & 0.566  & 0.335 & 0.597  & 0.349 \\
Wiki20000 & \textcolor{red}{\textbf{10.273}} & \textcolor{red}{\textbf{0.302}} & \textcolor{blue}{\underline{10.520}} & 0.321 & 10.459 & \textcolor{red}{\textbf{0.309}} & 10.509 & \textcolor{blue}{\underline{0.310}} & 10.527 & \textcolor{blue}{\underline{0.310}} \\
\bottomrule
\end{tabular}%
}
\label{tab:ablation_full}
\end{table*}

As \method{} integrates two core components, the Hierarchical Latent Query Network and the Hierarchical Upsampling Network, we assess their individual contributions through ablation studies. Specifically, we evaluate the impact of: (1) \textbf{w/o hierarchical} by retaining only a single layer for dimensionality reduction, (2) \textbf{w/o latent query} by setting $\mathbf{Q}_\ell$ requires\_grad=False, and (3)\textbf{w/o upsampling} by using a simple linear projection to restore channel dimension. In addition, we examine the role of the covariance full-rank regularisation by \textbf{w/o $\mathcal{L}_{\text{cov}}$}, i.e., setting $\alpha=0$. 

Table~\ref{tab:ablation_full} shows that removing any component generally degrades \method{}'s performance, confirming their necessity. \textbf{First}, \textit{w/o latent query} still performs well on some datasets (e.g., NREL Solar Power), likely due to high inter-variable correlation (0.998), where random projections suffice. \textbf{Second}, \textit{w/o upsampling} unexpectedly improves performance on NREL, suggesting that encoder representations are already sufficient and further decoding may introduce redundancy or overfitting. \textbf{Third}, removing the covariance regularization has the largest impact on structured datasets like SIRS (MSE rises from 0.007 to 0.038), showing its importance in enforcing feature decorrelation.

\section{Efficiency}
\label{sec:efficiency}
We compare the training time per batch with recent state-of-the-art models in Figure~\ref{fig:train_time1} and Figure~\ref{fig:train_time2}. In Figure~\ref{fig:train_time2}, the Atec, Temp, and Wind datasets are removed to provide a clearer view of the efficiency trend at higher dimensions. The x-axis corresponds to different datasets, arranged from left to right in increasing order of dimensionality.

The results indicate that:
\textbf{(1)} \method{} consistently maintains favorable training efficiency across different dimensionalities;
\textbf{(2)} Its efficiency advantage over other non-linear methods grows as dimensionality increases;
\textbf{(3)} The efficiency of CI-based methods (e.g., PAttn, PatchTST) is sensitive to both channel and temporal dimensions. In Figure~\ref{fig:train_time2}, when channel dimensionality is very large, as in Wiki-20k, their efficiency drops sharply; similarly, in Figure~\ref{fig:train_time1}, longer input lengths (Atec, Temp, Wind) also lead to significant degradation;
\textbf{(4)} As shown in Figure~\ref{fig:train_time2}, the efficiency of CD-based methods (iTransformer, \method{}) is mainly affected by the number of input channels. However, \method{} is more robust than iTransformer, showing more stable efficiency as dimensionality increases.

The detailed efficiency comparison of both the training and inference phases is provided in Appendix~\ref{app:detailed_efficiency_comparison}. In both stages, \method{} achieves the best trade-off between performance and efficiency.

\begin{figure}[htbp!]
    \centering
    \resizebox{0.99\linewidth}{!}{\includegraphics{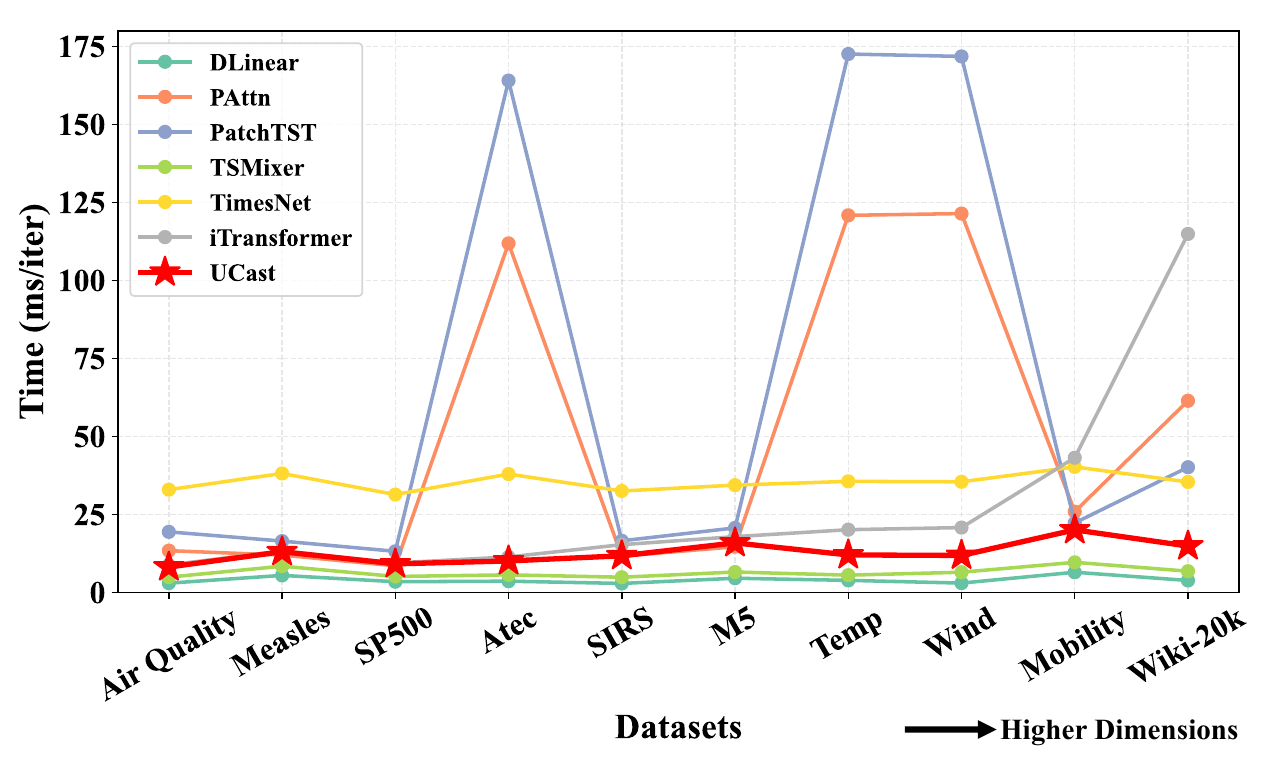}}
    \caption{Training time across datasets. The x-axis represents different datasets, arranged from left to right in order of increasing dimensionality.}
    \label{fig:train_time1}
\end{figure}

\begin{figure}[htbp!]
    \centering
    \resizebox{0.99\linewidth}{!}{\includegraphics{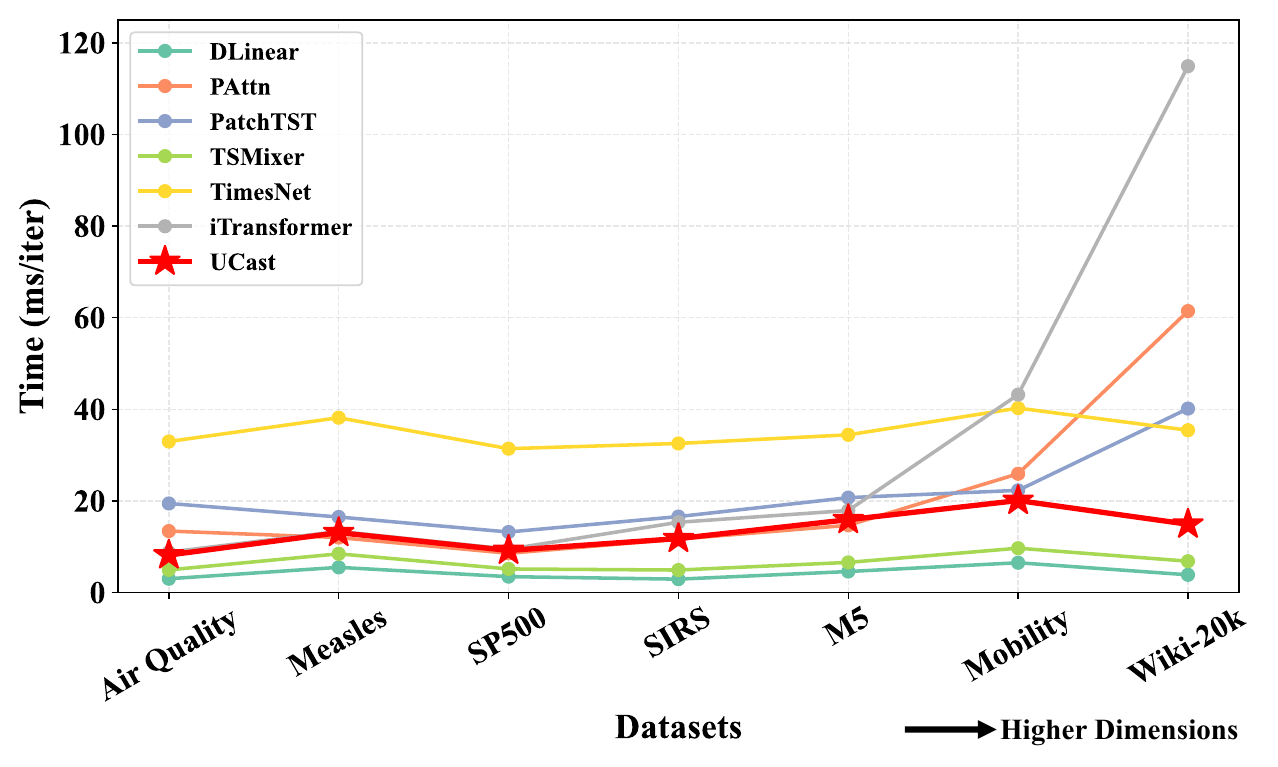}}
    \caption{Training time across datasets (\textbf{w/o Atec, Temp, and Wind datasets}). The x-axis represents different datasets, arranged from left to right in order of increasing dimensionality.}
    \label{fig:train_time2}
\end{figure}

\section{Additional Baselines}
\label{sec:more_baselines}

\begin{table}[h]
\centering
\caption{Preliminary forecasting results of MCANN and NECPlus on selected datasets. Lower values indicate better performance.}
\label{tab:additional_baselines}
\begin{tabular}{lcccc}
\toprule
 & \textbf{MCANN (MSE)} & \textbf{MCANN (MAE)} & \textbf{NECPlus (MSE)} & \textbf{NECPlus (MAE)} \\
\midrule
\textbf{Measles} & 1.727 & 0.851 & 1.160 & 0.722 \\
\textbf{Air Quality} & 0.466 & 0.457 & 0.447 & 0.449 \\
\textbf{Mobility} & 3.471 & 1.367 & 1.815 & 0.899 \\
\textbf{M5} & 17.031 & 1.441 & 16.842 & 1.387 \\
\bottomrule
\end{tabular}
\end{table}

We additionally include two methods, \textbf{MCANN}~\cite{li2025mc} and \textbf{NECPlus}~\cite{li2023extreme}, which further broaden the coverage of our benchmark. Both methods have been integrated into our \textbf{Time-HD-Lib}. For NECPlus, we adopt the standard model implementation. Preliminary results are presented in Table~\ref{tab:additional_baselines}.

\section{Showcases}
For clear comparison, we present test set showcases in Appendix~\ref{app:show_cases}, where \method{} shows better performance.

\clearpage
\section{Detailed Efficiency Comparison}
\label{app:detailed_efficiency_comparison}

\begin{figure*}[h!]
    \centering
    \includegraphics[width=0.94\textwidth]{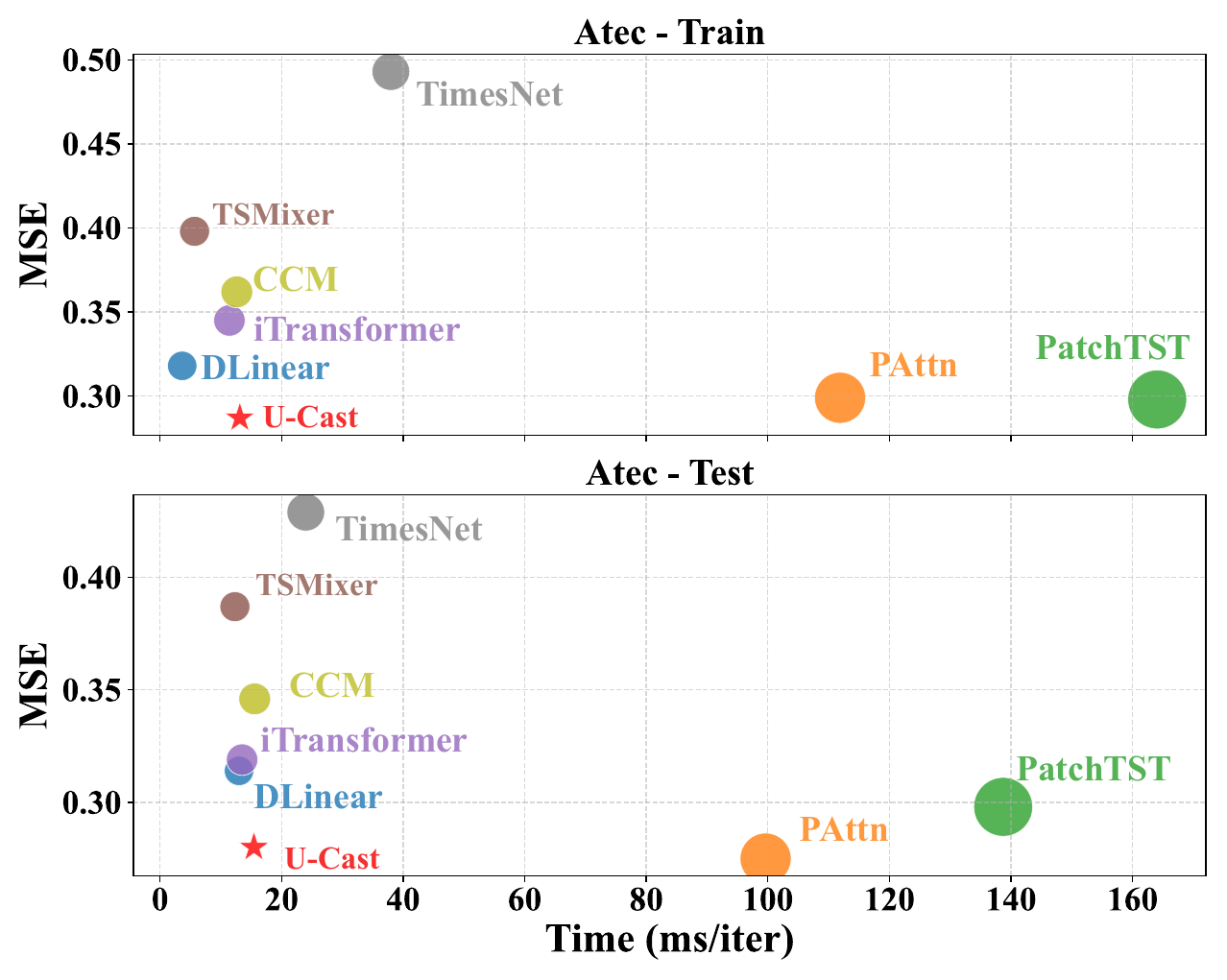}
\end{figure*}

\begin{figure*}[h!]
    \centering
    \includegraphics[width=0.94\textwidth]{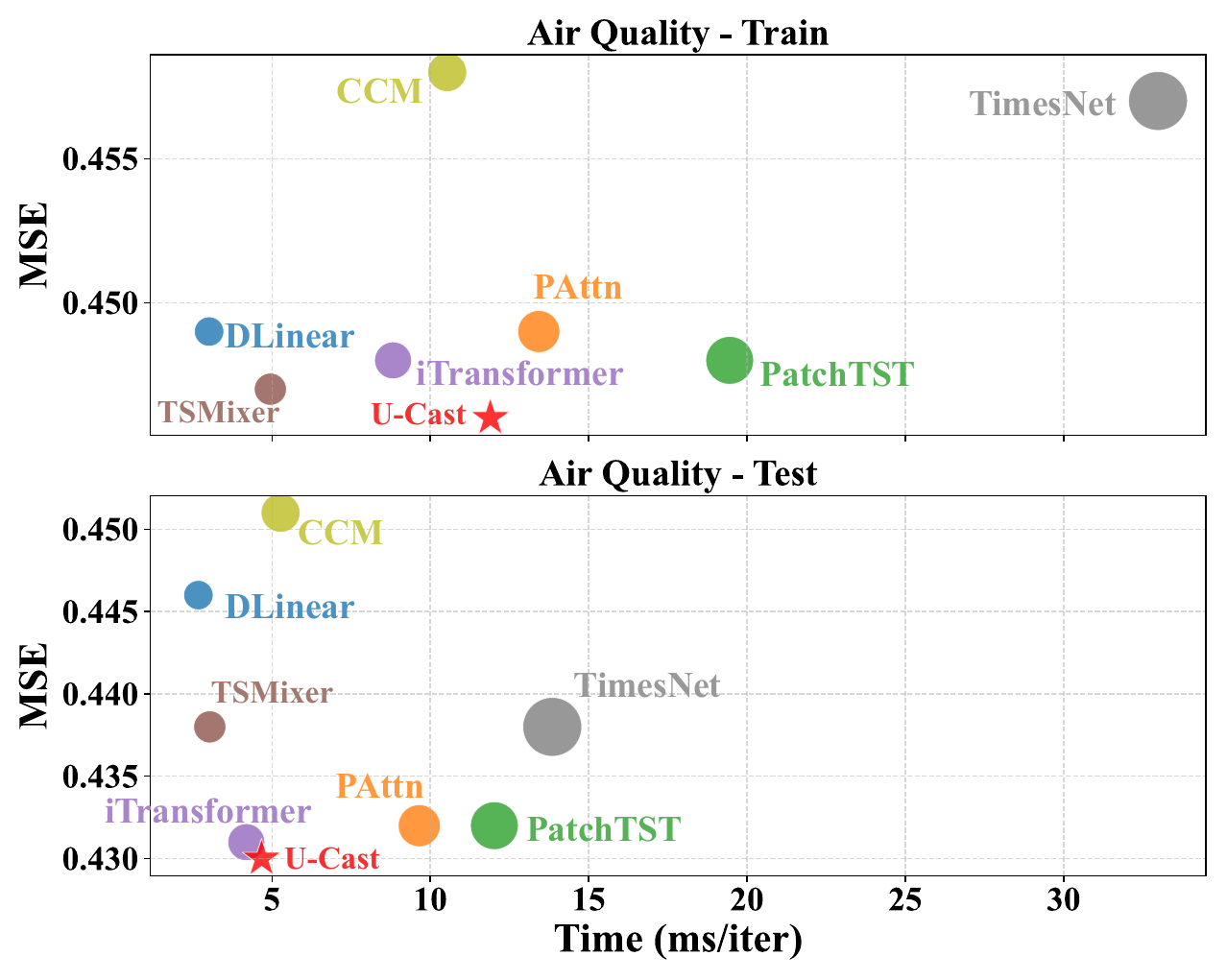}
\end{figure*}

\begin{figure*}[h!]
    \centering
    \includegraphics[width=0.94\textwidth]{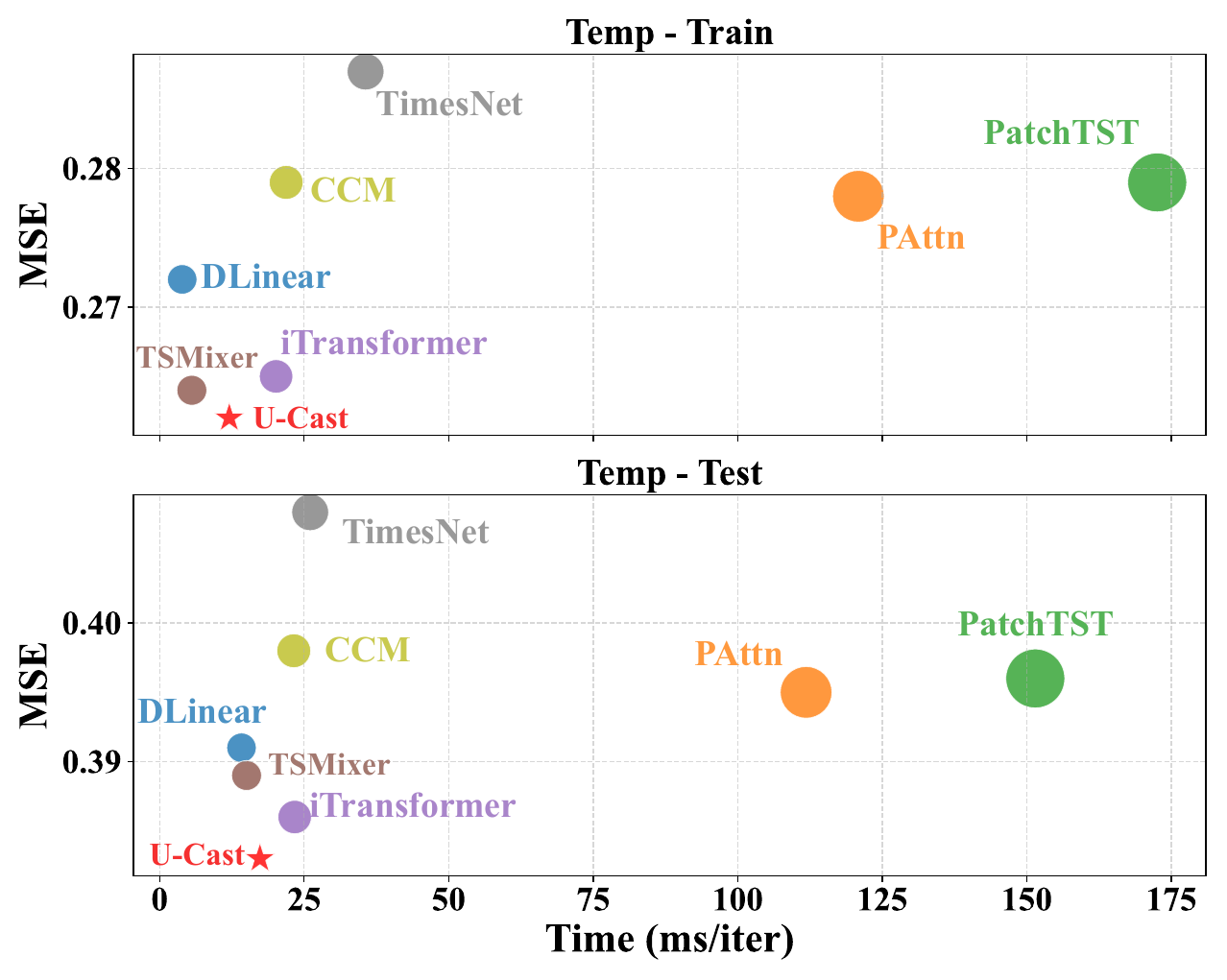}
\end{figure*}

\begin{figure*}[h!]
    \centering
    \includegraphics[width=0.94\textwidth]{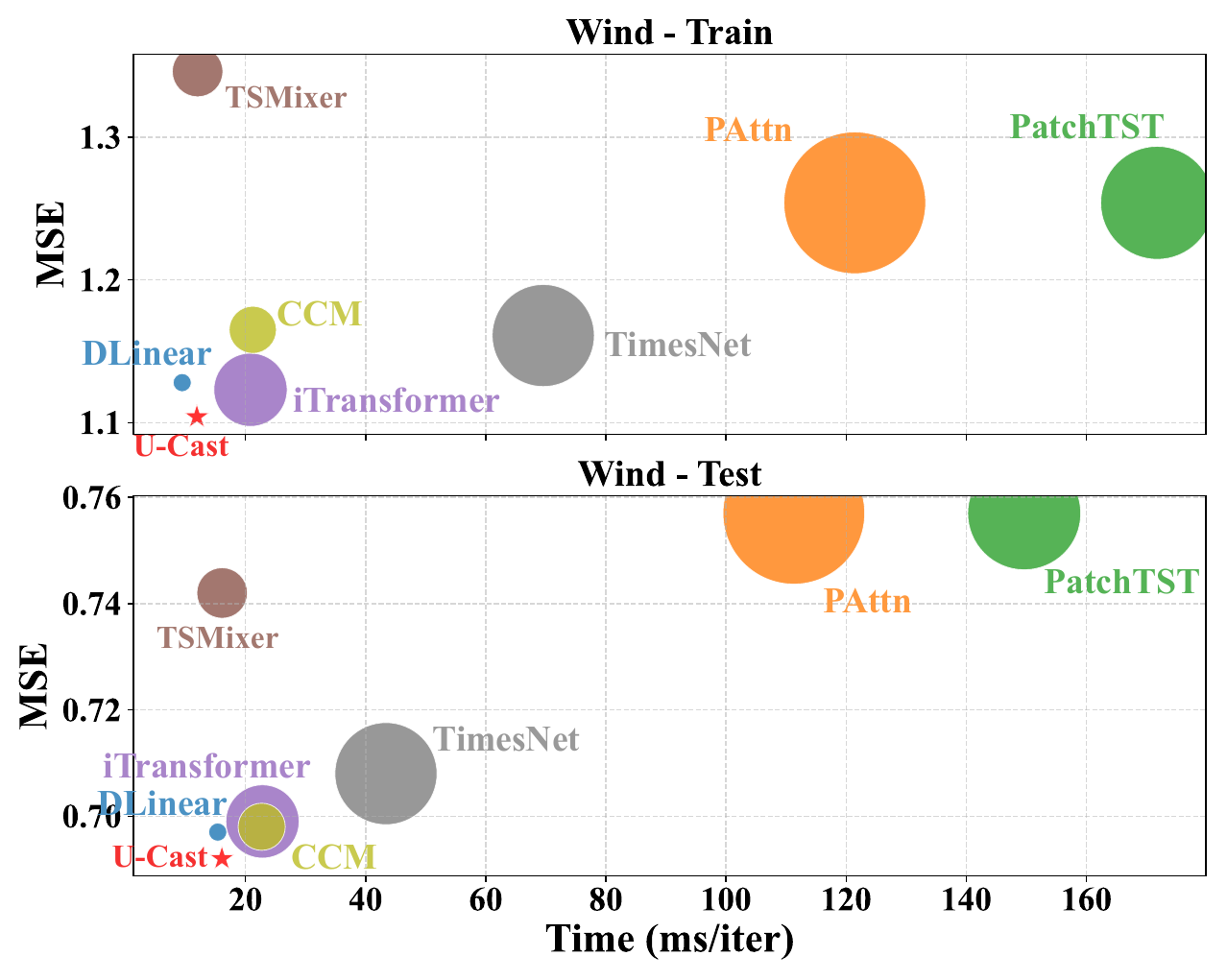}
\end{figure*}

\begin{figure*}[h!]
    \centering
    \includegraphics[width=0.94\textwidth]{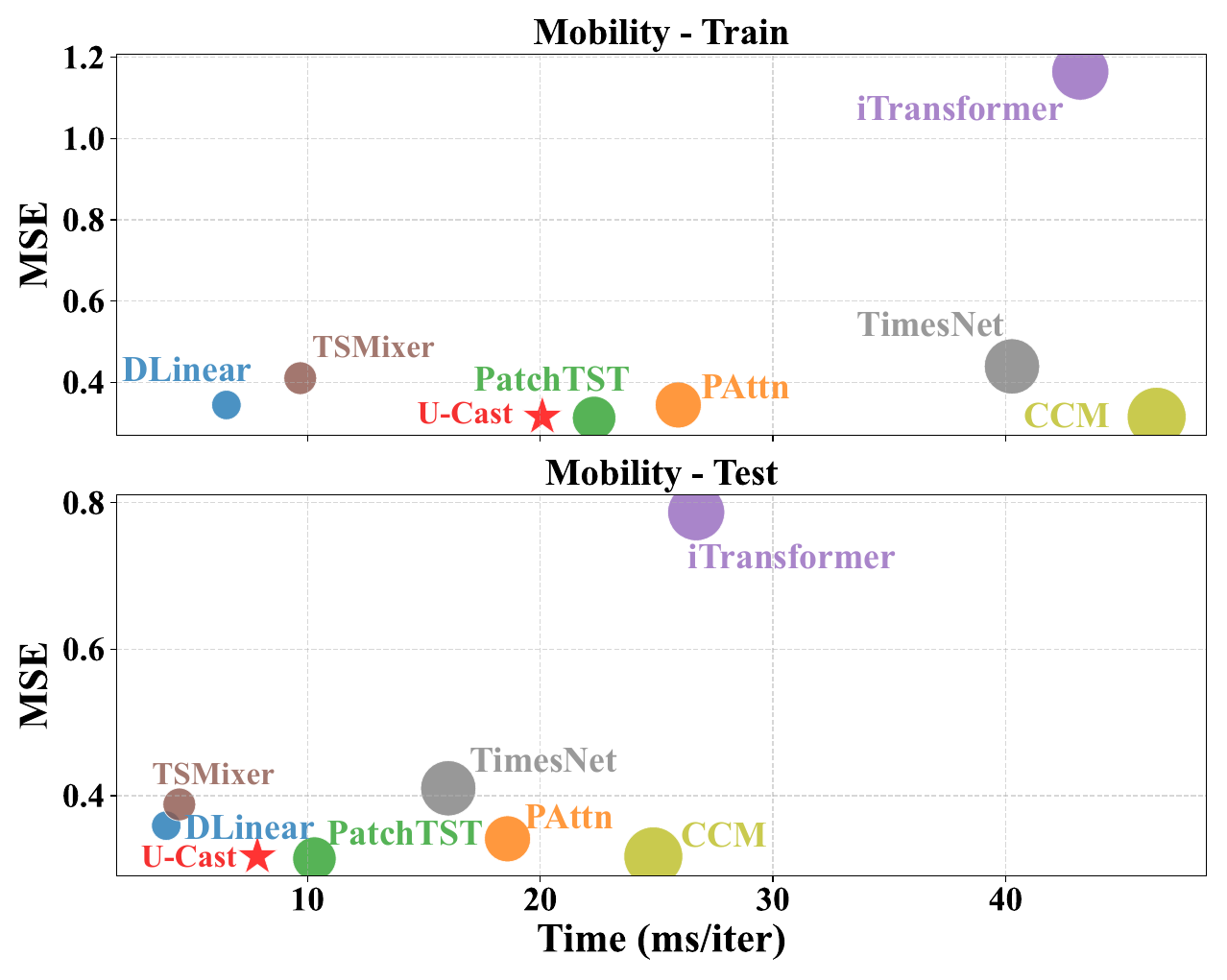}
\end{figure*}

\begin{figure*}[h!]
    \centering
    \includegraphics[width=0.94\textwidth]{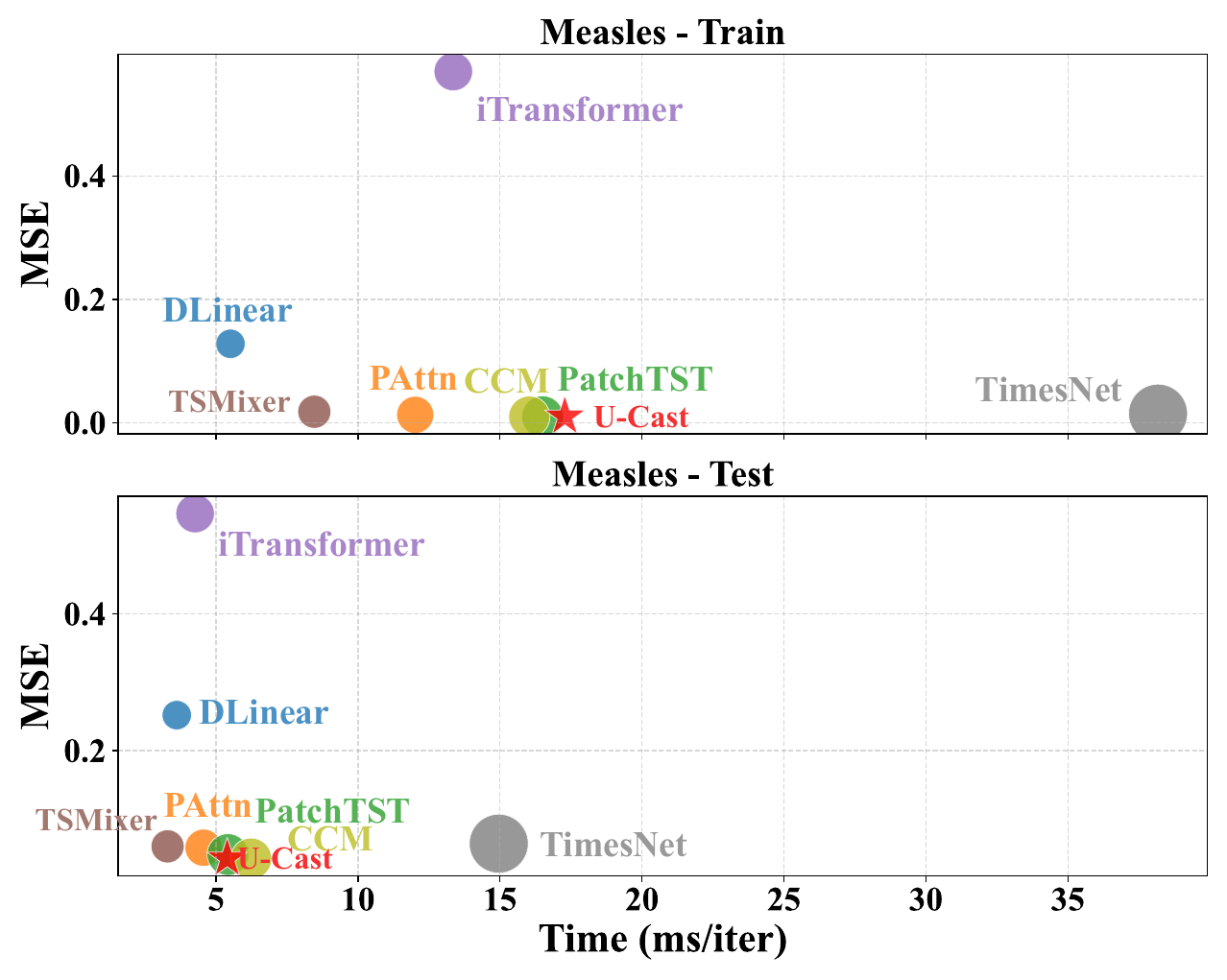}
\end{figure*}

\begin{figure*}[h!]
    \centering
    \includegraphics[width=0.94\textwidth]{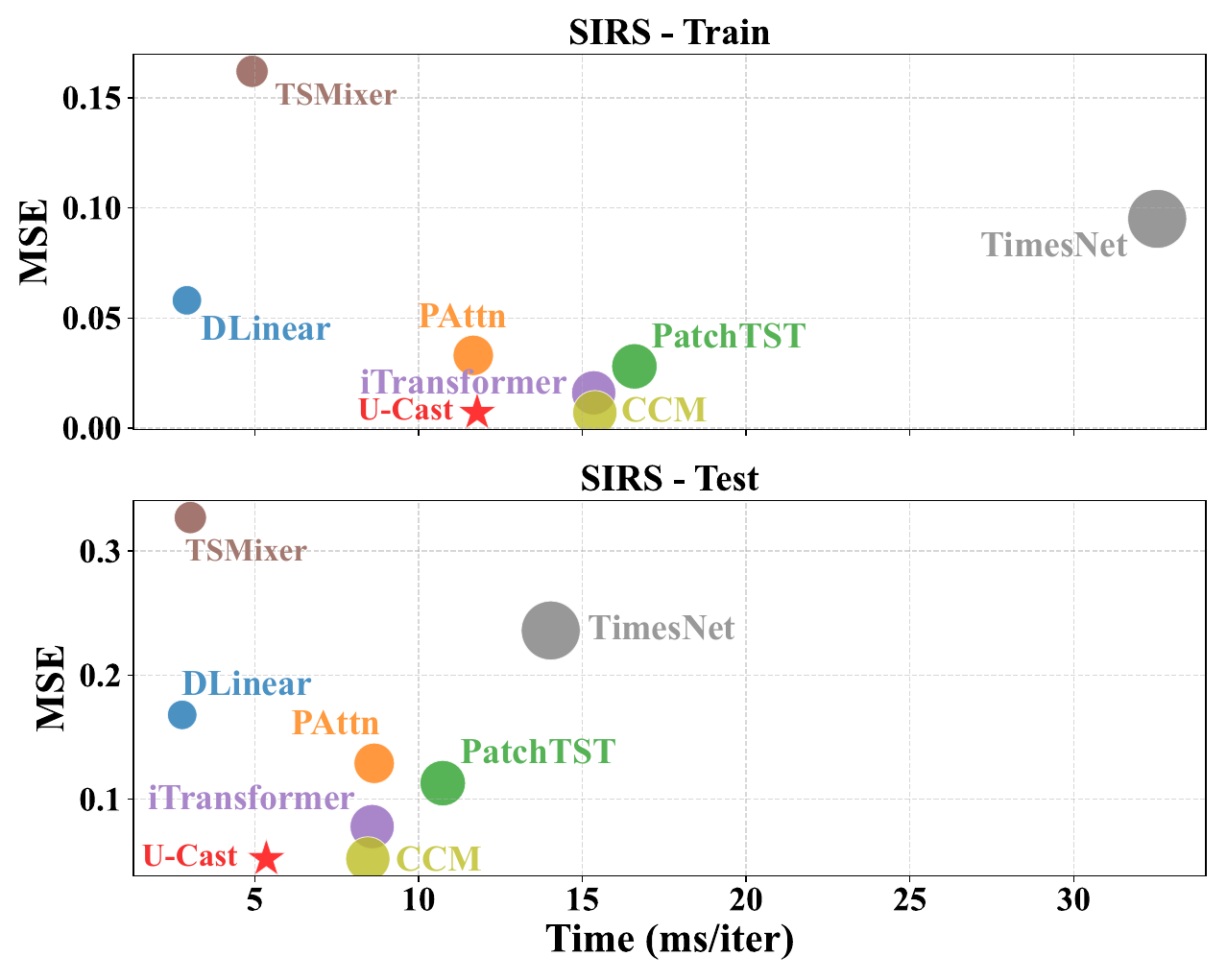}
\end{figure*}

\begin{figure*}[h!]
    \centering
    \includegraphics[width=0.94\textwidth]{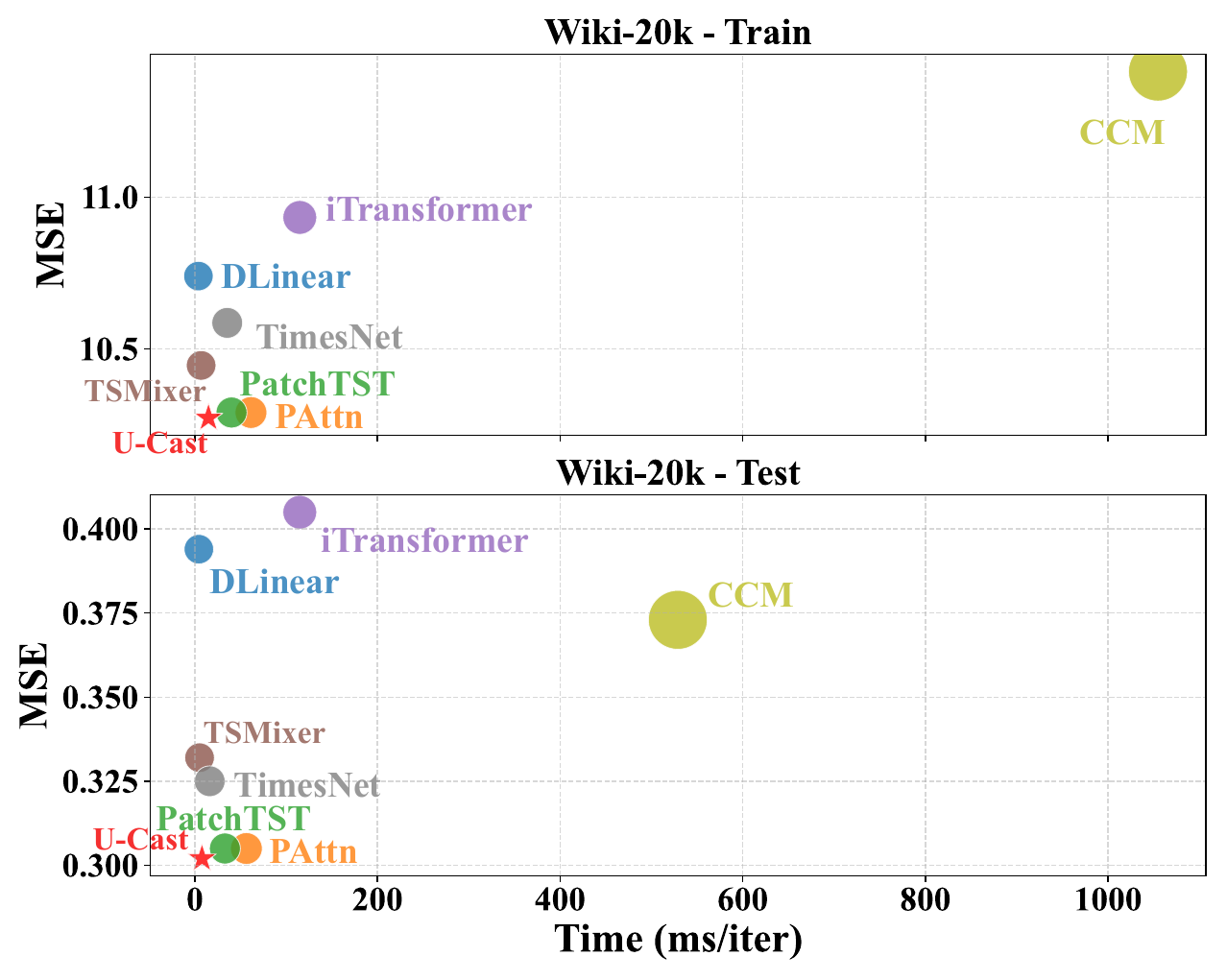}
\end{figure*}

\clearpage
\section{Showcases}
\label{app:show_cases}
\begin{figure*}[h!]
    \centering
    \includegraphics[width=0.99\textwidth]{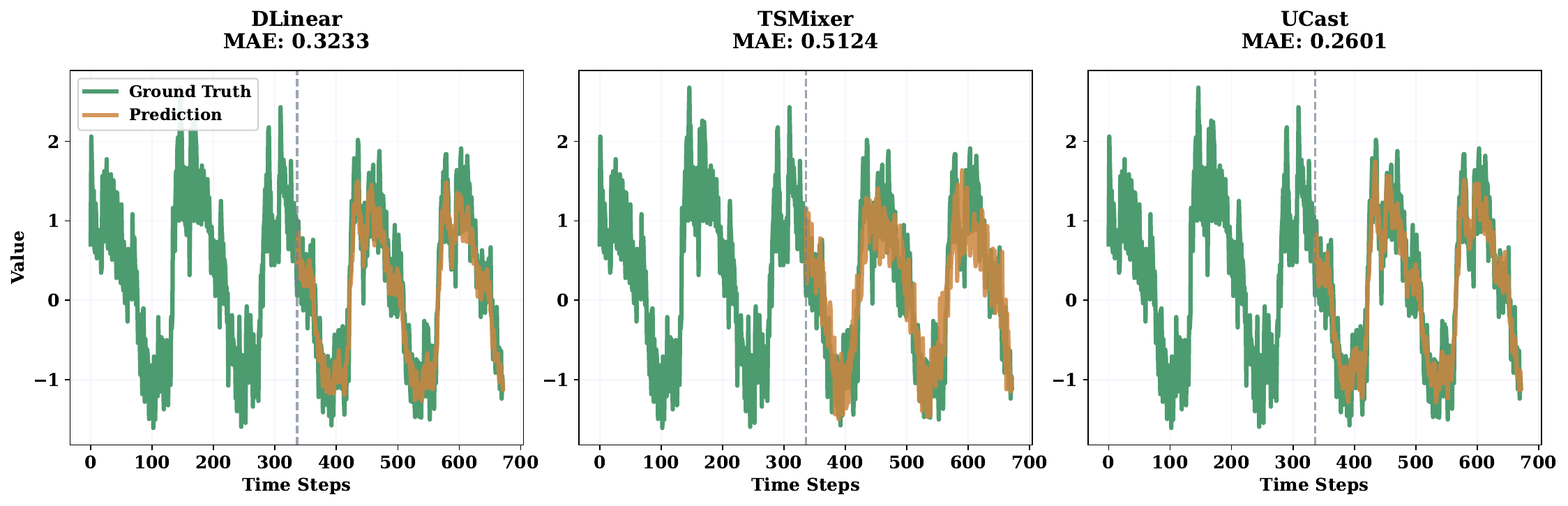}
    % \vskip -1em
    \caption{Prediction cases from Atec by different models under the input-S-predict-S settings.}
    % \vskip -1em
\end{figure*}

\begin{figure*}[h!]
    \centering
    \includegraphics[width=0.99\textwidth]{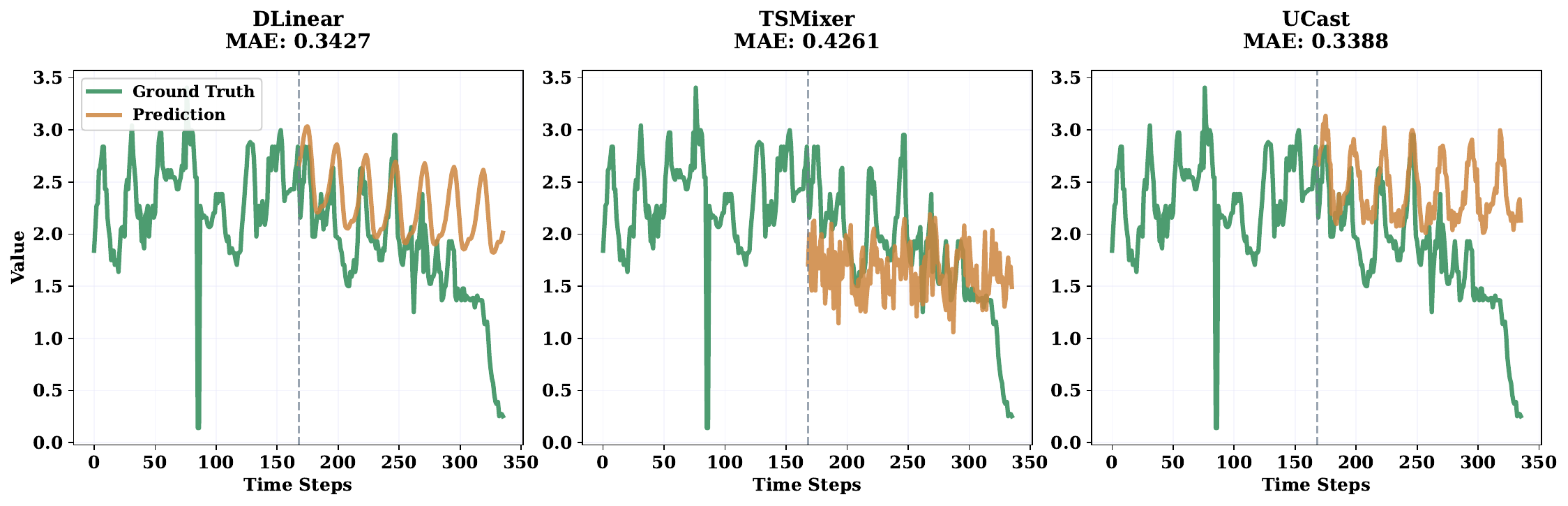}
    % \vskip -1em
    \caption{Prediction cases from Temp by different models under the input-S-predict-S settings.}
    % \vskip -1em
\end{figure*}

\begin{figure*}[h!]
    \centering
    \includegraphics[width=0.99\textwidth]{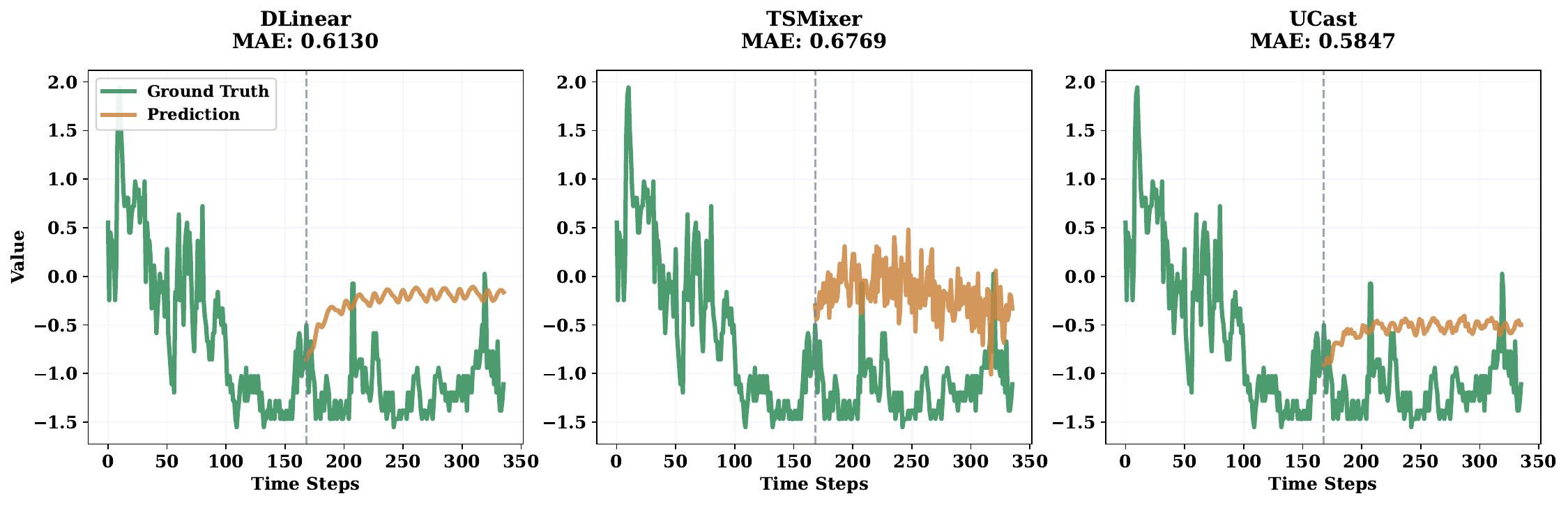}
    % \vskip -1em
    \caption{Prediction cases from Wind by different models under the input-S-predict-S settings.}
    % \vskip -1em
\end{figure*}

\begin{figure*}[h!]
    \centering
    \includegraphics[width=0.99\textwidth]{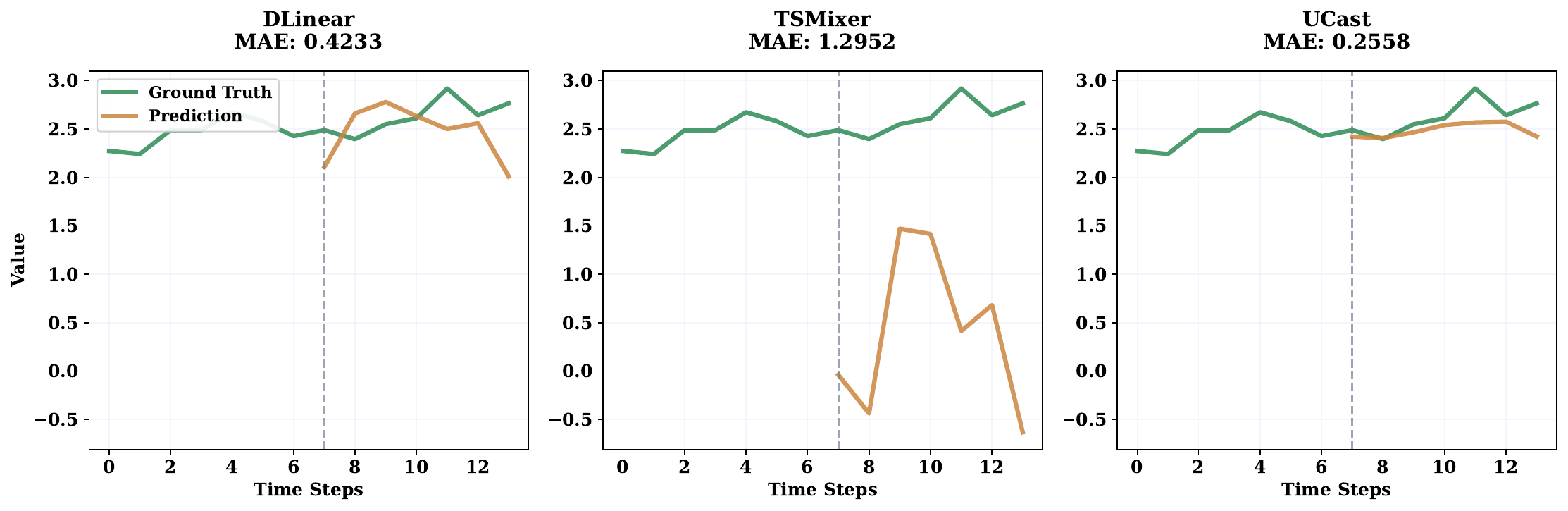}
    % \vskip -1em
    \caption{Prediction cases from Mobility by different models under the input-S-predict-S settings.}
    % \vskip -1em
\end{figure*}

% \begin{figure*}[h!]
%     \centering
%     \includegraphics[width=0.99\textwidth]{figures/case/M5_sample_268_rank_09.pdf}
%     % \vskip -1.5em
%     \caption{Prediction cases from M5 by different models under the input-S-predict-S settings.}
%     % \vskip -1em
% \end{figure*}

\begin{figure*}[h!]
    \centering
    \includegraphics[width=0.99\textwidth]{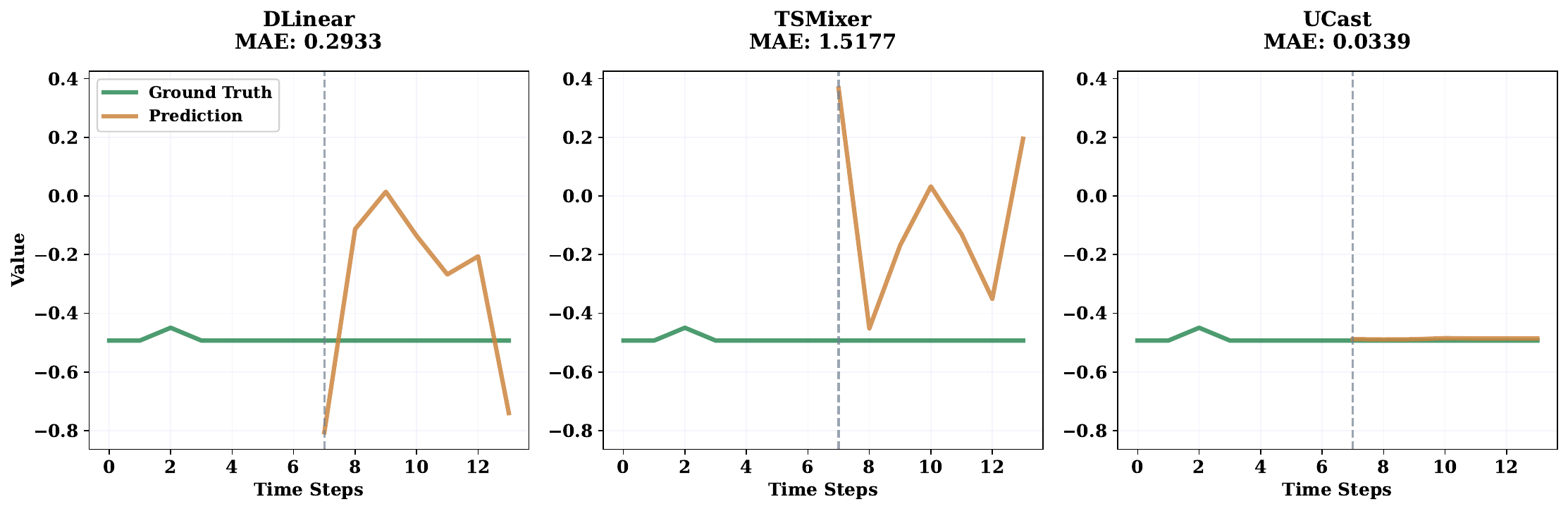}
    % \vskip -1.5em
    \caption{Prediction cases from Measles by different models under the input-S-predict-S settings.}
    % \vskip -1em
\end{figure*}

\begin{figure*}[h!]
    \centering
    \includegraphics[width=0.99\textwidth]{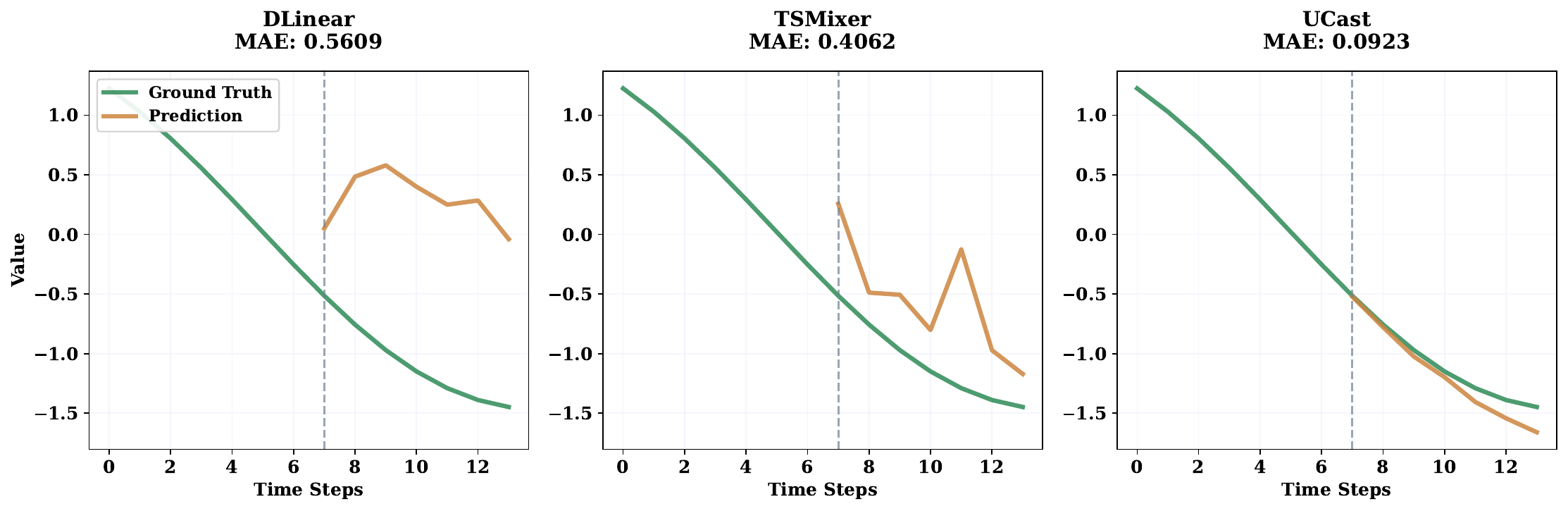}
    % \vskip -1.5em
    \caption{Prediction cases from SIRS by different models under the input-S-predict-S settings.}
    % \vskip -1em
\end{figure*}

% \begin{figure*}[h!]
%     \centering
%     \includegraphics[width=0.99\textwidth]{figures/case/Wiki_20k_sample_482_rank_07.pdf}
%     % \vskip -1.5em
%     \caption{Prediction cases from Wiki-20k by different models under the input-S-predict-S settings.}
%     % \vskip -1em
% \end{figure*}

\end{document}